\documentclass[sigconf,nonacm,screen,authoryear]{arxiv}

\usepackage{balance} % for balancing columns on the final page

\usepackage{algorithm,algpseudocode}
\usepackage{dsfont}
\usepackage{pifont}
\usepackage{makecell}
\usepackage{graphicx}

\usepackage{thm-restate}

\usepackage{subcaption}
\allowdisplaybreaks

\usepackage[utf8]{inputenc} % allow utf-8 input
\usepackage[T1]{fontenc}    % use 8-bit T1 fonts
\usepackage{url}            % simple URL typesetting
\usepackage{booktabs}       % professional-quality tables
\usepackage{amsfonts}       % blackboard math symbols
\usepackage{nicefrac}       % compact symbols for 1/2, etc.
\usepackage{microtype}      % microtypography
\usepackage{xcolor}         % colors
\usepackage{enumitem}

\newenvironment{subproof}[1]{\begin{proof}[Proof of #1]}{\end{proof}}

\newenvironment{proofOfOne}   {\subproof{(i)}}  {\end{proof}}
\newenvironment{proofOfTwo}   {\subproof{(ii)}} {\end{proof}}
\newenvironment{proofOfThree} {\subproof{(iii)}}{\end{proof}}
\newenvironment{proofOfFour}  {\subproof{(iv)}} {\end{proof}}
\newenvironment{proofS}       {\begin{proof}[Proof sketch]}{\end{proof}}

\newcommand{\va}{\text{\textcolor{violet}{A}}}
\newcommand{\vb}{\text{\textcolor{violet}{B}}}
\newcommand{\vc}{\text{\textcolor{violet}{C}}}
\newcommand{\vd}{\text{\textcolor{violet}{D}}}
\newcommand{\ve}{\text{\textcolor{violet}{E}}}

\newcommand{\vbo}{\color{violet}{\text{B}_1}}
\newcommand{\vbt}{\color{violet}{\text{B}_2}}
\newcommand{\vbth}{\color{violet}{\text{B}_3}}
\newcommand{\vbf}{\color{violet}{\text{B}_4}}
\newcommand{\vbfi}{\color{violet}{\text{B}_5}}

\newcommand{\vdo}{\color{violet}{\text{D}_1}}
\newcommand{\vdt}{\color{violet}{\text{D}_2}}

\newcommand{\update}[1]{{\color{black}#1}}

\setcopyright{none}
\acmConference[]{}{}{}{} 
\acmBooktitle{}          
\acmYear{} \acmDOI{} \acmISBN{} 
\renewcommand\footnotetextcopyrightpermission[1]{} 
\setcitestyle{authoryear,round} 

\title[]{{Decentralized Asynchronous Multi-player Bandits}}

\author{Jingqi Fan}
\affiliation{
  \institution{Northeastern University, China}
  \city{Shenyang}
  \country{China}}
\email{fanjingqi@stumail.neu.edu.cn}

\author{Canzhe Zhao}
\affiliation{
  \institution{Shanghai Jiao Tong University}
  \city{Shanghai}
  \country{China}}
\email{canzhezhao@sjtu.edu.cn}

\author{Shuai Li}
\affiliation{
  \institution{Shanghai Jiao Tong University}
  \city{Shanghai}
  \country{China}}
\email{shuaili8@sjtu.edu.cn}

\author{Siwei Wang}
\authornote{Corresponding author.}
\affiliation{
  \institution{Microsoft Research Asia}
  \city{Beijing}
  \country{China}}
\email{siweiwang@microsoft.com}

\begin{abstract}
In recent years, multi-player multi-armed bandits (MP-MAB) have been extensively studied due to their wide applications in cognitive radio networks and Internet of Things systems. 
While most existing research on MP-MAB focuses on synchronized settings, real-world systems are often decentralized and asynchronous, where players may enter or leave the system at arbitrary times, i.e., there does not exist a global variable that can be assumed as a common knowledge of other users' clock.
This decentralized asynchronous setting introduces two major challenges. 
First, without a global clock, players cannot implicitly coordinate their actions through time, making it difficult to avoid collisions. 
Second, it is important to detect how many players are in the system, but doing so may cost a lot.
In this paper, we address the challenges posed by such a fully asynchronous setting in a decentralized environment. 
We develop a novel algorithm in which players adaptively change between exploration and exploitation. During exploration, players uniformly pull their arms, reducing the probability of collisions and effectively mitigating the first challenge. Meanwhile, players continue pulling arms currently exploited by others with a small probability, enabling them to detect when a player has left, thereby addressing the second challenge. 
We prove that our algorithm achieves a regret of \(\mathcal{O}(\sqrt{T \log T} + {\log T}/{\Delta^2})\), where \(\Delta\) is the minimum expected reward gap between any two arms. 
To the best of our knowledge, this is the first efficient MP-MAB algorithm in the asynchronous and decentralized environment.
Extensive experiments further validate the effectiveness and robustness of our algorithm, demonstrating its applicability to real-world scenarios. 
\end{abstract}


\newcommand{\BibTeX}{\rm B\kern-.05em{\sc i\kern-.025em b}\kern-.08em\TeX}

\begin{document}

%%% The following commands remove the headers in your paper. For final 
%%% papers, these will be inserted during the pagination process.

\pagestyle{fancy}
\fancyhead{}

%%% The next command prints the information defined in the preamble.

\maketitle

\section{Introduction}

Multi-armed bandit (MAB) is a well-established model with broad applications in areas such as online advertising, clinical trials, and recommendation systems \citep{auer2002finite}. In this problem, at each time step $t\leq T$, a player pulls an arm \( k \) from a finite set \([K] := \{1, \dots, K\}\) and receives a stochastic reward \( X_k(t) \). The goal is to maximize the cumulative reward of this player, which is equivalent to minimizing the regret, defined as the cumulative reward difference between the optimal arm and the chosen arms over time. However, many real-world scenarios exhibit complexities that the standard MAB model cannot fully capture. For instance, in cognitive radio systems, efficient spectrum sharing among users is crucial \citep{wyglinski2009cognitive}. Unlike the traditional MAB setup, these systems contain multiple players, and face collisions when more than one users select the same channel, leading to failed transmissions. This challenge gives rise to the multi-player multi-armed bandit (MP-MAB) problem, where \( M \) players simultaneously pull arms from $[K]$. When multiple players pull the same arm, their rewards turn to zero, indicating that no information is transmitted. 
\update{Compared with the single-player setting, the MP-MAB problem introduces additional layers of difficulty, as players need coordinate with others while still dealing with uncertainty in reward distributions.}

When players can observe the arm selections and corresponding rewards of all other players at each step, the problem falls into the centralized setting. In this case, shared information enables coordination to avoid collisions and optimize resource allocation through joint strategy updates. \citet{komiyama2015optimal} proposed algorithms in this setting that achieve an asymptotic optimal regret of $\mathcal{O}({\log T}/\Delta)$. 
However, in practical systems, 
frequent explicit communication among players 
incurs high energy overhead, making centralized coordination costly. To address these limitations, recent research has focused on the decentralized setting, where each player acts based solely on her own observations, and direct communication is not allowed. Despite this constraint, many existing approaches deliberately introduce collisions as an \textit{implicit communication} %
mechanism, enabling players to indirectly share information and thereby approximate the performance of the centralized case \citep{boursier2019sic, huang2022towards}. These methods typically assume a synchronous environment, where all players enter the system simultaneously and remain active throughout.
As a result, all players know a global clock, which is critical for the sharing protocol.

In contrast, real-world applications often involve inherently asynchronous systems. For example, in cognitive radio networks, users access the spectrum based on local availability and transmission demands, joining and leaving the network at arbitrary times \citep{liang2011cognitive}. 
Similarly, in Internet of Things deployments, sensors and edge devices operate on independent schedules, waking up or going offline in response to environmental triggers or battery levels \citep{li2015internet}. In such environments, players may join or leave the system at unpredictable times. 
This fails most of the existing algorithms under synchronous setting. %e.g., a new player may break a communication between two existing players, since he does not know others are trying to communicate. 
There are also some prior works who try to relax the synchronization assumption. For example, \citet{rosenski2016multi} assume a shared global clock to synchronize each epoch, and require to use a lower bound of $\Delta$ as input, which could be impractical in real applications. 
\citet{boursier2019sic} allow players to join at different times but require them to remain active until the end. Other models assume that each player is active in each round with some fixed probability \citep{bonnefoi2017multi, dakdouk2022massive, richard2024constant}. 
More details of related works are deferred to Appendix \ref{sec:related}.
While these approaches offer valuable insights into partially asynchronous settings, their applicability remains limited under more general forms of asynchrony and decentralization. %, where players act independently without global timing or coordination. 

\subsection{Our Contribution} \label{sec:cont}

In this paper, we consider a decentralized asynchronous setting in which players are unaware of the global clock, and may join or leave the system at arbitrary times. Compared to existing work that assumes players either become active with some fixed probability or enter the system arbitrarily but remain until the end, our asynchrony model is more general and better reflects real-world scenarios. The unpredictable access patterns in the decentralized environment introduce two major challenges:
\begin{enumerate}[label=(\roman*)]
    \item The absence of a global clock makes implicit communication through collisions unreliable. Since new players may join and pull arms at arbitrary times, they can unintentionally collide with existing players. This disrupts the structured implicit communication patterns, ultimately leading to frequent and uncontrolled collisions.
    \item The dynamic nature of player participation makes it important to detect the number of current active players. If the number is overestimated, then the player may exploit an arm that is not good enough, leading to unacceptable regret. How to detect the number of current active players with minimum cost is also challenging. 
\end{enumerate}

To deal with the above challenges, we propose a novel algorithm named \textbf{Adaptive Change between Exploration and Exploitation (ACE)}. %for the decentralized asynchronous MP-MAB problem. 
ACE enables every player $j$ to estimate an arm set \(\mathcal{A}^j\), which contains all arms that are believed to be currently exploited by other players. 
Based on this estimation, they can adaptively alternate between exploration and exploitation. %In particular, each player \(j\) maintains an arm set \(\mathcal{A}^j\), which contains all arms that are believed to be currently exploited by other players. 
In the exploration phase, player \(j\) randomly explores arms in \([K] \setminus \mathcal{A}^j\), and switches to the exploitation phase once she identifies a high-probability optimal arm \(\hat{k}^j \in [K] \setminus \mathcal{A}^j\). In the exploitation phase, she repeatedly pulls \(\hat{k}^j\) with high probability, and returns to exploration once she detects that an arm which was previously in $\mathcal{A}^j$ becomes available again and is not sufficiently explored. 
In this way, ACE addresses the two challenges. On the one hand, uniformly pulling arms from \([K] \setminus \mathcal{A}^j\) during the exploration phase ensures a sufficiently randomized access, which significantly reduces the probability of collisions, thereby mitigating challenge (i). On the other hand, players continue pulling arms in \(\mathcal{A}^j\) with a small probability, allowing them to use a small cost to detect when such arms become available and update their exploitation choices accordingly. This mechanism effectively addresses challenge (ii).

Our analysis shows that ACE achieves a regret upper bound of \(\mathcal{O}(\sqrt{T \log T} + {\log T}/{\Delta^2})\), where the \(\mathcal{O}(\sqrt{T \log T})\) term arises from Challenge (ii), as players must occasionally try arms in \(\mathcal{A}^j\) with probability $\varepsilon = \mathcal{O}(\sqrt{\log T / T})$ to detect changes in availability. The \(\mathcal{O}({\log T}/{\Delta^2})\) term corresponds to Challenge (i), due to the unavoidable collisions caused by uniform exploration, resulting in a dependence on \(1/\Delta^2\) rather than the standard \(1/\Delta\).
We further support our theoretical findings with comprehensive experiments, which confirm the practical effectiveness and robustness of ACE across a variety of asynchronous settings, including large-scale scenarios with many players and arms.

\section{Preliminaries}

We consider a \(T\)-step decentralized asynchronous multi-player multi-armed bandit problem with \(K\) arms and \(M\) players.  
Let \([K] := \{1, 2, \ldots, K\}\) denote the set of arms, and \([M] := \{1, 2, \ldots, M\}\) denote the set of players.  
Each player \(j \in [M]\) joins the system at time step \(T^j_{\mathrm{start}}\) and leaves at time step \(T^j_{\mathrm{end}}\).  
%
% We define \(T^j = T^j_{\mathrm{end}} - T^j_{\mathrm{start}}\) as the active period length of player \(j\).  
%
% Note that in the decentralized asynchronous setting, player \(j\) does not know \(T^j_{\mathrm{start}}\) or \(T^j_{\mathrm{end}}\), but only the total duration \(T^j\)
% Note that in the decentralized asynchronous setting, player \(j\) is unaware of her own \(T^j_{\mathrm{start}}\) and \(T^j_{\mathrm{end}}\), but only knows that the game lasts for totally $T$ time steps, i.e., $T^j_{\mathrm{start}}$, $T^j_{\mathrm{end}}$ (and real time step $t$) are not able to be the input of her algorithm.
Note that in the decentralized asynchronous setting, player \(j\) is unaware of her own \(T^j_{\mathrm{start}}\) and \(T^j_{\mathrm{end}}\), and only knows that the game lasts for a total of \(T\) time steps. That is, \(T^j_{\mathrm{start}}\), \(T^j_{\mathrm{end}}\), and the actual time step \(t\) cannot be used as inputs to her algorithm.

At each \update{discrete} time \update{step} $T^j_\mathrm{start} \le t \le T^j_\mathrm{end}$, player $j$ selects an arm $\pi^j(t) \in [K]$ to pull (for $t < T^j_\mathrm{start}$ or $t > T^j_\mathrm{end}$, we let $\pi^j(t) = 0$).
If more than one players choose arm $k$ at $t$, then there is a collision, and $\eta_k(t):= \mathds{1}[\sum_{j\leq M} \mathds{1}[ \pi^j(t) = k ]>1]$ denotes the collision indicator.
For player $j$, her observation at step $t$ contains two values, $\eta^j(t) = \eta_{\pi^j(t)}(t)$ tells her whether there is a collision, and $r^j(t):= (1-\eta_{\pi^j(t)}(t))X_{\pi^j(t)}(t)$ is her reward in this step. Here $X_{\pi^j(t)}(t)$ is drawn independently according to an unknown fixed distribution with expectation $\mu_{\pi^j(t)}\in[0,1]$. 
Without loss of generality, we assume that $\mu_1 > \mu_2 > \cdots > \mu_K$ \citep{wang2020optimal, mahesh2022multi, mahesh2024attacking}.
For player $j$, her own history information is given by $\mathcal{F}_{t-1}^j = \{(t'-T^j_\mathrm{start},\pi^j(t^\prime),\eta^j(t^\prime),r^j(t^\prime))| T^j_\mathrm{start}\le t^\prime \le t-1\}$.

The goal of the players is to choose arms properly based on their own history $\mathcal{F}_{t-1}^j$'s to minimize the regret defined as 
\begin{equation*}
    R(T) := \sum_{t=1}^T \sum_{k \leq m_t}\mu_k -\mathbb{E}\left[\sum_{t=1}^{T}\sum_{j: T_\mathrm{start}^j \le t \le T_\mathrm{end}^j } r^j(t)\right],
\end{equation*}
where $m_t := |\{j: T_\mathrm{start}^j \le t \le T_\mathrm{end}^j \}|$ denotes the number of active players at time \( t \), and the baseline $\sum_{t=1}^T \sum_{k \leq m_t}\mu_k$ is the best expected reward one can get in a centralized offline setting. %We consider the following assumptions.

\begin{assumption}\label{Assumption:active_players}
    The number of active players during the game is upper bounded, i.e., there exists a constant $m$ such that for any \( t \leq T \), $m_t \le m \le {{K}/{2}}$.
\end{assumption}

% Assumption~\ref{Assumption:knownT} is introduced to facilitate the analysis by enabling a bound on the regret of bad events, where the estimated reward deviates significantly from the true mean.  
% Note that this assumption can be relaxed to a more general case where each player \(j\) satisfies \(T^j \leq T_{\max} \leq T\), with \(T_{\max}\) known in advance. This condition is realistic in cognitive radio networks where secondary users operate under time-limited licenses or expected activity durations, and thus have access to an upper bound on their participation time \citep{wyglinski2009cognitive, liang2011cognitive}.

Assumption~\ref{Assumption:active_players} ensures that players have access to enough arms. This kind of assumption is common in real world applications \citep{jia2009cooperative, jin2010multicast, zhang2010dynamic, ngo2011distributed, naeem2013resource, kumar2021multi, mughal2024intelligent}, and is well adopted in MP-MAB literatures \citep{besson2018multi,bistritz2018distributed,boursier2019sic, shi2020decentralized,shi2021heterogeneous, xiong2023decentralized,boursier2024survey}.

\section{Algorithm}

In this section, we propose our \textbf{Adaptive Change between Exploration and Exploitation (ACE)} algorithm. The key idea behind ACE is that players adaptively change between exploration and exploitation based on the number of collisions over a given period, enabling them to update their exploitation arms to the latest optimal choices.

\begin{algorithm*}[t]
\caption{ACE (from the view of player $j$)}
\label{alg:1} 
\begin{algorithmic}[1] 
    \Statex \textbf{Input:} $T$, $K$ (the number of arms), $m$ (the maximum number of players), $\varepsilon$ \update{(the probability of pulling arms in $\mathcal{A}^j$ during exploration)}
    \State  \textbf{Init:} $\hat{k}^j = 0$, $\mathcal{A}^j= \emptyset$ (the set of occupied arms), $ \texttt{Correction} = \texttt{False}, \texttt{Phase} = \texttt{Exploration}$. For all $ k\leq K$, initialize $\mathcal{P}^j_{k}, \mathcal{Q}^j_{k}$ as empty queue separately with length $L_p, L_q$ as defined in \eqref{eq:length}.
    \While{Player $j$ remains in the system}
    \State $k^j_1, k^j_2 \leftarrow \mathrm{DoubleSelection}()$ \label{l:uni}
    \State Pull $k^j_1, k^j_2$ and observe $r^j(t_1), \eta_{k^j_1}(t_1), r^j(t_2), \eta_{k^j_2}(t_2)$
    \State \textbf{if} {$k^j_1\in \mathcal{A}^j$} {($k^j_2\in \mathcal{A}^j$)} \textbf{then} Add $[1-\eta_{k^j_1}(t_1)]$ to the end of $\mathcal{Q}^j_{k^j_1}$ (Add $[1-\eta_{k^j_2}(t_2)]$ to the end of $\mathcal{Q}^j_{k^j_2}$) \textbf{end if} \label{l:p_col-0}
    %\State \textbf{if} {$k^j_2\in \mathcal{A}^j$} \textbf{then} Add $[1-\eta_{k^j_2}(t_2)]$ to the end of $\mathcal{Q}^j_{k^j_2}$ \textbf{end if} \label{l:p_col}
    \If{$\texttt{Phase} = \texttt{Exploration}$}
    \State Update $N^j_k, \hat{\mu}^j_k, \mathrm{UCB}^j_{k}, \mathrm{LCB}^j_k, \forall k\in [K] \setminus \mathcal{A}^j$ according to \eqref{eq:mu1}, \eqref{eq:mu2}, \eqref{eq:ucb_compute1} and \eqref{eq:ucb_compute2} \label{l:update}
    \State \textbf{if} $k^j_1 = k^j_2$ \textbf{then} Add $[\eta_{k^j_1}(t_1)\cdot \eta_{k^j_2}(t_2)]$ to the end of $\mathcal{P}^j_{k^j_1}$ \textbf{end if} \label{l:puttwo}
    \If{$\exists k \in [K]\setminus \mathcal{A}^j \text{ s.t. }\sum_{i\in \mathcal{P}^j_k} i\geq \lceil 0.85L_p \rceil$} \label{l:cond1} \textcolor{violet}{\Comment{Find an occupied arm}}
    \State Add $k$ to $\mathcal{A}^j$ and reset $\mathcal{P}_k^j$  \label{l:putA}
    \State \textbf{if} $|\mathcal{A}^j| > m - 1$ \textbf{then} $\texttt{Correction} \leftarrow \texttt{True}$ \textbf{end if} \label{l:cor-grt}
    \EndIf 
    \If{$\exists k \in \mathcal{A}^j,\text{ s.t. }\sum_{i\in \mathcal{Q}^j_{k}} i \geq \lceil 0.142L_q \rceil$}  \label{l:cond2_new} \textcolor{violet}{\Comment{Find a released arm}} 
    \State Remove $k$ from $\mathcal{A}^j$ and reset $Q^j_k$ 
    \State \textbf{if} $|\mathcal{A}^j| < m$ \textbf{then} $\texttt{Correction} \leftarrow \texttt{False}$ \textbf{end if}
    \EndIf \label{l:delete-2}
    \If{$\texttt{Correction} = \texttt{False}$ \textbf{and} $\exists k \in [K]\setminus \mathcal{A}^j$, $k$ satisfies Conditions \ref{eq:cd1} and \ref{eq:cd2}} \label{l:nsame_0}
    \State $\hat{k}^j \leftarrow k$ and $\texttt{Phase} \leftarrow \texttt{Exploitation}$ \textcolor{violet}{\Comment{Be prepared to exploit $\hat{k}^j$}}
    \EndIf
    \ElsIf{$\texttt{Phase} = \texttt{Exploitation}$}
    \If{$\exists k \in \mathcal{A}^j,\text{ s.t. }\sum_{i\in \mathcal{Q}^j_{k}} i \geq \lceil 0.142L_q \rceil$}  \label{l:cond2} \textcolor{violet}{\Comment{Find a released arm}} 
    \State Remove $k$ from $\mathcal{A}^j$ and reset $\mathcal{Q}_k^j$ \label{l:remove-2}
    \State \textbf{if} {$\mathrm{LCB}^j_{\hat{k}^j} < \mathrm{UCB}^j_{k}$} \textbf{then} $\hat{k}^j \leftarrow 0$ and $\texttt{Phase} \leftarrow \texttt{Exploration}$ \textbf{end if} \textcolor{violet}{\Comment{Back to Exploration Phase}} \label{l:return}
    \EndIf
    \EndIf
    %\State $t\leftarrow t+2$
    \EndWhile
\end{algorithmic}
\end{algorithm*}

\subsection{Notations}

In general, the algorithm is divided into exploration phase and exploitation phase. 
Let $\hat{k}^j$ denote the arm that is exploited by player $j$ during the exploitation phase, and $\mathcal{A}^j$ denote the set of arms that player \(j\) perceives as occupied (i.e., arms she believes are being exploited by others). 
To adaptively change between exploration and exploitation, we maintain two queues \(\mathcal{P}^j_k\) and \(\mathcal{Q}^j_k\) for each player \(j\) and arm \(k\). The lengths of these queues are defined as:
\begin{equation}
    |\mathcal{P}^j_k|:=L_p= \lceil 866\ln(T) \rceil, \quad |\mathcal{Q}^j_k|:= L_q= \lceil 570\ln(T)\rceil ~.
    \label{eq:length}
\end{equation}
Let \(\varepsilon \in (0,1)\) denote a parameter that controls the trade-off between doing exploration-exploitation and checking whether $\mathcal{A}^j$ is accurate.
% for player \(j\). 
We use the phrase “player \(j\) occupies arm \(k\)” to indicate that player \(j\) is selecting \(\hat{k}^j\) to exploit in the exploitation phase, during which she pulls \(\hat{k}^j\) with high probability. Similarly, we say “arm \(k\) is occupied at time \(t\)” if there exists a player \(j\) such that \(\hat{k}^j = k\). Conversely, “arm \(k\) is released at time \(t\)” refers to the situation where the player previously occupying arm \(k\) either leaves the system or stops exploitation and returns to exploration.

\subsection{Exploration Phase} \label{sec:exp}

When a player \(j\) enters the system, she first initializes two empty queues \(\mathcal{P}^j_k\) and \(\mathcal{Q}^j_k\) for each arm \(k \in [K]\), as well as an empty set \(\mathcal{A}^j\). She then sets her current phase to exploration. 

In the exploration phase, the player sequentially pulls two arms over two consecutive time steps, denoted by \(k^j_1\) and \(k^j_2\), as specified in Algorithm~\ref{alg:2}. Specifically, $k^j_1$ is the arm to be explored and is uniformly sampled from $[K] \setminus \mathcal{A}^j$. Then, with probability $\varepsilon$, $k^j_2$ is uniformly sampled from $\mathcal{A}^j$ if $\mathcal{A}^j\neq \emptyset$; otherwise, $k^j_2$ is set equal to $k^j_1$ (Line \ref{l:ds-exp-1}-\ref{l:ds-exp-2} in Algorithm~\ref{alg:2}). 
After pulling $k^j_1$ and $k^j_2$, player $j$ inserts $[1-\eta_{k^j_2}(t_2)]$ to the end of $\mathcal{Q}^j_{k^j_2}$ if $k^j_2$ is sampled from $\mathcal{A}^j$ (Line \ref{l:p_col-0} %-\ref{l:p_col} 
in Algorithm \ref{alg:1}). 
How these queues work will be explained after a few lines. 
The player also updates her estimations upon receiving feedback at each time step \(t\) according to
\begin{align}
& \hat{\mu}^j_k(t) := \frac{\sum_{t^\prime=1}^t r^j_k(t^\prime)\, \mathds{1}\{\eta_k(t^\prime) = 0\}}{N^j_k(t)} ~, 
\label{eq:mu1} \\
& N^j_k(t) := \sum_{t^\prime=1}^t \mathds{1}\{\pi^j(t^\prime) = k,\ \eta_k(t^\prime) = 0\} ~, 
\label{eq:mu2} 
\end{align}
where $\hat{\mu}^j(t)$ denotes player $j$'s estimate of the mean reward of arm $k$ at time $t$, and \(N^j_k(t)\) denotes the number of successful (i.e., non-collided) pulls of arm \(k\) by player \(j\) up to time \(t\). In addition, the player updates the upper and lower confidence bounds as
\begin{align}
& \mathrm{UCB}^j_k(t) := \hat{\mu}^j_k(t) + \sqrt{\frac{6 \log T}{N^j_k(t)}}~, \label{eq:ucb_compute1} \\
& \mathrm{LCB}^j_k(t) := \hat{\mu}^j_k(t) - \sqrt{\frac{6 \log T}{N^j_k(t)}} ~. \label{eq:ucb_compute2}
\end{align}

Then, player $j$ stores the product \([\eta_{k^j_1}(t_1) \cdot \eta_{k^j_2}(t_2)]\) into \(\mathcal{P}^j_{k^j_1}\) when $k^j_1=k^j_2$. If there exists an arm $k\in [K]\setminus \mathcal{A}^j$ such that $\sum_{i\in \mathcal{P}_k^j} i \geq \lceil 0.85L_p \rceil$, i.e., too many collisions have occurred, player $j$ adds arm $k$ into $\mathcal{A}^j$ and resets $\mathcal{P}_k^j$ (Line~\ref{l:putA} in Algorithm \ref{alg:1}). 
Intuitively, a high cumulative collision count in $\mathcal{P}^j_k$ indicates that another player is exploiting arm $k$. We prove that (Lemma \ref{lm:temp-step-new} in Appendix \ref{app:proof}), with high probability, an occupied arm will be detected correctly, and a non-occupied arm will not be detected as occupied. 
After adding an arm into $\mathcal{A}^j$, the player will check whether there are too many arms in $\mathcal{A}^j$, i.e., if $|\mathcal{A}^j| > m -1$, she will start to do correction by only selecting arms from $\mathcal{A}^j$ (Line \ref{l:cor-select} in Algorithm \ref{alg:2}). 
If there exists an arm \(k \in \mathcal{A}^j \) such that \(\sum_{i \in \mathcal{Q}_k^j} i \geq \lceil 0.142L_q \rceil\), i.e., many non-collisions are observed, this suggests that arm \(k\) has likely been released by the player who previously occupied it. 
Lemma \ref{lm:temp-step-new} proves that, with high probability, a released arm will be detected correctly, and a non-released arm will not be detected as released.
In this case, player $j$ will remove that arm from $\mathcal{A}^j$, reset its $\mathcal{Q}_k^j$ and stop correction since now $|\mathcal{A}^j| \le m-1$ and is probably correct.

To switch to the exploitation phase, player $j$ needs to find an arm $k\in [K]\setminus \mathcal{A}^j$ that satisfies the following two conditions.
\begin{condition}
    $\eta_{k^j_1}(t_1)+\eta_{k^j_2}(t_2)=0$, where $k^j_1 = k^j_2 = k$. \label{eq:cd1}
\end{condition}
\begin{condition}
    $\forall \ell\neq k, \ell\in  [K]\setminus \mathcal{A}^j \text{ s.t. } \mathrm{LCB}^j_k(t) \geq \mathrm{UCB}^j_{\ell}(t)$.  \label{eq:cd2}
\end{condition}
Condition~\ref{eq:cd1} ensures that no other player is occupying the same arm \(k\) as player \(j\). The requirement of observing two consecutive collision-free pulls is crucial because, during the exploitation phase, a player may not pull her exploitation arm in every step (Line~\ref{l:2-xxxx} in Algorithm \ref{alg:2}). Therefore, a single collision-free pull does not imply that the arm is not being exploited by other players. In contrast, two consecutive non-collision steps ensure that the arm is truly unoccupied: %even if an exploiting player temporarily pulls a different arm, she 
an exploiting player will always select her exploitation arm at least once in two consecutive steps (Line \ref{l:explt1}-\ref{l:nsame_2} in Algorithm \ref{alg:2}).
Condition~\ref{eq:cd2} is a regular condition for explore-then-exploit algorithms, which guarantees that arm \(k\) is the best available option for player \(j\), i.e., its lower confidence bound dominates the upper confidence bounds of all other remaining arms in \([K] \setminus \mathcal{A}^j\). 
If both conditions are satisfied, player \(j\) sets \(\hat{k}^j = k\) and transitions to the exploitation phase. 

\begin{remark}\label{remark:c1}
    Note that our mechanism ensures that: as long as a player $j$ is in the exploration phase, for any arm $k$, the probability of choosing to pull $k$ is at most $1/m$. Specifically, when she is doing regular exploration, there are at most $m-1$ arms in $\mathcal{A}^j$, hence the probability of choosing some arm $k$ is at most $1/(K-|\mathcal{A}^j|) \le 1/m$.  On the other hand, when she is doing correction, there are at least $m$ arms in $\mathcal{A}^j$, hence the probability of choosing some arm $k$ is still at most $1/|\mathcal{A}^j| \le 1/m$. Because of this, we can obtain an upper bound for the collision probability over all the players who are doing exploration, and solve Challenge (i) described in Section \ref{sec:cont}.
\end{remark}

\begin{algorithm}[t] 
\caption{DoubleSelection (from the view of player $j$)}
\label{alg:2} 
\begin{algorithmic}[1] 
    \State Sample $Y^j \sim \text{Bernoulli}( \varepsilon)$
    % \If{$Y^j = 0$} 
    \If{$\texttt{Phase} = \texttt{Exploration}$}
    \If{$\texttt{Correction} = \texttt{False}$}
    \State $k^j_1 \sim \mathrm{Uniform}([K] \setminus \mathcal{A}^j)$ \label{l:ds-exp-1} \textcolor{violet}{\Comment{Explore unoccupied arms}}
    \State \textbf{if} {$Y^j = 1$ and $\mathcal{A}^j\neq \emptyset$} \textbf{then}
    \State $\quad\ \ k^j_2 \sim \mathrm{Uniform}(\mathcal{A}^j)$ \label{l:nsame_1}
    \State \textbf{else} $k^j_2 \leftarrow k^j_1$ \label{l:ds-exp-2} \textbf{end if}
    \Else \textcolor{violet}{\Comment{Try to quickly detect error in \(\mathcal{A}^j\)}}
    \State $k^j_1\sim \mathrm{Uniform}(\mathcal{A}^j)$
    \State $k^j_2\sim \mathrm{Uniform}(\mathcal{A}^j)$ \label{l:cor-select}
    \EndIf
    \Else
    \State $k^j_1 \leftarrow \hat{k}^j$ \label{l:explt1} \textcolor{violet}{\Comment{Exploit arm $\hat{k}^j$}}
    \State \textbf{if} {$Y^j = 1$ and $\mathcal{A}^j\neq \emptyset$} \textbf{then}
    \State $\quad\ \ k^j_2 \sim \mathrm{Uniform}(\mathcal{A}^j)$ \label{l:2-xxxx}
    \State \textbf{else} $\ k^j_2 \leftarrow k^j_1$ \textbf{end if} \label{l:nsame_2}
    \EndIf 
    \Statex \textbf{Output:} $k^j_1, k^j_2$
\end{algorithmic}
\end{algorithm}

\subsection{Exploitation Phase}

During the exploitation phase, player \(j\) selects \(k^j_1 = \hat{k}^j\) and \(k^j_2 = \hat{k}^j\) with probability \(1 - \varepsilon\). With the remaining probability \(\varepsilon\), she selects \(k^j_1 = \hat{k}^j\), and then uniformly selects an arm \(k^j_2\) from \(\mathcal{A}^j\) to pull (Line~\ref{l:explt1}-\ref{l:nsame_2} in Algorithm \ref{alg:2}). Then, player $j$ inserts $[1-\eta_{k^j_2}(t_2)]$ to the end of $\mathcal{Q}^j_{k^j_2}$ if $k^j_2$ is sampled from \(\mathcal{A}^j\) (Line \ref{l:p_col-0} %-\ref{l:p_col} 
in Algorithm \ref{alg:1}). 
If there exists an arm \(k \in \mathcal{A}^j\) such that \(\sum_{i \in \mathcal{Q}_k^j} i \geq \lceil 0.142L_q \rceil\), arm \(k\) is considered released, and player \(j\) removes \(k\) from \(\mathcal{A}^j\), resets \(\mathcal{Q}_k^j\) to empty, and compares \(\mathrm{LCB}^j_{\hat{k}^j}(t)\) with \(\mathrm{UCB}^j_k(t)\). 
If \(\mathrm{LCB}_{\hat{k}^j}(t) < \mathrm{UCB}_k(t)\), i.e., arm $k$ might be better than her exploitation arm $\hat{k}^j(t)$, she then sets \(\hat{k}^j = 0\) and returns to the exploration phase (Line \ref{l:return} in Algorithm \ref{alg:1}). Otherwise, it implies that arm \(\hat{k}^j\) is better than \(k\), and player $j$ will continue exploiting $\hat{k}^j$.

\begin{remark}
    Our algorithm lets the players keep updating $\mathcal{A}^j$ even in the exploitation phase, by pulling arms in $\mathcal{A}^j$ with probability $\varepsilon$. The set $\mathcal{A}^j$ is an estimation of current active players who are doing exploitation, which can be regarded as a lower bound for current active players. Hence, a correct estimation of $\mathcal{A}^j$ guarantees that her exploitation arm is good enough. By doing trade-off on parameter $\varepsilon$, we solve Challenge (ii) described in Section \ref{sec:cont}, i.e., the player can use limited cost to obtain a sufficiently accurate estimation of current players, and thus avoid the potential high cost of missing the exact optimal arm in the exploitation phase. 
\end{remark}

\section{Theoretical Analysis} \label{sec:th}

This section presents a theoretical analysis of the proposed algorithm ACE, establishing the following regret bound of $\mathcal{O}(\sqrt{T \log T} + \log T / \Delta^2)$.

\begin{theorem}
Let $\varepsilon = \min\{ \sqrt{\frac{1141m^3\ln(T) }{2T}}, \frac{1}{K}, \frac{1}{10} \}$. Then given $K$ arms and $M$ players, the regret of Algorithm \ref{alg:1} is bounded by
\begin{align*}
    R(T) \leq & \frac{576emKM\log(T)}{\Delta^2} + 96m^{3/2}M\sqrt{T\ln(T)} \\
    & + 7704m^2KM\ln(T) + (4emKM)^2  ~,
\end{align*}
where $\Delta := \min_{k\leq m}(\mu_k- \mu_{k+1})$.
\label{th:upper}
\end{theorem}

\update{The following provides a sketch of the proof for Theorem~\ref{th:upper}, and the complete version is deferred to Appendix~\ref{app:proof}.}

\begin{proofS}
Let \(\mathcal{T}^j_{\mathrm{exp}}\), \(\mathcal{T}^j_{\mathrm{explt}}\) denote the sets of time steps during which player \(j\) is in the exploration, exploitation phases, respectively. Define \(T^j_{\mathrm{exp}} := |\mathcal{T}^j_{\mathrm{exp}}|\), \(T^j_{\mathrm{explt}} := |\mathcal{T}^j_{\mathrm{explt}}|\), %, \(T^j_{\mathrm{cor}} := |\mathcal{T}^j_{\mathrm{cor}}|\) 
and $\mathcal{T}^j := \mathcal{T}^j_{\mathrm{exp}} \cup \mathcal{T}^j_{\mathrm{explt}}$.
With slight abuse of notation, we denote by \(\mathcal{A}^j(t)\) the set of occupied arms from the view of player \(j\) at time \(t\).  
Define the following event:
\begin{align*}
    & \mathcal{E}_0 := \left\{ \exists t\in \mathcal{T}^j, j\leq M, k\leq K : |\hat{\mu}^j_k(t) - \mu_k |\geq \sqrt{\frac{6\log(T)}{N^j_k(t)}} \right\} ~,
\end{align*}
and two sets of time steps:
\begin{align*}
    & \mathcal{G}^{j}_1 := \left\{ t\in \mathcal{T}^j: \exists j^{\prime}\neq j, j^{\prime}\in[M],\exists k\leq K, k\notin \mathcal{A}^j(t), \hat{k}^{j^{\prime}} = k \right\} ~,\\
    & \mathcal{G}^{j}_2 := \left\{ t\in \mathcal{T}^j: \exists k\in \mathcal{A}^j(t), \forall j^{\prime} \neq j, j^{\prime} \in [M], \hat{k}^{j^{\prime}}\ne k \right\} ~,% \\
    %& \mathcal{G}^{j}_3 := \left\{T^j_\mathrm{start} \le  t \le T^j_\mathrm{end}: t\in \mathcal{T}^j_{\mathrm{cor}} \right\} ~.
\end{align*}
Here $\mathcal{E}_0$ denotes that the estimated reward significantly deviates from the expected reward at some time step. \(\mathcal{G}_1^j\) denotes the set of time steps during which an arm \(k\) is occupied by player \(j'\) but has not yet been discovered by player \(j\). \(\mathcal{G}_2^j\) denotes the set of time steps during which an arm \(k \in \mathcal{A}^j(t)\) has been released but but remains undiscovered to player \(j\). %\(\mathcal{G}_2^j\) denotes the set of time steps when player $j$ is during the correction phase. 

Then we can decompose the regret as follows:
\begin{align}
    \!\!\!\!\!\!\!\!\!\!R(T) &= \sum_{t=1}^T \sum_{k \leq m_t}\mu_k -\mathbb{E}\left[\sum_{t=1}^{T}\sum_{j: T_\mathrm{start}^j \le t \le T_\mathrm{end}^j } r^j(t)\right] \label{eq:new-x} \\
    &\update{= \sum_{t=1}^T \sum_{k =1}^{m_t} \mu_k -\mathbb{E}\left[\sum_{t=1}^{T}\sum_{j: T_\mathrm{start}^j \le t \le T_\mathrm{end}^j } X_{\pi^j(t)}(t) [1- \eta^j(t) ] \right]} \nonumber \\
    &\update{\le \sum_{t=1}^T \sum_{k =1}^{m_t} \mu_k -\mathbb{E}\left[\sum_{t=1}^{T}\sum_{j: T_\mathrm{start}^j \le t \le T_\mathrm{end}^j } X_{\pi^j(t)}(t) [1- \eta^j(t) ] \mathds{1}[ \pi^j(t) \le m_t ] \right]} \label{eq:new-y} \\
    &\update{= \sum_{t=1}^T \sum_{k =1}^{m_t} \mu_k - \sum_{t=1}^{T}\sum_{j: T_\mathrm{start}^j \le t \le T_\mathrm{end}^j } \mu_{\pi^j(t)} \mathbb{E}\left[ \mathds{1}[\eta^j(t)=0 , \pi^j(t) \le m_t ] \right]} \nonumber \\
    &\update{= \sum_{t=1}^T \sum_{k =1}^{m_t} \mu_k\left(\mathbb{E}\left[1 - \sum_{j: T_\mathrm{start}^j \le t \le T_\mathrm{end}^j } \mathds{1}[\eta^j(t)=0 , \pi^j(t) = k,\pi^j(t) \le m_t ] \right]\right)}  \label{eq:new1} \\
    &\update{\le  \sum_{t=1}^T \sum_{k =1}^{m_t} \left(\mathbb{E}\left[1 - \sum_{j: T_\mathrm{start}^j \le t \le T_\mathrm{end}^j } \mathds{1}[\eta^j(t)=0 , \pi^j(t) = k,\pi^j(t) \le m_t ] \right]\right)} \nonumber \\
    &\le \sum_{t=1}^T \left(m_t - \mathbb{E}\left[\sum_{j: T_\mathrm{start}^j \le t \le T_\mathrm{end}^j } \mathds{1}[\pi^j(t) \le m_t, \eta^j(t) = 0] \right] \right) \label{eq:dcmp-1} \\
    &\le \sum_{t=1}^T \sum_{j: T_\mathrm{start}^j \le t \le T_\mathrm{end}^j }  \mathbb{E}\left[1 - \mathds{1}[\pi^j(t) \le m_t, \eta^j(t) = 0] \right] \label{eq:dcmp-1.5}  \\
    &\le \sum_{j \le M}\sum_{t=T_\mathrm{start}^j}^{T_\mathrm{end}^j}   \mathbb{E}\left[1 - \mathds{1}[\pi^j(t) \le m_t, \eta^j(t) = 0] \right] \label{eq:dcmp-2} \\
    &\le \sum_{j \le M}\sum_{t=T_\mathrm{start}^j}^{T_\mathrm{end}^j}   \mathbb{E}\left[1 - \mathds{1}[\pi^j(t) \le m_t, \eta^j(t) = 0] \,\middle|\, \overline{\mathcal{E}_0}\, \right] + \sum_{j \le M} T^j \Pr[\mathcal{E}_0] \label{eq:new-dcmp} \\
    &\le \underbrace{\sum_{j \le M} \mathbb{E}\left[|\mathcal{G}^{j}_2| \,\middle|\, \overline{\mathcal{E}_0} \right]}_{\va} 
    + \underbrace{\sum_{j \le M}\sum_{t \in \mathcal{T}^j_{\mathrm{exp}}}  \mathbb{E}\left[\mathds{1}[t \notin \mathcal{G}^{j}_1 \cup \mathcal{G}^{j}_2] \,\middle|\, \overline{\mathcal{E}_0} \right]}_{\vb} \nonumber \\
    &\quad+ \underbrace{\sum_{j \le M}\sum_{t \in \mathcal{T}^j_{\mathrm{exp}}}  \mathbb{E}\left[\mathds{1}[t \in \mathcal{G}^{j}_1] \,\middle|\, \overline{\mathcal{E}_0} \right]}_{\vc} +\underbrace{\sum_{j \le M} T^j \Pr[\mathcal{E}_0]}_{\ve} \nonumber \\
    &\quad+ \underbrace{\sum_{j \le M}\sum_{t \in \mathcal{T}^j_{\mathrm{explt}}}  \mathbb{E}\left[\left(1 - \mathds{1}[\pi^j(t) \le m_t, \eta_{\pi^j(t)}(t) = 0]\right)\mathds{1}[t \notin \mathcal{G}^{j}_2] \,\middle|\,  \overline{\mathcal{E}_0} \right]}_{\vd} ~. \nonumber
\end{align}
\update{The intuition from \eqref{eq:new-x} to \eqref{eq:dcmp-1} follows from the facts that} $\mu_k \le 1$ for all $k \in [K]$ and $\mu_1 > \mu_2 > \dots > \mu_K$, which implies that the first $m_t$ arms are optimal. Consequently, the regret at each time step is decomposed into the total number of players minus the number of players who select arms $\pi^j(t) \le m_t$ and do not experience collisions. Specifically, \eqref{eq:new-y} results from the omission of the event $\pi^j(t) > m_t$ and \eqref{eq:new1} holds because there is at most one player $j$ with $\eta^j(t)=0, \pi^j(t) = k,\pi^j(t) \le m_t$, and $\mu_k \le 1$. 
%
%\update{For the left decompisition,} 
Then \eqref{eq:dcmp-1.5} is because that $|\{j: T_\mathrm{start}^j \le t \le T_\mathrm{end}^j \}| = m_t$, and \eqref{eq:dcmp-2} holds by exchanging the summation. %\update{The detailed decomposition from \eqref{eq:new-dcmp} to the final inequality can be found in Appendix \ref{app:dcmp}.}

By Hoeffding's Inequality (Lemma \ref{lm:hfd}), $\Pr[\mathcal{E}_0] \leq 2KM/T$. Thus, \ve\ is upper bounded by $2KM^2$.

\(\vb\) corresponds to the regret from exploring to identify optimal arms when the occupied arm set \(\mathcal{A}^j(t)\) is correctly estimated (i.e., \(t \notin \mathcal{G}^j_1 \cup \mathcal{G}^j_2\) ). The following lemma provides an upper bound of $\vb$.
\begin{restatable}{lemma}{Appvb}
\label{lm:vb}
Given \(K\) arms and \(M\) players, \(\vb\) is bounded as
\[
\vb \leq \frac{576emKM\log(T)}{\Delta^2} + 12e^2m^2K^2M + 2KM^2 + \sum_{j\leq M} \varepsilon T^j_{\mathrm{exp}} ~.
\]
\end{restatable}
The bound follows from a standard \(\mathcal{O}(\log(T)/\Delta^2)\) sample complexity for distinguishing two arms. The fact that there are totally \(M\) players and each of them needs to explore \(K\) arms contributes the \(KM\) factor. The additional multiplicative factor \(m\) accounts for repeated exploration. For example, at some time step $t$, \(m-1\) players are exploiting the first \(m-1\) arms. Player \(j\) then joins, explores every arm in $\{m,m+1,\cdots,K\}$ for \(\mathcal{O}(\log(T)/\Delta^2)\) times, and enters exploitation phase with $\hat{k}^j=m$ and $\mathcal{A}^j(t) = \{1,2,\cdots,m-1\}$. 
Then a player $j^\prime$ leaves the system and releases arm \(m-1\).
After a while, player \(j\) finds that arm \(m-1\) is released, removes arm $m-1$ from $\mathcal{A}^j(t)$ and re-enters the exploration phase. 
Now she still needs to uniformly pull all the available arms again, including previously distinguished sub-optimal ones (as highlighted in Remark \ref{remark:c1}, the uniform exploring is very important). 
That is, to distinguish that arm $m-1$ is better than arm $m$, she needs to pull all the arms in $\{m,m+1,\cdots,K\}$ for another \(\mathcal{O}(\log(T)/\Delta^2)\) times. 
This process can repeat up to \(m\) times, resulting in the multiplicative factor $m$. The last term \(\sum_{j \leq M} \varepsilon T^j_{\mathrm{exp}}\) arises from the process in which players pull occupied arms in \(\mathcal{A}^j(t)\) with probability \(\varepsilon\).

\va\ and \vc\ denote the regret incurred due to incorrect estimation of the occupied arm set $\mathcal{A}^j$. These two terms are jointly bounded as follows:
\begin{restatable}{lemma}{Appvavc}
\label{lm:va-vc}
Given \(K\) arms and \(M\) players, \(\va+\vc\) is bounded as
\begin{equation*}
    \va + \vc  \leq  \frac{1141m^3M\ln(T)}{\varepsilon} +  3852m^2KM\ln(T) + 4KM^2 ~.
\end{equation*}
\end{restatable}

The regret \(\va\) arises when an arm \(k\) has been released, but players fail to detect this in time. As a result, they may continue to exploit a sub-optimal arm. 
One key observation here is that, if \(k \in \mathcal{A}^j(t)\) is released, when player \(j\) pulls $k$, %during the process of exploring whether arms in $\mathcal{A}^j(t)$ are still occupied, 
the non-collision probability increases from \( \varepsilon\) to approximately \(1/2e\). 
Because of this, after \(1141\ln T\) times of pulling arm $k$, player $j$ can be almost sure that arm $k$ is released. 
Since the probability of choosing this specific arm $k$ is at least $\varepsilon/m$, this period can last for  \(1141m\ln T / \varepsilon \) in expectation. 
Another important observation is that releasing arms can only happen due to a permanent departure of one player.
Each departure can cause at most $m$ times of releasing (i.e., after the other players realize that the arm is released, they may also change to exploration phase, and thus releasing their exploiting arms), and it can cause at most $m^2$ times of deleting arms in $\mathcal{A}^j(t)$ (taking sum over all the players). 
Therefore, the total number of such deletions (over all the players) is at most $\mathcal{O} (m^2M)$, and there is another factor of $\mathcal{O}(m^2M)$ in the first term of Lemma \ref{lm:va-vc}.

On the other hand, \(\vc\) captures the regret when an arm \(k\) is currently occupied but mistakenly excluded from \(\mathcal{A}^j(t)\). In this case, player \(j\) may pull it during exploration, leading to wasted effort and collisions. 
If \(k \in \mathcal{A}^j(t)\) is occupied, when player \(j\) pulls $k$, the collision probability increases from \(1-1/2e\) to \(1 - \varepsilon\).
Because of this, after \(1926\ln T\) times of pulling arm $k$, player $j$ can be almost sure that arm $k$ is occupied. 
Since the probability of choosing this specific arm $k$ is at least $1/K$, this period can last for  \(1926K \ln (T)\) in expectation. 
Similar to term \(\va\), we still need another factor of $\mathcal{O}(m^2M)$ to count for all such additions, and this leads to the second term in Lemma \ref{lm:va-vc}.

\(\vd\) denotes the regret incurred during the exploitation phase when \(t \notin \mathcal{G}^j_2\). Note that if \(t \notin \mathcal{G}^j_2\), then our algorithm makes sure that $|\mathcal{A}^j(t)| < m_t$, and thus player $j$ must be exploiting an arm $k \le m_t$. Also, no other player is able to exploit the same arm $k$ at time step $t$. In this case, the regret can appears only if: i) a player $j'$ is in the exploring phase, and she does not realized that arm $k$ is occupied; ii) a player $j'$ is trying arms in $\mathcal{A}^{j'}(t)$ with probability $\epsilon$ and collides with player $j$; iii) player $j$ is pulling arms in $\mathcal{A}^j$ with probability $\varepsilon$. Because of this, we have the following lemma. Here the first term captures the regret from case i), using a similar technique in the proof of Lemma \ref{lm:va-vc}; the third term arises because of case ii) and iii).

\begin{restatable}{lemma}{Appvd}
\label{lm:vd}
Given \(K\) arms and \(M\) players, \(\vd\) is bounded as
\begin{equation*}
    \vd \leq 3852m^2KM\ln(T) + 2KM^2 + \sum_{j \le M}  \varepsilon  ( \max_{j^\prime\leq M} T^{j^\prime}_{\mathrm{explt}} + T^j_{\mathrm{explt}} ) ~.
\end{equation*}
\end{restatable}

Putting all the terms together and setting $\varepsilon = \min\{ \sqrt{\frac{1141m^3\ln(T) }{2T}},$ $\frac{1}{K}, \frac{1}{10} \}$, we obtain the final regret bound as stated. %This completes the proof of Theorem \ref{th:upper}
\end{proofS}

\begin{remark}\label{remark:het}
While our theoretical analysis is conducted under the homogeneous reward setting, ACE can be applied to heterogeneous reward scenarios in practice, since each player independently explores and exploits arms based on her own feedback. A formal regret analysis under heterogeneous rewards is left as future work. 
\end{remark}

\begin{figure*}[th]
    \centering
    \begin{subfigure}[b]{0.3\textwidth}
        \includegraphics[width=\textwidth]{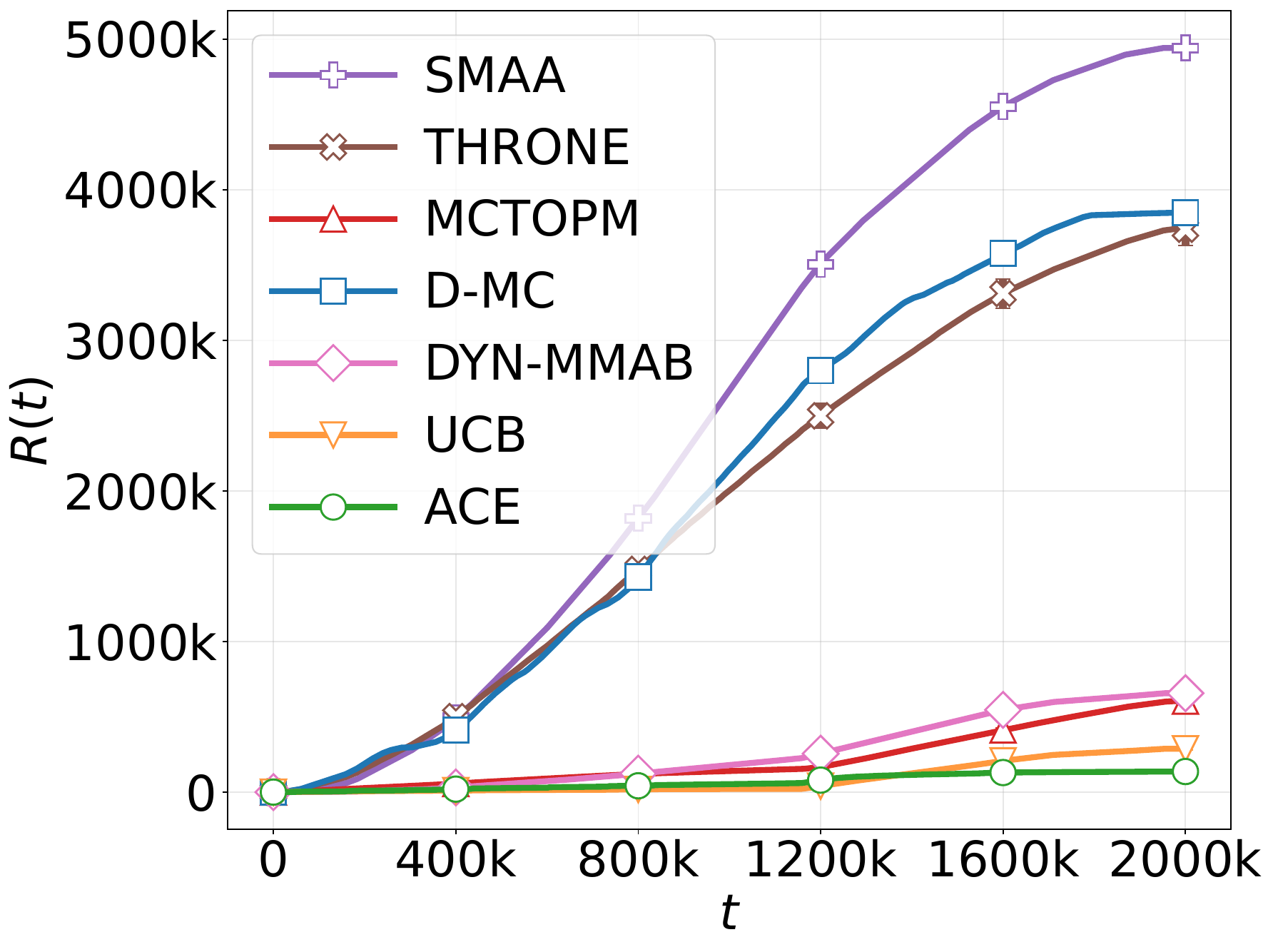}
        \caption{K=20, random.}
        \label{fig:asyn-rdm}
    \end{subfigure}
    \hfill
    \begin{subfigure}[b]{0.3\textwidth}
        \includegraphics[width=\textwidth]{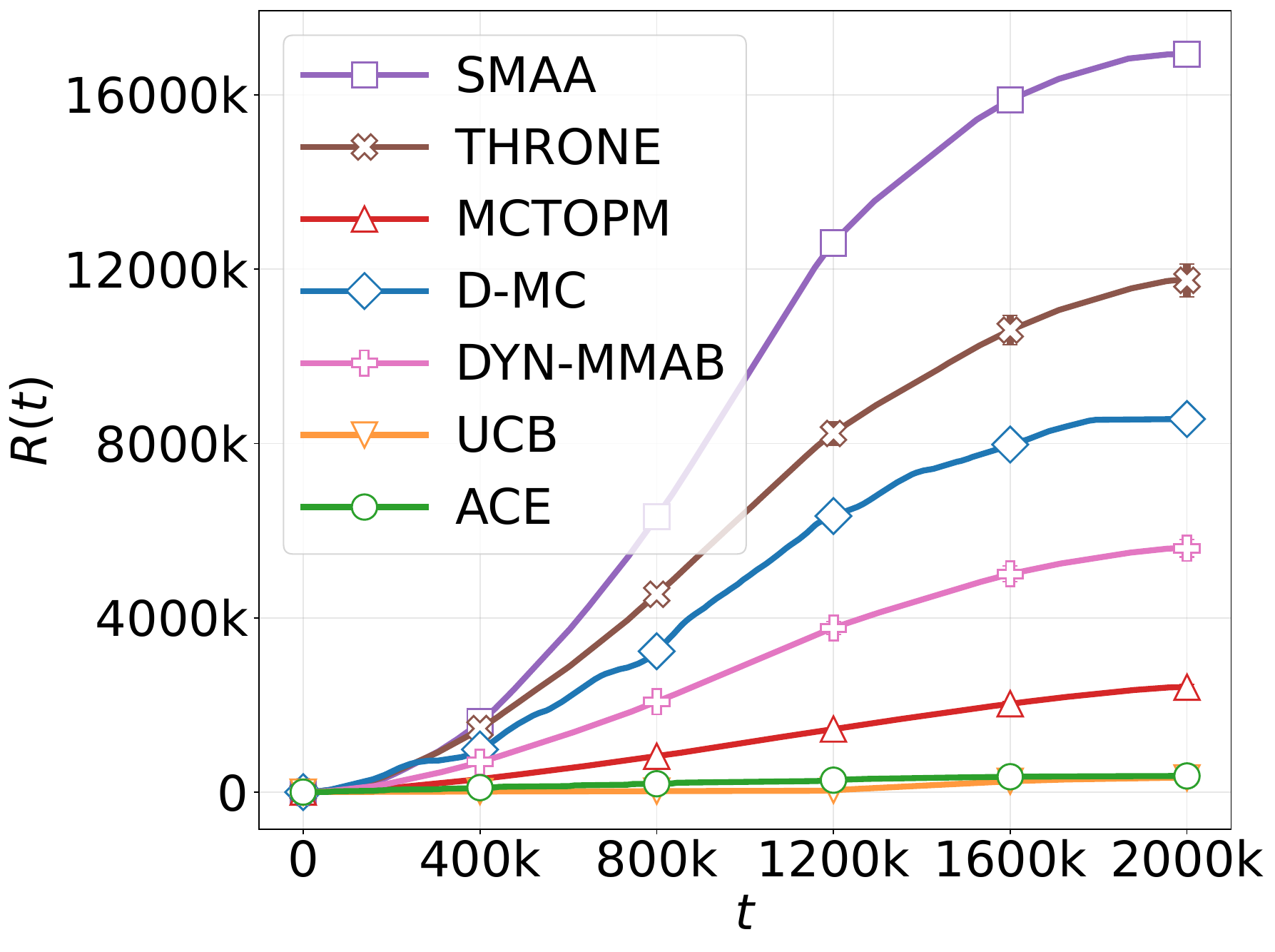}
        \caption{K=50, random.}
        \label{fig:r-k-50}
    \end{subfigure}
    \hfill
    \begin{subfigure}[b]{0.3\textwidth}
        \includegraphics[width=\textwidth]{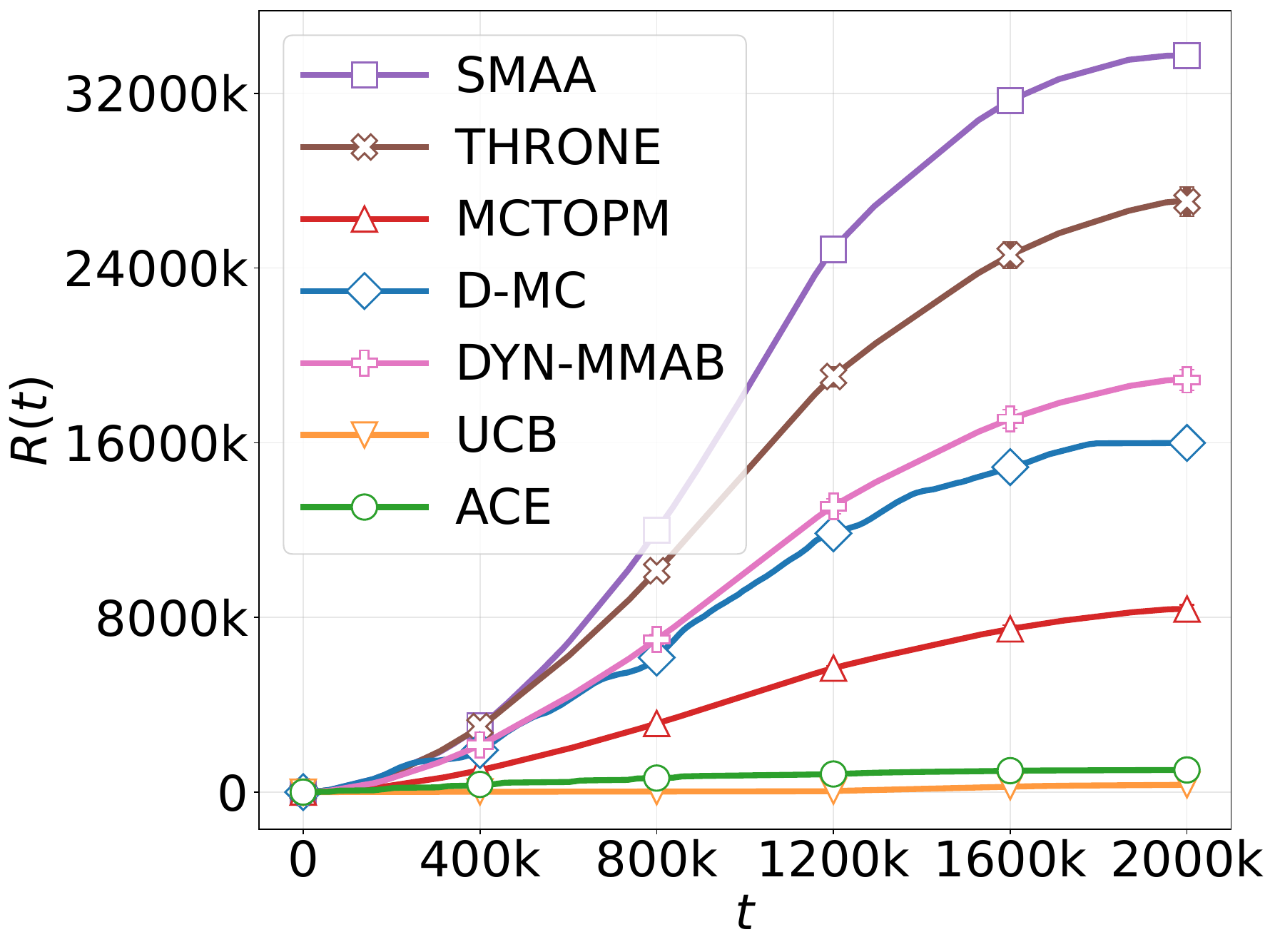}
        \caption{K=100, random.}
        \label{fig:r-k-100}
    \end{subfigure}
    \\[1.5em]
    \begin{subfigure}[b]{0.3\textwidth}
        \includegraphics[width=\textwidth]{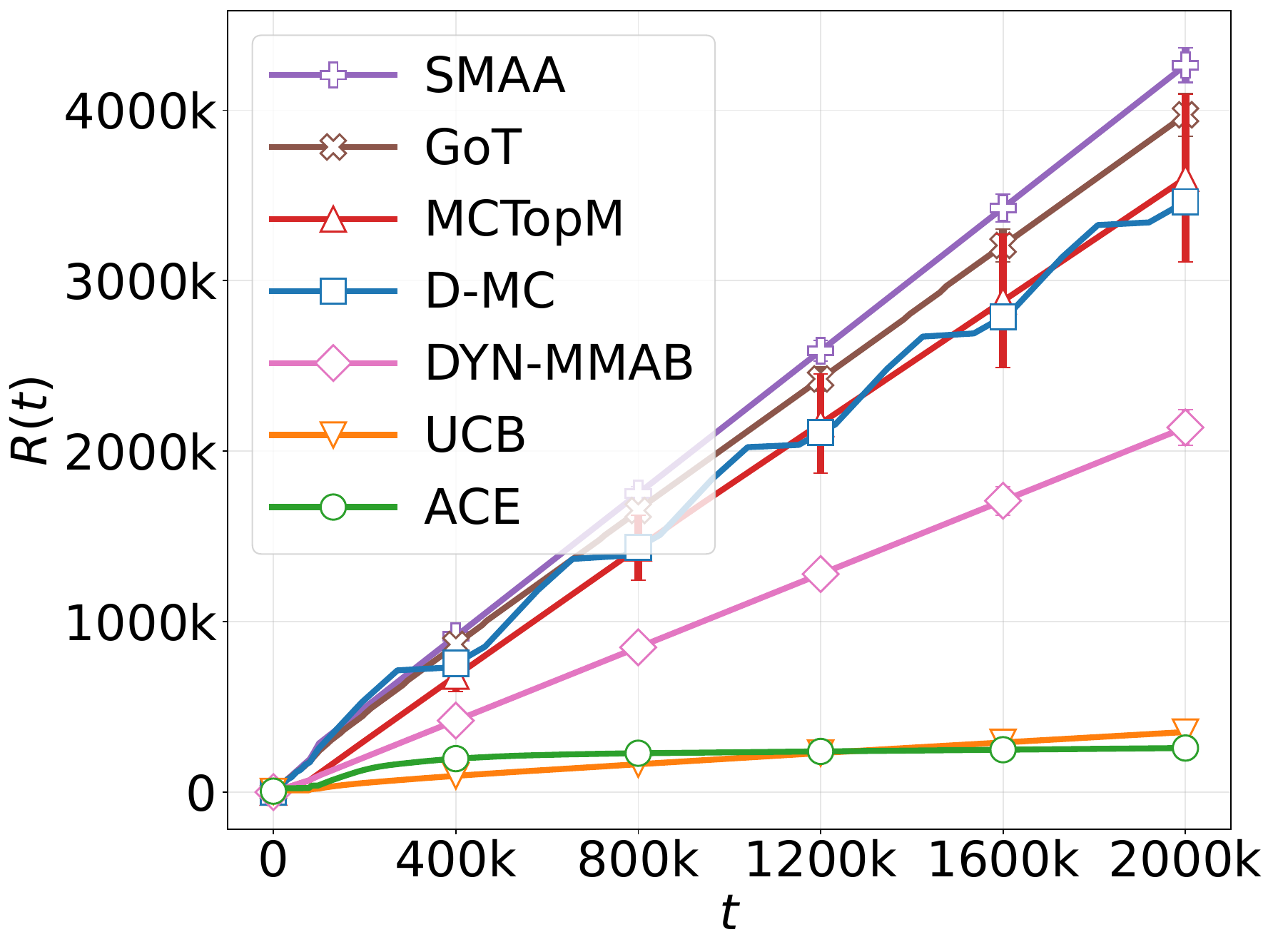}
        \caption{K=20, synthetic.}
        \label{fig:asyn-syn}
    \end{subfigure}
    \hfill
    \begin{subfigure}[b]{0.3\textwidth}
        \includegraphics[width=\textwidth]{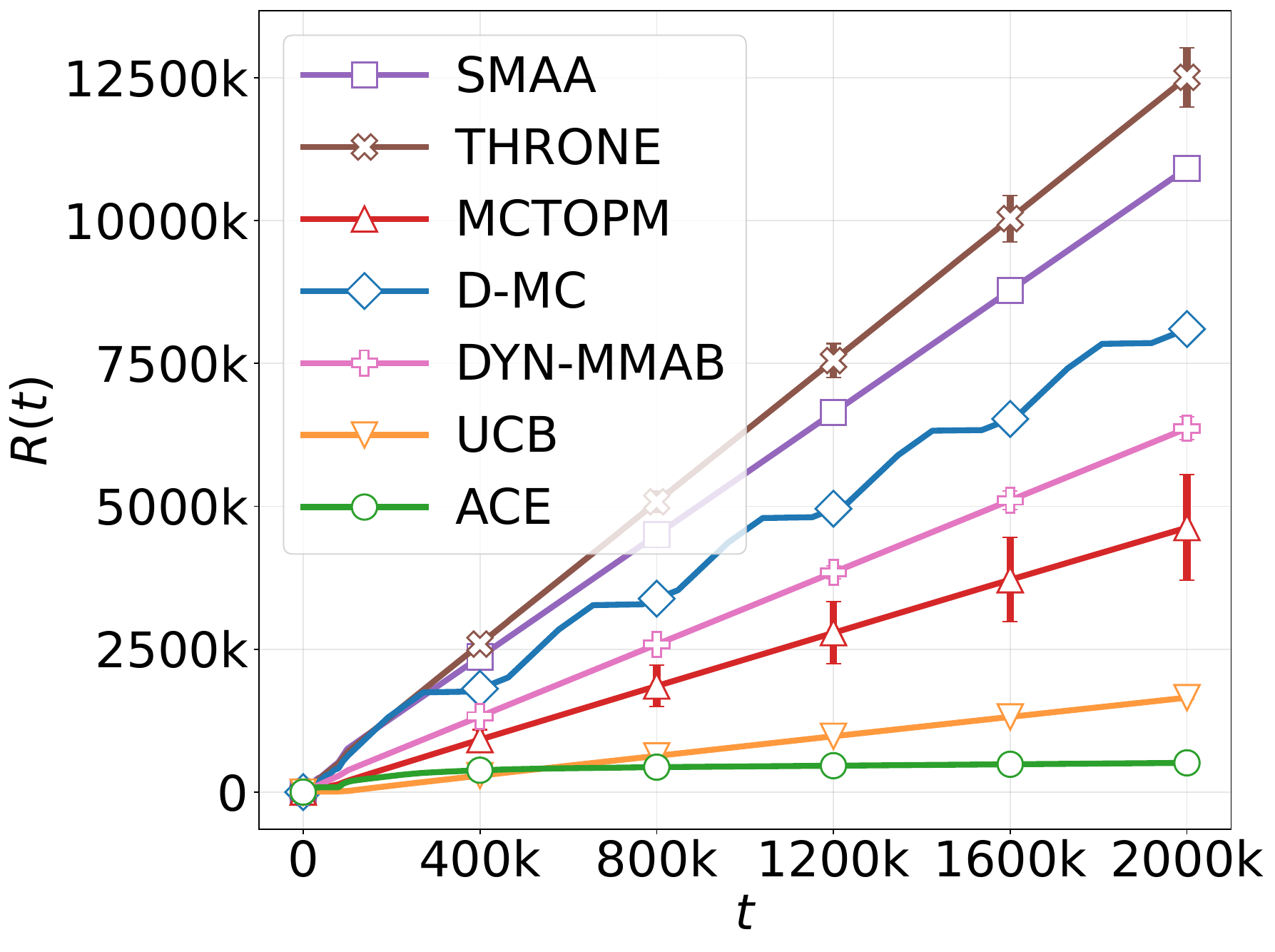}
        \caption{K=50, synthetic.}
        \label{fig:s-k-50}
    \end{subfigure}
    \hspace{0.02\textwidth}
    \begin{subfigure}[b]{0.3\textwidth}
        \includegraphics[width=\textwidth]{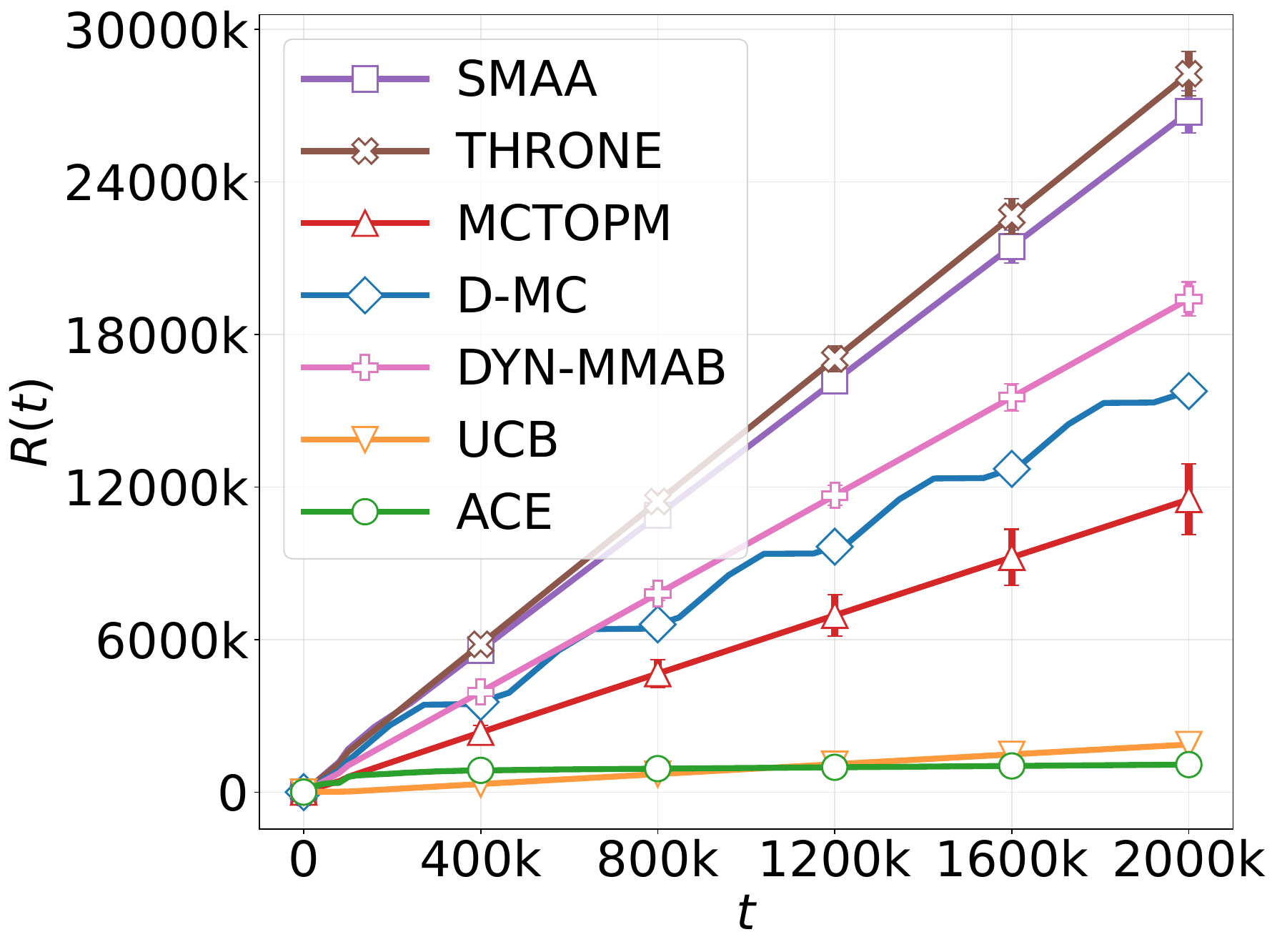}
        \caption{K=100, synthetic.}
        \label{fig:s-k-100}
    \end{subfigure}
    \vspace{0.2cm}
    \caption{Comparison of cumulative regret for different numbers of arms \(\mathbf{K}\) under different asynchronization settings.}
    \Description{Comparison of cumulative regret for different numbers of arms \(K\) under different asynchronization settings.}
    \label{fig:k}
\end{figure*}

\section{Experiments}

We conduct a series of experiments to validate our theoretical findings. Each experiment is independently repeated \update{50 times}, and the resulting standard error across runs is visualized using error bars in the plots. The proposed algorithm, ACE, is compared with Dynamic Musical Chair (D-MC) \citep{rosenski2016multi}, Game of Thrones (GoT) \citep{bistritz2018distributed}, MCTopM \citep{besson2018multi}, DYN-MMAB \citep{boursier2019sic}, SMAA \citep{xu2023competing}, \update{ SefishUCB (UCB) \citep{besson2018multi}, and Randomized Sefish UCB (RD-UCB) \citep{trinh2021high}.}
This section presents comparisons across different numbers of arms $K$ under varying asynchronous settings. Additional results for various values of players \(M\) and implementation details are reported in Appendix~\ref{app:exp}.

\begin{table}[ht]
\centering
% ---- Row 1 ---- (flipped data)
\begin{subtable}[t]{0.23\textwidth}
    \centering
    \begin{tabular}{ccc}
    \toprule
    \textbf{j} & \textbf{Start} & \textbf{End} \\
    \midrule
        1 & 431945 & 1291229 \\
        2 & 304242 & 1524756 \\
        3 & 181824 & 1183404 \\
        4 & 832442 & 1212339 \\
        5 & 20584  & 1969909 \\
        6 & 601115 & 1708072 \\
        7 & 58083  & 1866176 \\
        8 & 156018 & 1155994 \\
        9 & 731993 & 1598658 \\
        10 & 374540 & 1950714 \\
    \bottomrule
    \end{tabular}
    \\[1em]
\caption{Random setting.}
\label{tb:asyn-set-r}
\end{subtable}
\hfill
\begin{subtable}[t]{0.23\textwidth}
    \centering
    \begin{tabular}{ccc}
    \toprule
    \textbf{j} & \textbf{Start} & \textbf{End} \\
    \midrule
        1 & 1 & 100000 \\
        2 & 1 & 100000 \\
        3 & 1 & 100000 \\
        4 & 1 & 100000 \\
        5 & 80000 & 2000000 \\
        6 & 80000 & 2000000 \\
        7 & 80000 & 2000000 \\
        8 & 80000 & 2000000 \\
        9 & 1 & 2000000 \\
        10 & 1 & 2000000 \\
    \bottomrule
    \end{tabular}
    \\[1em]
\caption{Synthetic setting.}
\label{tb:asyn-set-s}
\end{subtable}
\caption{Players' active periods for the comparison across different asynchronization settings.}
\label{tb:asyn-set}
\end{table}

\textbf{Setup} $\ \ $ The experiments comparing different asynchronous settings are conducted in a Gaussian bandit environment. \update{Note that the reason why we choose Gaussian bandits rather than Bernoulli Bandits is to maintain consistent reward gaps across different numbers of arms. }
Specifically, the reward of each arm $k$ follows a Gaussian distribution \(\mathcal{N}(\mu_k, 0.5^2)\), where the smallest mean \(\mu_K\) is \(0.1\) and the gap between adjacent arms is fixed at \(0.05\). In all experiments, the number of players is fixed as \(M=10\). We evaluate the performance under \(K=20\), \(K=50\), and \(K=100\).

For the asynchronous setting, we consider two types: \textit{random} and \textit{synthetic}. 
In the random setting, each player $j$ is active between two time steps \(T^j_{\mathrm{start}} \in [0, T/2]\) and \(T^j_{\mathrm{end}} \in [T/2, T]\), which are selected uniformly at random, subject to \(T^j_{\mathrm{end}} - T^j_{\mathrm{start}} \geq T/M\). The synthetic setting, on the other hand, is manually constructed to simulate a challenging scenario: certain players holding optimal arms leave the system in the middle of the game, and the remaining players may continue exploiting arms that are no longer optimal. 
This challenges the algorithms' capacity of adaptively choosing exploitation arms. \update{The detailed active periods %for all players in the random and synthetic asynchronization setting 
are listed in Table~\ref{tb:asyn-set}. }

\begin{figure*}[th]
    \centering
    \begin{subfigure}[b]{0.3\textwidth}
        \includegraphics[width=\textwidth]{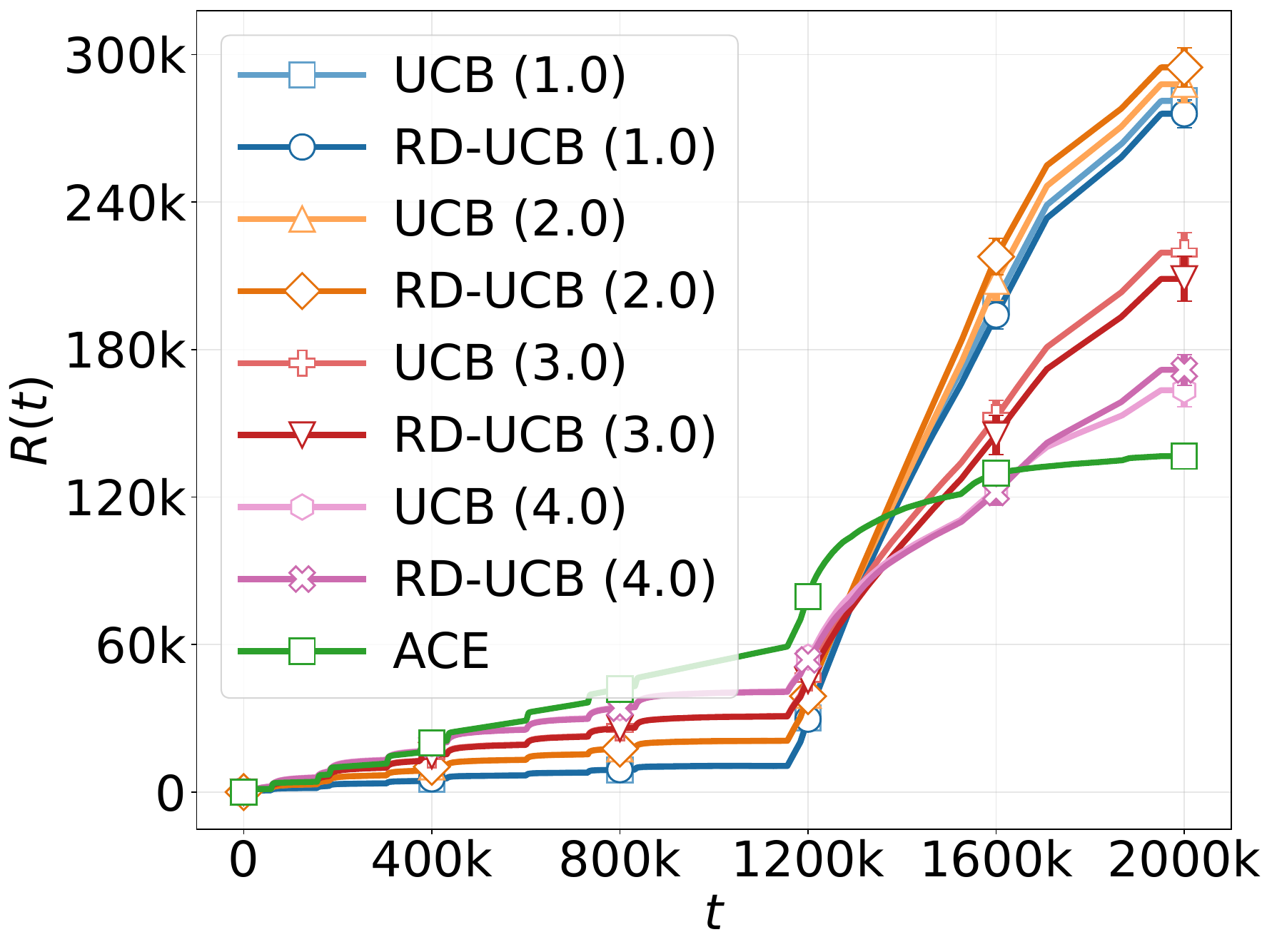}
        \caption{K=20, random, with UCBs.}
        % \label{fig:ucb-asyn-rdm}
    \end{subfigure}
    \hfill
    \begin{subfigure}[b]{0.3\textwidth}
        \includegraphics[width=\textwidth]{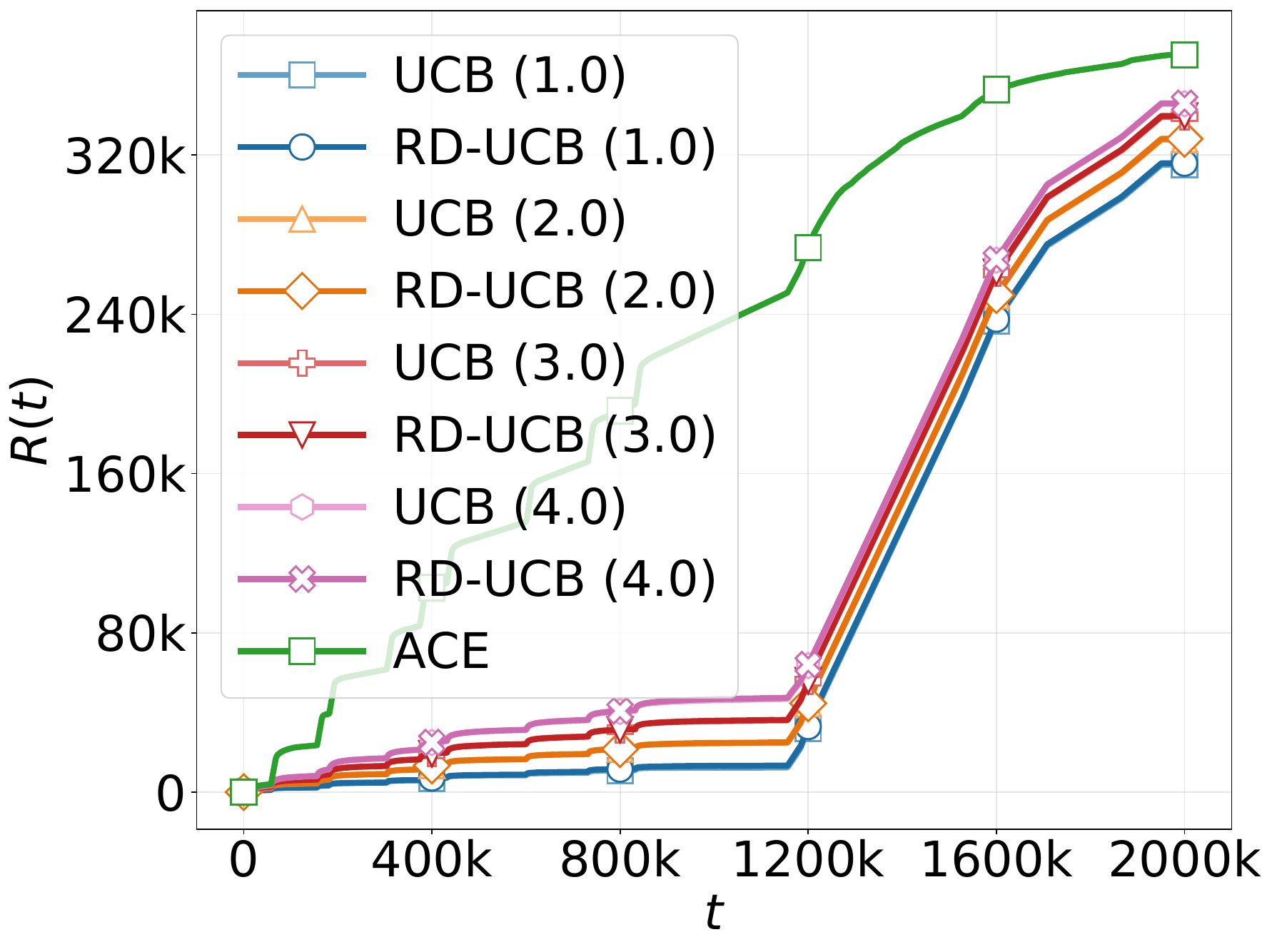}
        \caption{K=50, random, with UCBs.}
        % \label{fig:ucb-asyn-syn}
    \end{subfigure}
    \hfill
    \begin{subfigure}[b]{0.3\textwidth}
        \includegraphics[width=\textwidth]{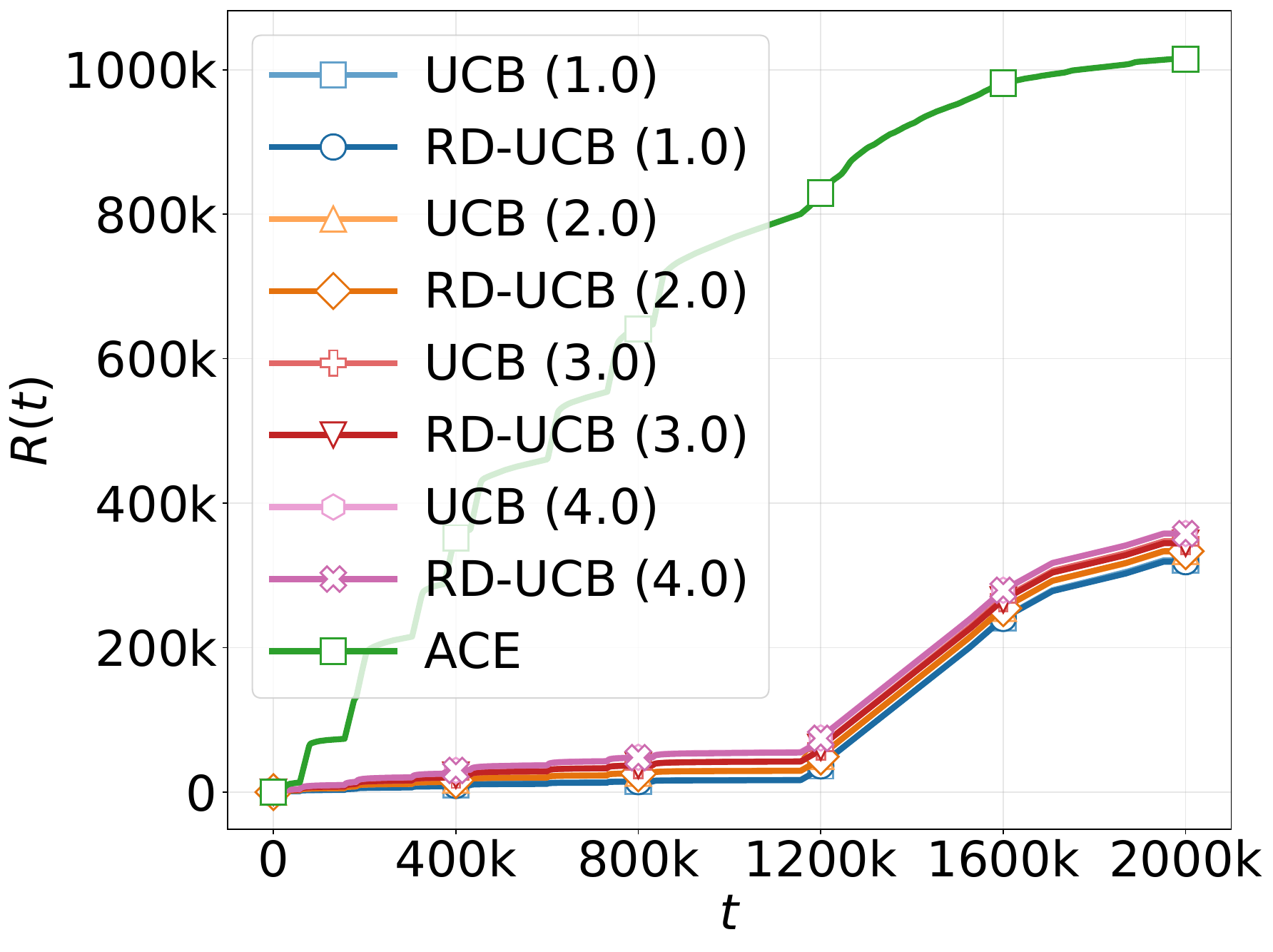}
        \caption{K=100, random, with UCBs.}
        % \label{fig:ucb-asyn-syn}
    \end{subfigure}
    \\[1.5em]
    \begin{subfigure}[b]{0.3\textwidth}
        \includegraphics[width=\textwidth]{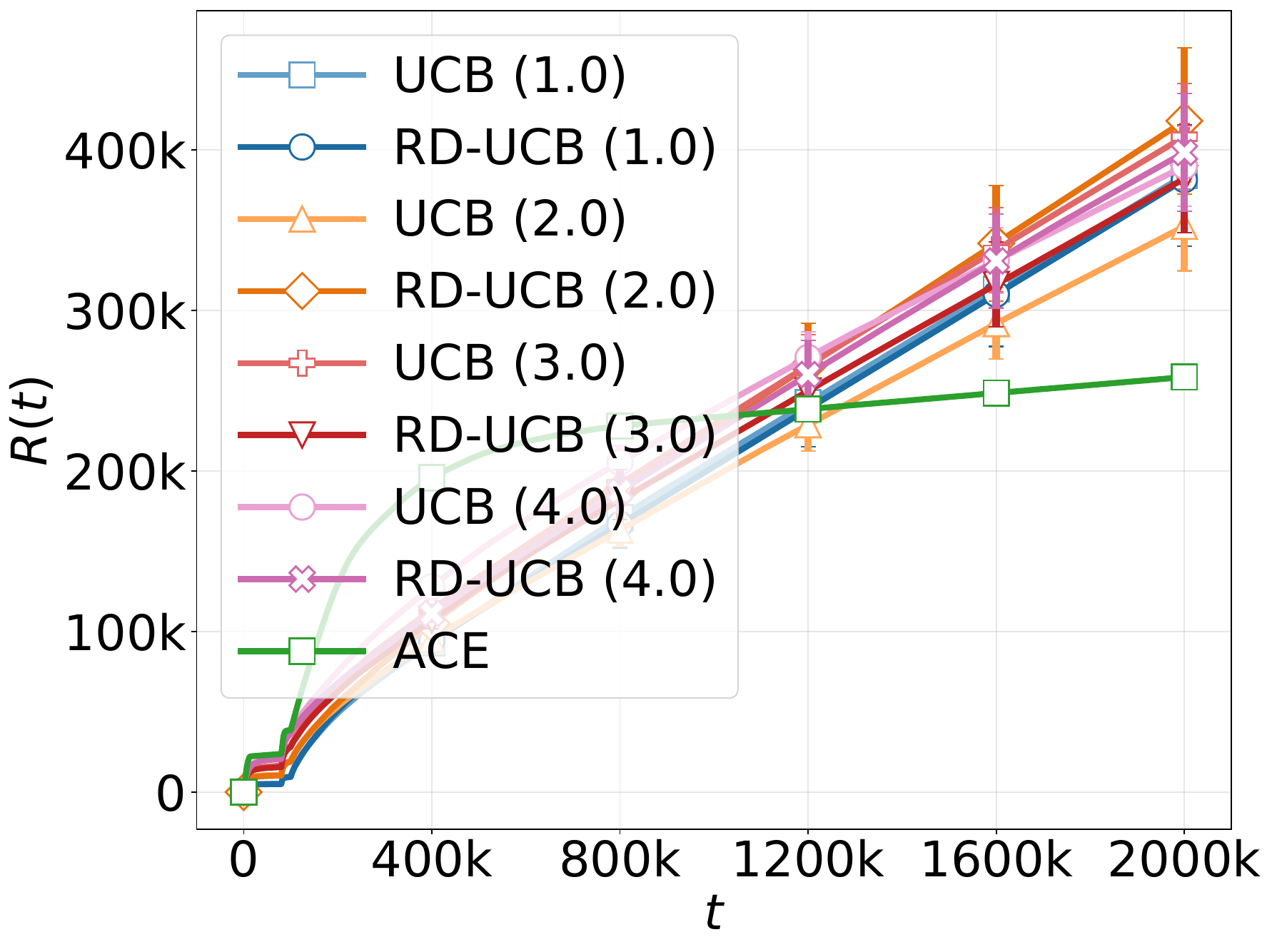}
        \caption{K=20, synthetic, with UCBs.}
        \label{fig:ucb-asyn-syn}
    \end{subfigure}
    \hfill
    \begin{subfigure}[b]{0.3\textwidth}
        \includegraphics[width=\textwidth]{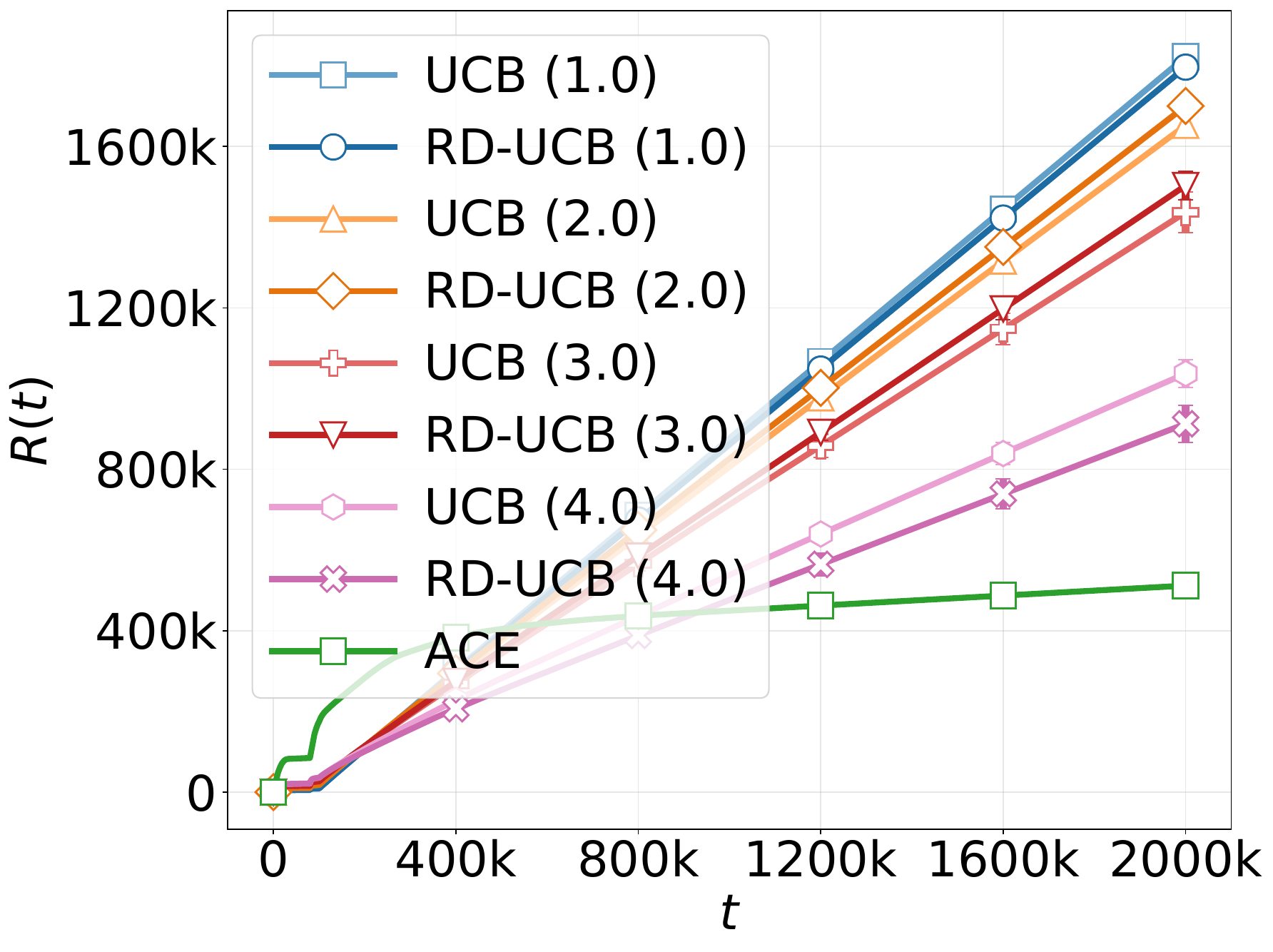}
        \caption{K=50, synthetic, with UCBs.}
        % \label{fig:ucb-asyn-syn}
    \end{subfigure}
    \hfill
    \begin{subfigure}[b]{0.3\textwidth}
        \includegraphics[width=\textwidth]{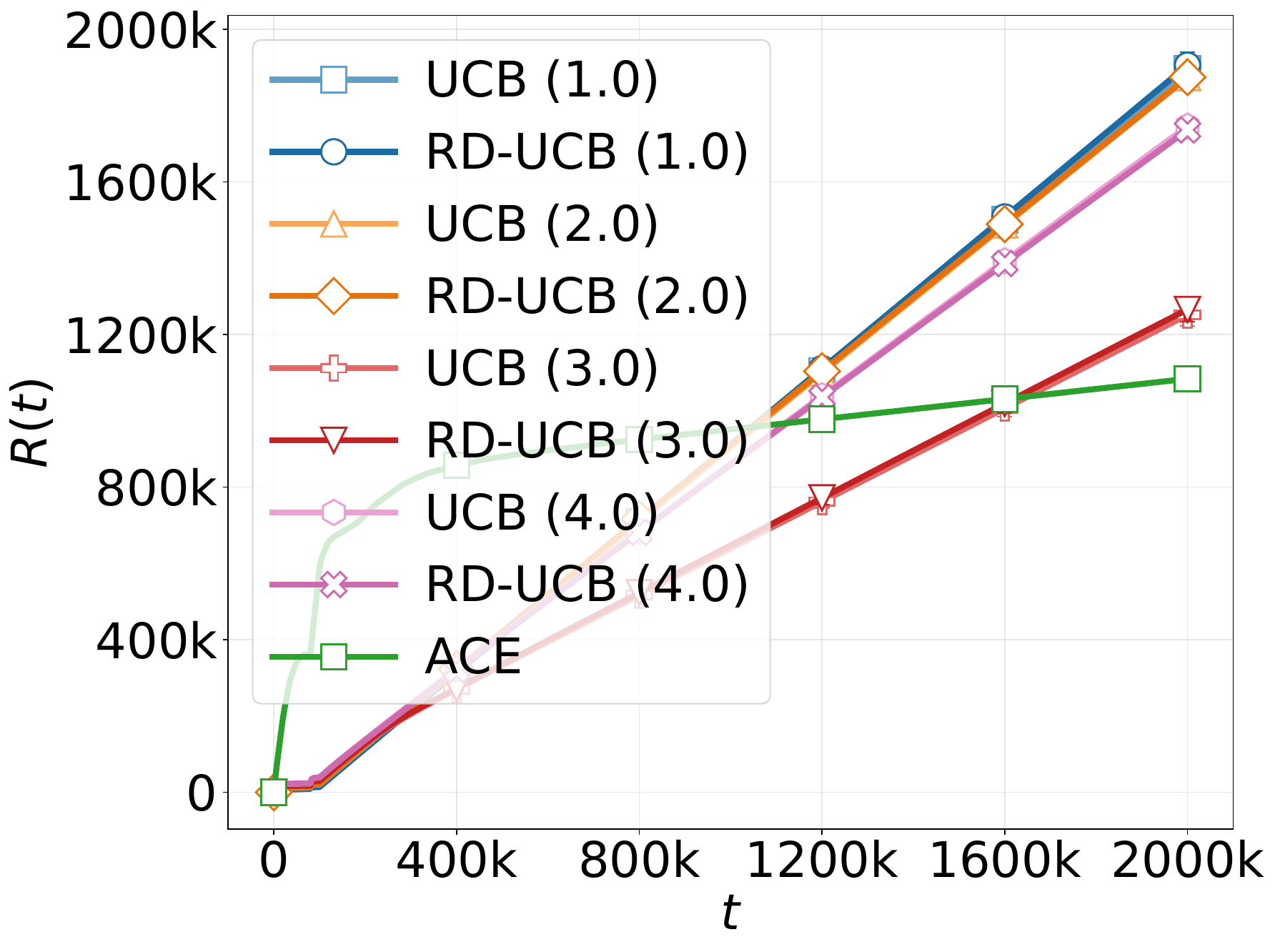}
        \caption{K=100, synthetic, with UCBs.}
        % \label{fig:ucb-asyn-syn}
    \end{subfigure}
    \vspace{0.2cm}
    \caption{Comparison of cumulative regret between UCB with multiple parameters and ACE for different $\mathbf{K}$ under different asynchronous settings.}
    \Description{Comparison of cumulative regret between UCB with multiple parameters and ACE for different $\mathbf{K}$ under different asynchronous settings.}
    \label{fig:ucb-k}
\end{figure*}

\textbf{Result Analysis on Figure \ref{fig:k}} $\ \ $ 
\update{
From Figure~\ref{fig:asyn-rdm} to Figure \ref{fig:r-k-100}, we compare the cumulative regret across different numbers of arms \(K\) under the random asynchronization setting. Since players gradually leave toward the end of the time horizon, the regret of all algorithms increases slowly in the later stages. Among the algorithms, ACE and UCB demonstrate superior performance.

Figure~\ref{fig:asyn-syn} to Figure \ref{fig:s-k-100} compare the cumulative regret across different numbers of arms \(K\) under the synthetic asynchronization setting. 
% In Table \ref{tb:asyn-set-s}, since players 7–10 arrive earlier than players 3–6, the arms they release upon leaving tend to have higher expected rewards than the arms currently being exploited by players 3–6. This highlights a key challenge in decentralized asynchronous settings: previously optimal arms may become sub-optimal as players leave.
We observe that for algorithms including SMAA, GoT, DYN-MMAB, and MCTopM, the regret grows linearly as $t$ increases, indicating that player departures cause the remaining players’ selected arms to become sub-optimal. D-MC exhibits phase-wise growth and eventually converges, but with a much higher regret. This is because DMC requires both a global clock and a known lower bound of $\Delta$ as inputs. For a fair comparison, we do not supply it with a global clock and an exact lower bound. Such a situation could be common in real applications. 
In comparison, ACE does not require these extra knowledge, and consistently achieves stable convergence, demonstrating better robustness to various environments.

}

\textbf{Result Analysis on Figure \ref{fig:ucb-k}} $\ \ $ 
\update{
Since Figure~\ref{fig:k} shows that ACE and UCB exhibit comparable regret, we further compare ACE with two types of the UCB algorithm using different confidence parameters in Figure~\ref{fig:ucb-k}. The upper confidence bound in UCB(c) is defined as $\mathrm{UCB}^j_k(c,t) := \hat{\mu}_k^j(t) + \sqrt{c\log T / N_k^j(t)}$. For RD-UCB(c), the upper confidence bound is defined as $\mathrm{\text{RD-UCB}}^j_k(c,t) := \hat{\mu}_k^j(t) + \sqrt{c\log T / N_k^j(t)} + Z^j_k(t)/t$, where $\{Z^j_k(t)\}_{j=1,...,M, k=1,...K, t=1,...T}$ are i.i.d. Gaussian random variables with mean 0 and variance 1. %\cite{trinh2021high}. 
% In Figure~\ref{fig:ucb-k}, we compare UCB and RD-UCB with confidence parameters $c=1,2,3,4$. 
Note that the UCB algorithm shown in Figure~\ref{fig:k} corresponds to UCB(2.0) in Figure~\ref{fig:ucb-k}. 

Figure~\ref{fig:ucb-k} demonstrates that for both UCB and RD-UCB, the regret begins to increase rapidly once a certain point is reached, especially in the synthetic asynchronous setting. 
In comparison, ACE can converge after a brief period of growth. In the following, we provide an example to explain why UCB suffers a linear regret under this synthetic asynchronous setting.
%A similar phenomenon also occurs for ACE; however, after a brief period of growth, ACE can converge. We provide an example to explain that UCB suffers a linear regret under this asynchronous setting.

\begin{example} (Why UCB suffers a linear regret)
Let $K=2$. The expected rewards of arm 1 and 2 are $\mu_1$ and $\mu_2$, respectively. We assume $\mu_1 > \mu_2$ w.l.o.g. Let the time horizon $T$ be sufficiently large. 

% We aim to show that Selfish-UCB can suffer linear regret under this setting.

Suppose player 1 joins the system first and quickly identifies arm 1 as the optimal arm by $t_1 = 0.1 T$. At this point, player 2 enters the system. From player 2's perspective, arm 1 consistently yields zero reward due to collisions with player 1, while arm 2 yields the expected reward $\mu_2$. Consequently, player 2 views arm 2 as the better option.

Since arm 1 appears suboptimal to player 2, the algorithm will only explore it with low probability: the number of pulls grows in an order of ${\log t/ \mu_2^2}$. In other words, player 2 pulls arm 1 at exponentially increasing intervals, approximately at the following time steps: $\exp(\mu_2^2), \exp(2\mu_2^2),\exp(3\mu_2^2), \cdots$. By the time the system reaches $t_2 = 0.6 T$, the probability that player 2 pulls arm 1 at any step $t>t_2$ becomes approximately ${1/ \mu_2^2 T}$. 

Now, suppose player 1 leaves the system at time step $t_2 = 0.6T$. Let $t_3$ denote the first time after 
$t_2$ that player 2 pulls arm 1. Due to the extremely low exploration frequency, the time interval between $t_2$ and $t_3$ can be linear in $T$, during which player 2 continues to exploit the suboptimal arm 2. This results in a significant regret accumulation over that period.

\end{example}

The above example, together with our experimental results, further highlights our contribution: whereas all existing algorithms incur linear regret in the general asynchronous setting, our proposed ACE algorithm achieves a sub-linear regret upper bound.

%This theoretical insight aligns with our experimental results. As shown in Figure \ref{fig:ucb-k}, UCB algorithms consistently exhibit a phase of linear regret growth, confirming the expected behavior from our analysis.
}

\section{Limitations and Future Work} \label{app:lmt}

One limitation of ACE is its reliance on uniform arm selection, which may lead to frequent collisions when \(m\) is close to \(K\). To mitigate this, we rely on Assumption~\ref{Assumption:active_players}, which ensures that \(m \leq K/2\). A promising future direction is to design an explicit initialization phase that allows players to estimate their relative ranks even under asynchronous settings, inspired by techniques developed for the synchronous setting in~\citet{boursier2019sic, wang2020optimal}. These estimated ranks can then be used to implement a round-robin arm selection strategy that avoids collisions. With such a mechanism, Assumption~\ref{Assumption:active_players} could potentially be removed, and the regret may also be reduced due to fewer collisions.

Another limitation is that ACE does not fully utilize the exploration information collected by different players, resulting in a regret bound that includes a multiplicative factor of $M$ due to the summation over all players. One possible remedy is to introduce a communication phase where players intentionally trigger collisions to exchange reward information. Communication strategies of this type have been explored in the synchronous setting by \citet{boursier2019sic, shi2020decentralized, huang2022towards}, and may be adapted to support more efficient collaborative exploration in asynchronous environments in our future research.

As discussed in Remark \ref{remark:het}, ACE is naturally applicable to the heterogeneous reward setting, since each player explores and exploits arms independently. This suggests that a regret analysis under heterogeneous rewards is also feasible for our algorithm, which is left as future work. While near-optimal regret has been achieved in the decentralized synchronous setting \citep{shi2021heterogeneous}, extending such results to the asynchronous case remains an open problem.

% \clearpage
\bibliographystyle{plainnat}
\bibliography{reference}
%%%%%%%%%%%%%%%%%%%%%%%%%%%%%%%%%%%%%%%%%%%%%%%%%%%%%%%%%%%%

\clearpage
\appendix
\onecolumn

\section{Related Work}\label{sec:related}

The problem of multi-player multi-armed bandits (MP-MAB) has been extensively studied in the literature under various settings. Within the scope of synchronization, \citet{anantharam1987asymptotically} first study the problem in the centralized setting, and \citet{komiyama2015optimal} achieve asymptotically optimal regret. 
In decentralized MP-MAB, implicit communication based on collisions, which are used to transmit binary information, was gradually developed by \citet{rosenski2016multi, boursier2019sic, wang2020optimal}. This approach allows players to avoid collisions entirely during exploration and fully leverage each other's exploration results.

This implicit communication mechanism has been widely adopted across a range of MP-MAB settings. For instance, heterogeneous reward scenarios where each player has different mean rewards across arms have been addressed by \citet{tibrewal2019distributed, shi2021heterogeneous}. Implicit communication has also been applied in the no-sensing setting, where players can observe only their own rewards but not the collision indicator \(\eta_k(t)\) \citep{boursier2019sic, shi2020decentralized, lugosi2022multiplayer, huang2022towards}. 
Other notable variations include adversarial collisions \citep{mahesh2022multi}, collision-dependent rewards \citep{shi2021multi}, matching markets \citep{zhang2022matching}, and shareable rewards \citep{wang2022multi}. All these approaches rely on implicit communication.
In contrast, another line of work considers fully decentralized settings without any form of communication, where each player explores independently. This includes heterogeneous reward settings \citep{besson2018multi, bistritz2018distributed}, and scenarios with shareable rewards \citep{xu2023competing}.

In the asynchronous MP-MAB problem, several existing studies give their solutions under different assumptions. \citet{rosenski2016multi} design algorithm in decentralized environments with a shared global clock for epoch synchronization, and require to use a lower bound of $\Delta$ as input.
\citet{boursier2019sic} deal with the setting that players may enter at different times but will remain until the end of the game. Another setting where players become active at each time step with some probability has also been considered \citep{bonnefoi2017multi, dakdouk2022massive, richard2024constant}. Note that while \citet{dakdouk2022massive} study the decentralized problem, their approach allows explicit communication between players, which is typically assumed only in centralized environments.

In comparison, our asynchronous setting, in which players do not have a global clock and may enter or exit the system unpredictably, is more general and better aligned with real-world scenarios.
Table~\ref{tb:1} summarizes the regret bounds of algorithms under different asynchronous assumptions. Since the settings differ significantly from ours, these regret bounds are not directly comparable. 

The multi-agent multi-armed bandit problem (MA-MAB) considers a related but distinct setting, where \( M \) players pull arms from \([K]\) at each time \( t \), and no collision occurs even if multiple players select the same arm \citep{szorenyi2013gossip, martinez2019decentralized, yang2021cooperative, wang2023achieving}. Asynchronous variants of MA-MAB have also been explored \citep{chen2023demand, wang2025asynchronous}. However, these works focus on accelerating learning via decentralized communication protocols, rather than addressing collisions. This fundamental difference distinguishes MA-MAB from the MP-MAB setting we consider, and hence their algorithms are also very different from ours.

\section{Proof of Theorem \ref{th:upper}} \label{app:proof}
% Let $\mathcal{T}^j$ denote the set of active time steps for player $j$. Denote by $C^j_k(t)$ the random variable that player $j$ inserts into $\mathcal{P}_k^j$ at time step $t$. Let $t^j_{k,p}(\tau_p)$ denote the time step at which  player $j$ inserts a value into $\mathcal{P}_k^j$ for the $\tau_p$-th time. Also, denote by $D^j_k(t)$ the random variable that player $j$ inserts into $\mathcal{Q}_k^j$ at time step $t$. Let $t^j_{k,q}(\tau_q)$ denote the time step at which player $j$ inserts a value into $\mathcal{Q}_k^j$ for the $\tau_q$-th time. The proof of Theorem \ref{th:upper} is divided into several lemmas. To facilitate the analysis, we consider four events, denoted \( \mathcal{E}_0 \) through \( \mathcal{E}_3 \). The event \( \mathcal{E}_0 \) has already been introduced in the main text, while \( \mathcal{E}_1 \), \( \mathcal{E}_2 \) and \( \mathcal{E}_3 \) are defined below for completeness.

Let $\mathcal{T}^j$ denote the set of active time steps for player $j$. Denote by $C^j_k(t)$ \update{the number (1 or 0)} that player $j$ inserts into $\mathcal{P}_k^j$ at time step $t$. Let $t^j_{k,p}(\tau_p)$ denote the time step at which  player $j$ inserts a value into $\mathcal{P}_k^j$ for the $\tau_p$-th time. Also, denote by $D^j_k(t)$ \update{the number (1 or 0)} that player $j$ inserts into $\mathcal{Q}_k^j$ at time step $t$. 
% \update{In the following we will analyse $\sum_{\tau_p= \tau}^{L_p+\tau} \mathbb{E}\left[C^j_k\left(t^j_{k,p}(\tau_p)\right) | \mathcal{F}_{\tau_{p-1}}\right]$}
%
Let $t^j_{k,q}(\tau_q)$ denote the time step at which player $j$ inserts a value into $\mathcal{Q}_k^j$ for the $\tau_q$-th time. The proof of Theorem \ref{th:upper} is divided into several lemmas. To facilitate the analysis, we consider four events, denoted \( \mathcal{E}_0 \) through \( \mathcal{E}_3 \). The event \( \mathcal{E}_0 \) has already been introduced in the main text, while \( \mathcal{E}_1 \), \( \mathcal{E}_2 \) and \( \mathcal{E}_3 \) are defined below for completeness.

\begin{align*}
    & \mathcal{E}_0 = \left\{ \exists t\in \mathcal{T}^j, j\leq M, k\leq K : |\hat{\mu}^j_k(t) - \mu_k |\geq \sqrt{\frac{6\log(T)}{N^j_k(t)}} \right\} ~,\\
    & \mathcal{E}_1 := \left\{ \exists t \in \mathcal{T}^j, j\leq M, k\leq K: \left| \sum_{\tau_p= \tau}^{L_p+\tau} C^j_k\left(t^j_{k,p}(\tau_p)\right) - \sum_{\tau_p= \tau}^{L_p+\tau} \mathbb{E}\left[ C^j_k\left(t^j_{k,p}(\tau_p)\right)\update{\Big| \mathcal{F}_{\tau_{p-1}}}\right]  \right| \geq 0.034L_p  \right\} ~, \\
    & \mathcal{E}_2 := \left\{ \exists t \in \mathcal{T}^j, j\leq M, k\leq K: \left| \sum_{\tau_q= \tau}^{L_q+\tau} D^j_k\left(t^j_{k,q}(\tau_q)\right) - \sum_{\tau_q= \tau}^{L_q+\tau} \mathbb{E}\left[  D^j_k\left(t^j_{k,q}(\tau_q)\right)\update{\Big| \mathcal{F}_{\tau_{q-1}}}\right]\right| \geq 0.0419L_q  \right\} ~, \\
    & \mathcal{E}_3 = \left\{ \exists t\in \mathcal{T}^j, j\leq M, k\leq K : |N^j_k(t) - \mathbb{E}[N^j_k(t)]| \geq \frac{1}{2}\mathbb{E}[N^j_k(t)], \mathbb{E}[N_k^j(t)] \geq 36\ln(T)  \right\}~. 
\end{align*}
Here, \(\mathcal{E}_0\) denotes the event that the estimated reward deviates significantly from the expected reward at some time step. 
\(\mathcal{E}_1\) and \(\mathcal{E}_2\) refer to the events where the cumulative sums of \(C^j_k(\tau)\) and \(D^j_k(\tau)\) deviate significantly from their expectations \update{conditioned on the history $\mathcal{F}_{\tau_{p-1}}$ and $\mathcal{F}_{\tau_{q-1}}$, respectively. Note that once conditioned on $\mathcal{F}_{\tau_{p-1}}$ and $\mathcal{F}_{\tau_{q-1}}$, the relevant randomness becomes independent, since the actions of each player at different time steps are independently drawn given the past. 
This conditional independence enables the following concentration arguments to proceed (Lemma \ref{lm:e1}, Lemma \ref{lm:e2}).
} 
Finally, \(\mathcal{E}_3\) is the event that \(N^j_k(t)\) deviates significantly from its expectation \(\mathbb{E}[N^j_k(t)]\) at some time step while \(\mathbb{E}[N^j_k(t)] \geq 36 \ln T\).

\begin{table}[t]
  \centering
  \renewcommand{\arraystretch}{1.2} % 调整表格行间距
  \begin{tabular}{p{1.5cm} p{2.1cm} p{0.9cm} p{5.0cm} p{4.4cm}} % 调整列宽
    \toprule
    ~ & \textbf{Environment} & \textbf{Com} & \textbf{Async setting} & \textbf{Regret bound} \\
    \midrule
    \citet{boursier2019sic} & 
    \raggedright Decentralized & 
    No &
    \raggedright Players arrive at different times but never leave. & 
    $\mathcal{O} \left(\frac{KM\log T}{\Delta^2_{(1)}} + \frac{KM^2\log T}{\mu_M}  \right)$ \\
    \hline
    \citet{dakdouk2022massive} & 
    \raggedright Decentralized & 
    Yes &
    \raggedright Activation probability \( p \) & 
    $\mathcal{O} \left(\max \left\{K^2, \frac{\log(KT)}{Mp(1-p/K)^{M}}\right\}T^{2/3}\right)$ \\
    \hline
    \citet{richard2024constant} &  
    \raggedright Centralized &  
    Yes &
    \raggedright Known activation probability \( p \) &  
    $\mathcal{O} \left( \sqrt{KT\log(KT)\min\{K, Mp\}} \right)$ \\
    \hline
    \citet{richard2024constant} &  
    \raggedright Centralized &  
    Yes &
    \raggedright Known activation probability \( p \) &  
    $\mathcal{O} \left( \frac{(K^2+(1+p)M^2)\log(KT)}{\Delta_{(2)}} \right)$ \\
    \hline
    ACE &  
    \raggedright Decentralized &  
    No & 
    \raggedright Players arrive and leave arbitrarily over time. &  
    $\mathcal{O} \left( m^{3/2}M\sqrt{T\ln T} +  \frac{mKM\log T}{\Delta^2_{(3)}} \right)$
  \\
    \bottomrule
  \end{tabular}
  \\[1em]
  \caption{Comparison of different algorithms. The column "Com" indicates whether communication via a specific channel is allowed. Note that this refers to \textit{explicit communication}, where players directly exchange information, rather than relying on collisions as an implicit signaling mechanism. "Activation probability \( p \)" refers to the setting where a player becomes active at each step with probability \( p \). \( \Delta_{(1)} \), \( \Delta_{(2)} \) and \( \Delta_{(3)} \) represent different definitions of the reward gap.}
  \label{tb:1}
\end{table}

The first lemma is a well-established result based on Lemma \ref{lm:hfd}.
\begin{lemma}
The probability of event $\mathcal{E}_0$ is bounded by
\begin{align*}
    \Pr[\mathcal{E}_0]\leq \frac{2KM}{T} ~.
\end{align*}
\label{lm:p-e}
\end{lemma}

\begin{proof}
\begin{align}
    \Pr\left[\mathcal{E}_0 \right] &\leq \sum_{t=1}^T \sum_{j\leq M} \sum_{k\leq K} \Pr \left[ |\hat{\mu}^j_k(t) - \mu_k| \geq \sqrt{\frac{6\log(T)}{N^j_k(t)}}\, \right] \nonumber \\
    &\leq \sum_{t=1}^T \sum_{j\leq M} \sum_{k\leq K} \sum_{t_0=1}^{t} \Pr \left[ N^j_k(t) = t_0, |\hat{\mu}^j_k(t) - \mu_k| \geq \sqrt{\frac{6\log(T)}{t_0}}\, \right] \nonumber \\
    &\leq \sum_{t=1}^T \sum_{j\leq M} \sum_{k\leq K} t\cdot 2\exp(-12\log(T)) \label{eq:e-a} \\
    &\leq \frac{2KM}{T} ~, \nonumber
\end{align}
where \eqref{eq:e-a} is from Lemma \ref{lm:hfd}.
\end{proof}

By Lemma \ref{lm:p-e}, we have $\ve \leq 2KM^2$.

Lemma \ref{lm:e1} and Lemma \ref{lm:e2} guarantee that $\mathcal{E}_1$ and $\mathcal{E}_2$ also happens with very low probability. 

%players add or remove arms from $\mathcal{A}^j(t)$ correctly with high probability.
\begin{lemma}
For any player $j$, arm $k$ and $\tau$, given \(L_p = 866 \ln(T)\),
\begin{align*}
    \Pr\left[ \mathcal{E}_1 \right] \leq \frac{2MK}{T}~.
\end{align*}
\label{lm:e1}
\end{lemma}
\begin{proof}
\begin{align}
    \Pr\left[ \mathcal{E}_1 \right] &\leq
    \sum_{t=1}^T \sum_{j\leq M} \sum_{k\leq K} \Pr \left[ \left| \sum_{\tau_p= \tau}^{L_p+\tau} C^j_k\left(t^j_{k,p}(\tau_p)\right) - \sum_{\tau_p= \tau}^{L_p+\tau}\mathbb{E}\left[ C^j_k\left(t^j_{k,p}(\tau_p)\right)\update{\Big| \mathcal{F}_{\tau_{p-1}}}\right] \right| \geq 0.034L_p \right] \nonumber \\
    &\leq \sum_{t=1}^T \sum_{j\leq M} \sum_{k\leq K} 2\exp \left( \frac{ -2\cdot(0.034L_p)^2}{L_p} \right) \label{eq:e1-a} \\
    &= \sum_{t=1}^T \sum_{j\leq M} \sum_{k\leq K} 2\exp \left( \frac{ -2\cdot[0.034 \cdot 866\ln(T) ]^2}{866\ln(T)} \right) \nonumber\\
    &\leq \sum_{t=1}^T \sum_{j\leq M} \sum_{k\leq K} 2\exp \left( -2.002\ln(T) \right) \nonumber\\
    &\leq \frac{2KM}{T} ~. \nonumber
\end{align}
where \eqref{eq:e1-a} comes from Lemma \ref{lm:hfd-sum}.
\end{proof}

\begin{lemma}
For any player $j$, arm $k$ and $\tau$, given \(L_q = 570 \ln(T)\),
\begin{align*}
    \Pr\left[ \mathcal{E}_2 \right] \leq \frac{2MK}{T}~.
\end{align*}
\label{lm:e2}
\end{lemma}

\begin{proof}
\begin{align}
    \Pr\left[ \mathcal{E}_2 \right] &=
    \sum_{t=1}^T \sum_{j\leq M} \sum_{k\leq K} \Pr \left[ \left| \sum_{\tau_q= \tau}^{L_q+\tau} D^j_k\left(t^j_{k,q}(\tau_q)\right) - \sum_{\tau_q= \tau}^{L_q+\tau}\mathbb{E}\left[ D^j_k\left(t^j_{k,q}(\tau_q)\right)\update{\Big| \mathcal{F}_{\tau_{q-1}}}\right]  \right| \geq 0.0419L_q \right] \nonumber \\
    &\leq \sum_{t=1}^T \sum_{j\leq M} \sum_{k\leq K} 2\exp\left( \frac{-2\cdot (0.0419L_q)^2 }{L_q} \right) \label{eq:e2-a} \\
    &= \sum_{t=1}^T \sum_{j\leq M} \sum_{k\leq K} 2\exp\left( \frac{-2\cdot (0.0419\cdot 570\ln(T) )^2 }{570\ln(T)} \right) \nonumber \\
    &\leq \sum_{t=1}^T \sum_{j\leq M} \sum_{k\leq K} 2\exp\left( 2.001\ln(T) \right)\nonumber  \\
    &\leq \frac{2KM}{T} ~, \nonumber 
\end{align}
where \eqref{eq:e2-a} comes from Lemma \ref{lm:hfd-sum}.
\end{proof}

Next, Lemma \ref{lm:p-e-4} proves that $N^j_k(t)$ remains close to its expectation $\mathbb{E}[N^j_k(t)]$ with high probability when $\mathbb{E}[N_k^j(t)] \geq 36\ln(T)$.
\begin{lemma}
The probability of event $\mathcal{E}_3$ is bounded by
\begin{align*}
    \Pr\left[\mathcal{E}_3 \right]\leq \frac{2KM}{T} ~.
\end{align*}
\label{lm:p-e-4}
\end{lemma}

\begin{proof}
%Note that player $j$ only updates $N^j_k(t)$ when arm $k$ is sampled uniformly from either $[K]\setminus \mathcal{A}^j(t)$ (Line \ref{l:update}, Algorithm \ref{alg:1}). 
%Let $\delta = 1/2$,
\begin{align}
    \Pr\left[ \mathcal{E}_3 \right] &= \sum_{t=1}^T \sum_{j\leq M} \sum_{k\leq K} \Pr\left[ |N^j_k(t) - \mathbb{E}[N^j_k(t)]| \geq \frac{1}{2}\mathbb{E}[N^j_k(t)], \mathbb{E}[N_k^j(t)] \geq 36\ln(T) \right] \nonumber \\
    %&= \sum_{t=1}^T \sum_{j\leq M} \sum_{k\leq K}  \Pr\left[ |N^j_k(t) - \mathbb{E}[N^j_k(t)]| \geq \frac{1}{2}\mathbb{E}[N^j_k(t)], \mathbb{E}[N_k^j(t)] \geq 36\ln(T) \right] \nonumber \\
    &\leq \sum_{t=1}^T  \sum_{j\leq M} \sum_{k\leq K} \sum_{t_0=36\ln(T)}^t  \Pr\left[ \mathbb{E}[N_k^j(t)] = t_0, |N^j_k(t) - t_0| \geq \frac{1}{2}t_0 \right] \nonumber \\
    &\leq \sum_{t=1}^T \sum_{j\leq M} \sum_{k\leq K} t\cdot 2 \exp(-3\ln(T)) \label{eq:n-concentration-b} \\
    &\leq \frac{2KM}{T} ~, \nonumber
\end{align}
where \eqref{eq:n-concentration-b} is from Lemma \ref{lm:chr}.
\end{proof}

With slight abuse of notation, we denote by $\hat{k}^j(t)$ the arm occupied by player $j$ at time step $t$. When the context is clear, we write $\hat{k}^j$ to refer to $\hat{k}^j(t)$. For player $j$, let \(\hat{k}^j(t) = 0\) for all \(t \notin \mathcal{T}^j\). Lemma \ref{lm:nsame} guarantees that no two players can occupy the same arm simultaneously.

\begin{lemma}  
    % Let \( j_1 \) and \( j_2 \) be two players in the system. Then, f
    For any players \( j_1, j_2 \in [M] \), there does not exist a time step \( t  \) and an arm $k$ such that \( \hat{k}^{j_1}(t) = \hat{k}^{j_2}(t) = k \).  
    \label{lm:nsame}  
\end{lemma}

\begin{proof}
We consider the following two cases that may occur at step \(t\): 
\begin{itemize}
    \item Case 1: Both $j_1$ and $j_2$ are in exploration phase, and they both choose to set $\hat{k}^{j_1}(t) \leftarrow k$ and $\hat{k}^{j_2}(t) \leftarrow k$ in this step $t$.
    \item Case 2: $j_1$ is in exploitation phase with $\hat{k}^{j_1}(t) = k$, while $j_2$ is in exploration phase, and $j_2$ choose to set $\hat{k}^{j_2}(t) \leftarrow k$ in this step $t$.
\end{itemize}
If Case 1 happens, $j_1$ and $j_2$ execute Line \ref{l:nsame_0} in Algorithm \ref{alg:1} simultaneously. %; otherwise, they enter the exploitation phase. Let $t$ be the last time when $j_1$ and $j_2$ find $\hat{k}^{j_1}(t)$ and $\hat{k}^{j_2}(t)$ in the exploration phase. 
According to Line \ref{l:nsame_0}, a player transitions to the exploitation phase only when Condition \ref{eq:cd1} and Condition \ref{eq:cd2} are satisfied. To ensure Condition \ref{eq:cd1}, we have $\eta_k(t)=0$, indicating that only one player pulls the arm at $t$. Consequently, Case 1 never happens.

Next, we consider Case 2. Suppose that $j_1$ has already occupied $k$, i.e., $\hat{k}^{j_1} = k$. From Line \ref{l:explt1}-\ref{l:nsame_2} in Algorithm \ref{alg:2}, $j_1$ pulls $k$ twice with probability $1-\varepsilon$. Otherwise, she pulls $k$ in one time step and next pulls a different arm $k^{\prime}$, which is uniformly sampled from $\mathcal{A}^{j_1}(t)$. Thus, $j$ pulls $k$ and $k^{\prime}$ alternately. Meanwhile, $j_2$ is still in exploration phase and observes that arm $k$ satisfies Condition \ref{eq:cd2} at $t$. To ensure Condition \ref{eq:cd1}, we reqire $\eta_k(t-1)=0$ and $\eta_k(t)=0$, i.e., $j_2$ needs to observe two consecutive non-collision events on $k$. However, since player $j_1$ pulls $k$ and $k'$ in at least an alternating fashion, Case 2 cannot occur.
\end{proof}

Let $T_{o}^j$ denote the number of time steps that is required for player $j$ to identify an occupied arm $k$, and let $T_{r}^j$ denote the number of time steps that is required for player $j$ to identify a released arm $k$. The following lemma proves an upper bound for their expectation. %the time steps that player $j$ needs to update $\mathcal{A}^j(t)$.

\begin{lemma}
Under the condition of $\overline{\mathcal{E}_1}$ and $\overline{\mathcal{E}_2}$, for any player $j$ and arm $k$,

(i) if arm $k$ is occupied %at time step $t-1$ 
and remains occupied thereafter, 
player $j$ will add $k$ to $\mathcal{A}^j(t)$ with $\mathbb{E}[T_o^j] \le 1926 K\ln(T)$ time steps;

(ii) if arm $k$ is not occupied %at time step $t-1$ 
and remains not occupied thereafter, player $j$ will not add $k$ to $\mathcal{A}^j(t)$;

(iii) if arm $k$ is released %at time $t-1$ 
and never occupied again, then player $j$ will remove $k$ from $\mathcal{A}^j(t)$ with $\mathbb{E}[T_r^j] \le {1141 m\ln(T) }/{\varepsilon}$ time steps;

(iv) if arm $k$ is not released %at time $t-1$ 
and remains not released thereafter, player $j$ will not remove $k$ from $\mathcal{A}^j(t)$.
\label{lm:temp-step-new}
\end{lemma}

\begin{proofOfOne}
We begin to prove the first term. Let arm $k$ be occupied by player $j^\prime$ and never released. Suppose that the next time that player $j$ pulls arm $k$ twice is the $\tau$-th time we insert a value into $\mathcal{P}^j_k$. Then we know that for any $\tau_p \ge \tau$, the probability that player $j$ experiences two consecutive collisions is at least $1- \varepsilon$ (since player $j^\prime$ pulls $k$ twice with probability at least $1- \varepsilon$).
%
%at $1$
%
%at time steps $t-1$ and $t$ twice during the exploration phase, and $k\notin \mathcal{A}^j(t)$. 
%
%Let $p^j_i(t)$ denote the probability that player $j$ experiences two consecutive collisions at both $t-1$ and $t$ in this case. Since arm $k$ is occupied, player $j^\prime$ pulls $k$ twice with probability at least $1- \varepsilon$. Thus, $p^j_i(t)\geq 1- \varepsilon \geq 0.9$. Note that $C^j_k(t) \sim \mathrm{Bernoulli}(p^j_i(t))$, 
This implies
\begin{align*}
    \sum_{\tau_p= \tau}^{L_p+\tau}\mathbb{E}\left[ C^j_k\left(t^j_{k,p}(\tau_p)\right) \update{\Big| \mathcal{F}_{\tau_{p-1}}} \right] %= \sum_{\tau_p= \tau}^{L_p+\tau} p^j_i\left( t^j_k(\tau_p) \right) 
    \ge 0.9 L_p ~.
\end{align*}

Condition on $\overline{\mathcal{E}_1}$, we know that $\sum_{\tau_p= \tau}^{L_p+\tau} C^j_k(t^j_{k,p}(\tau_p)) \geq 0.9 L_p - 0.034 L_p \ge 0.85 L_p$ \update{under the condition of $\mathcal{F}_{\tau_{p-1}}$}. 
That is, after $L_p$ times of insert, player $j$ will put $k$ into $\mathcal{A}^j$.

%
%By Lemma \ref{lm:e1},
%\begin{align*}
%    \Pr\left[ \sum_{\tau_p= \tau}^{L_p+\tau} C^j_k\left(t^j_k(\tau_p)\right) \geq  0.85L_p  \right] &= \Pr\left[  \sum_{\tau_p= \tau}^{L_p+\tau} C^j_k\left(t^j_k(\tau_p)\right) \geq  0.9L_p - 0.05L_p  \right]  \\
%    &\geq \Pr\left[  \sum_{\tau_p= \tau}^{L_p+\tau} C^j_k\left(t^j_k(\tau_p)\right) \geq  \mathbb{E}[p^j_{i}(\tau_p)]L_p  - 0.05L_p  \right]  \\
%    &\geq \Pr\left[  \sum_{\tau_p= \tau}^{L_p+\tau} C^j_k\left(t^j_k(\tau_p)\right) \geq  \mathbb{E}[p^j_{i}(\tau_p)]L_p  - 0.034L_p  \right]  \\
%    &\geq \Pr \left[ \left| \sum_{\tau_p= \tau}^{L_p+\tau} C^j_k\left(t^j_k(\tau_p)\right) - \mathbb{E}\left[\sum_{\tau_p= \tau}^{L_p+\tau} C^j_k\left(t^j_k(\tau_p)\right)\right] \right| \leq 0.034L_p \right]  \\
%    &= \Pr\left[\,\overline{\mathcal{E}_1}\,\right] ~.
%\end{align*}
%Thus, when arm $k$ is occupied, $\sum_{\tau_p= \tau}^{L_p+\tau} C^j_k(t^j_k(\tau_p)) \geq 0.85 L_p$ and Line \ref{l:cond1} in Algorithm \ref{alg:1} is executed, resulting in arm $k$ being added to $\mathcal{A}^j(t)$.

Also note that if player $j$ is in the exploration phase and $k$ has not yet been added to $\mathcal{A}^j(t)$, the probability of pulling arm $k$ twice is at least $(1 - \varepsilon)/(K - |\mathcal{A}^j(t)|)$. %To fulfill Line \ref{l:cond1} in Algorithm \ref{alg:1}, 
Hence, $\mathbb{E}[T_{o}^j]$ is bounded by
\begin{align}
    \mathbb{E}[T_{o}^j] &\leq \max_{j,t} \frac{1}{(1-\varepsilon)/(K-|\mathcal{A}^j(t)| )} \cdot 2L_p + 1 \nonumber \\
    &\leq \max_{j,t} \frac{K-|\mathcal{A}^j(t)| }{1-\varepsilon}  \cdot  2L_p   + 1  \nonumber \\
    &\leq \frac{10}{9}K  \cdot  866 \ln(T) + 1 \label{eq:t-o-a} \\
    &\leq 1926 K\ln(T) ~,  \nonumber 
\end{align}
where \eqref{eq:t-o-a} is from $\varepsilon \leq 1/10$.
\end{proofOfOne}

\begin{proofOfTwo}
%Next, we prepare to bound the second term. 
Let arm $k$ remain unoccupied.  Suppose that the next time that player $j$ pulls arm $k$ twice is the $\tau$-th time we insert a value into $\mathcal{P}^j_k$. Then we know that for any $\tau_p \ge \tau$, the probability that player $j$ experiences two consecutive collisions is upper bounded by (note that by Remark \ref{remark:c1}, every player who is not exploiting arm $k$ can pull arm $k$ with probability at most $1/m$) :
%
%at least $1- \varepsilon$ (since player $j^\prime$ pulls $k$ twice with probability at least $1- \varepsilon$).
%Suppose player $j$ pulls arm $k$ at time steps $t-1$ and $t$ during the exploration phase, and $k\notin \mathcal{A}^j(t)$. 
%
%Let $p^j_{ii}(t)$ denote the probability that player $j$ experiences collisions at both $t-1$ and $t$. Although there does not exist a player $j^\prime \in [M]$ such that $\hat{k}^{j^\prime} = k$, arm $k$ may still be pulled due to random exploration by other players. Therefore, $p^j_{ii}(t)$ is bounded by 
\begin{align}
    1- \prod_{j'\ne j: j' \text{ is active}}(1-{\frac{1}{m}}) 
    %&\leq 1 - (1-\varepsilon)^{|\mathcal{M}_\mathrm{check}|} \cdot (1-\frac{1}{K-K/2})^{|\mathcal{M}_\mathrm{exp}|} \cdot (1-\frac{1}{K/2})^{|\mathcal{M}_\mathrm{cor}|} \label{eq:pdbar-a2} \\
    %&\leq 1 - (1-\frac{1}{K})^{|\mathcal{M}_\mathrm{check}|} (1-\frac{2}{K})^{|\mathcal{M}_\mathrm{exp}|} (1-\frac{2}{K})^{|\mathcal{M}_\mathrm{cor}|} \label{eq:pdbar-b2} \\
    %&\leq 1 - (1-\frac{2}{K})^{m} \nonumber \\
    %&\leq 1 - (1-\frac{2}{K})^{K/2} \label{eq:pdbar-d2} \\
    \leq 1 - (1-{\frac{1}{m}})^m\leq 1 - \frac{1}{2e} %\label{eq:pdbar-c2} \\
    \leq 0.816 ~. \nonumber
\end{align}

This implies
\begin{align*}
    \sum_{\tau_p= \tau}^{L_p+\tau}\mathbb{E}\left[ C^j_k\left(t^j_{k,p}(\tau_p)\right) \update{\Big| \mathcal{F}_{\tau_{p-1}}} \right] %= \sum_{\tau_p= \tau}^{L_p+\tau} p^j_i\left( t^j_k(\tau_p) \right) 
    \le 0.816 L_p ~.
\end{align*}

Condition on $\overline{\mathcal{E}_1}$, we know that $\sum_{\tau_p= \tau}^{L_p+\tau} C^j_k(t^j_{k,p}(\tau_p)) \leq 0.816 L_p + 0.034 L_p \le 0.85 L_p$ \update{under the condition of $\mathcal{F}_{\tau_{p-1}}$}. 
That is, player $j$ will never put $k$ into $\mathcal{A}^j$.
\end{proofOfTwo}

\begin{proofOfThree}
Now, we move to bound the third term. Let arm $k$ be released by player $j^\prime$ and never occupied. Suppose that the next time that player $j$ pulls arm $k$ is the $\tau$-th time we insert a value into $\mathcal{Q}^j_k$. Then we know that for any $\tau_q \ge \tau$, the probability that player $j$ does not experience a collision is at least $1/2e \ge 0.1839$, as stated in the proof of part (ii).
%
%at $1$
%
%at time steps $t-1$ and $t$ twice during the exploration phase, and $k\notin \mathcal{A}^j(t)$. 
%
%Let $p^j_i(t)$ denote the probability that player $j$ experiences two consecutive collisions at both $t-1$ and $t$ in this case. Since arm $k$ is occupied, player $j^\prime$ pulls $k$ twice with probability at least $1- \varepsilon$. Thus, $p^j_i(t)\geq 1- \varepsilon \geq 0.9$. Note that $C^j_k(t) \sim \mathrm{Bernoulli}(p^j_i(t))$, 
This implies
\begin{align*}
    \sum_{\tau_q= \tau}^{L_q+\tau} \mathbb{E}\left[ D^j_k\left(t^j_{k,q}(\tau_q)\right) \update{\Big| \mathcal{F}_{\tau_{q-1}}} \right] %= \sum_{\tau_p= \tau}^{L_p+\tau} p^j_i\left( t^j_k(\tau_p) \right) 
    \ge 0.1839 L_q ~.
\end{align*}

Condition on $\overline{\mathcal{E}_2}$, we know that $\sum_{\tau_q= \tau}^{L_p+\tau} C^j_k(t^j_{k,q}(\tau_q)) \geq 0.1839 L_q  - 0.0419 L_p \ge 0.142 L_q$ \update{under the condition of $\mathcal{F}_{\tau_{q-1}}$}. 
That is, after $L_q$ times of insert, player $j$ will remove $k$ from $\mathcal{A}^j$.

Also note that the probability of pulling a specific arm $k$ in $\mathcal{A}^j$ is either at least $1/K$ (during correction) or at least $\varepsilon / m$ (in other cases), and $\varepsilon / m \le 1/K$.
Therefore, \(\mathbb{E}[T_{r}^j]\) is bounded by
\begin{align*}
    \mathbb{E}[T_{r}^j] 
    &\leq \frac{m}{\varepsilon}  \cdot  2L_q  + 1   \\
    &\leq \frac{1140 m\ln(T) }{\varepsilon} + 1 \\
    &\leq \frac{1141 m\ln(T) }{\varepsilon}~.
\end{align*}

\end{proofOfThree}

\begin{proofOfFour}
Let arm $k$ remain occupied.  Suppose that the next time that player $j$ pulls arm $k$ s the $\tau$-th time we insert a value into $\mathcal{Q}^j_k$. Then we know that, similar to the proof of part (i), for any $\tau_q \ge \tau$, the probability that player $j$ does not experience a collision is at most $\varepsilon \le 0.1$.

This implies
\begin{align*}
    \sum_{\tau_q= \tau}^{L_q+\tau}\mathbb{E}\left[ C^j_k\left(t^j_{k,q}(\tau_q)\right) \update{\Big| \mathcal{F}_{\tau_{q-1}}} \right] %= \sum_{\tau_p= \tau}^{L_p+\tau} p^j_i\left( t^j_k(\tau_p) \right) 
    \le 0.1 L_q ~.
\end{align*}

Condition on $\overline{\mathcal{E}_2}$, we know that $\sum_{\tau_q= \tau}^{L_q+\tau} C^j_k(t^j_{k,q}(\tau_q)) \leq 0.1 L_q + 0.0419 L_q \le 0.142 L_q$ \update{under the condition of $\mathcal{F}_{\tau_{q-1}}$}. 
That is, player $j$ will never remove $k$ from $\mathcal{A}^j$.
\end{proofOfFour}

\begin{lemma}
Let $\overline{\mathcal{E}}_0$ holds, and consider any two arms $k$ and $k^\prime$ such that $\mu_k > \mu_{k^\prime}$ and $k \le m$. If \(\min\{N^j_k(t), N^j_{k^\prime}(t)\} \geq {96 \log(T)}/{\Delta^2}\),
then we have $\mathrm{LCB}^j_k(t) \geq \mathrm{UCB}^j_{k^\prime}(t)$.
\label{lm:exp-times}
\end{lemma}

\begin{proof} 
\(\min\{N^j_k(t), N^j_{k'}(t)\} \geq {96 \log(T)}/{\Delta^2}\) implies that
\begin{align}
    \mu_k - \mu_{k^\prime}  \geq \Delta \geq \sqrt{\frac{96 \log(T)}{\min\{N^j_k(t), N^j_{k'}(t)\}}} \label{eq:ucb-a} ~.
\end{align}
Since $\overline{\mathcal{E}_0}$ holds, for player $j$,
\begin{align*}
    \hat{\mu}_k - \sqrt{\frac{6\log(T)}{N^j_k(t)}} \leq \mu_{k} \leq \hat{\mu}_k + \sqrt{\frac{6\log(T)}{N^j_k(t)}} ,\ \forall k \in[K] ~.
\end{align*}
Then, $\mu_k - \mu_{k^\prime}$ is upped bounded by
\begin{align}  
    \mu_k - \mu_{k^\prime} &\leq  \hat{\mu}_k + \sqrt{\frac{6\log(T)}{N^j_k(t)}} - \mu_{k^\prime} \nonumber \\
    &\leq \left[\hat{\mu}_k + \sqrt{\frac{6\log(T)}{N^j_k(t)}}\ \right] - \left[ \hat{\mu}_{k^\prime} - \sqrt{\frac{6\log(T)}{N^j_{k^\prime}(t)}}\ \right]~. \label{eq:ucb-b}
\end{align}
Combining \eqref{eq:ucb-a} and \eqref{eq:ucb-b} leads to
\begin{align*}
\hat{\mu}_k - \sqrt{\frac{6\log(T)}{N^j_k(t)}} &\geq 
    \left[\hat{\mu}_k + \sqrt{\frac{6\log(T)}{N^j_k(t)}}\ \right] - 2\sqrt{\frac{6 \log(T)}{\min\{N^j_k(t), N^j_{k'}(t)\}}} \\
    &\geq \left[ \hat{\mu}_{k^\prime} - \sqrt{\frac{6\log(T)}{N^j_{k^\prime}(t)}}\ \right] + 2\sqrt{\frac{6 \log(T)}{\min\{N^j_k(t), N^j_{k'}(t)\}}} \\
    &\geq \hat{\mu}_{k^\prime} + \sqrt{\frac{6\log(T)}{N^j_{k^\prime}(t)}} ~,
\end{align*}
which indicates $\mathrm{LCB}^j_k(t) \geq \mathrm{UCB}^j_{k^\prime}(t)$.
\end{proof}

\begin{lemma}
\label{lm:mlogT-3}
Given the condition of $t \notin \mathcal{G}^{j}_1 \cup \mathcal{G}^{j}_2$, let $\overline{\mathcal{E}_0}$ and $\overline{\mathcal{E}_3}$ holds. Then for any player $j$ and arm $k$, we have $N^j_k(t) \le {288 m \log(T)}/{\Delta^2}$.
\end{lemma}

\begin{proof}
Let $\theta:= {96 \log(T)}/{\Delta^2}$. Since \( t \notin \mathcal{G}^{j}_1 \cup \mathcal{G}^{j}_2 \), there are at most \( m - 1 \) arms in \( \mathcal{A}^j(t) \). Given that event \( \overline{\mathcal{E}_3} \) holds, we have for all \( k \in [K] \):
\begin{align}
\frac{1}{2} \mathbb{E}[N^j_k(t)] \le N^j_k(t) \le \frac{3}{2} \mathbb{E}[N^j_k(t)]~. \label{eq:n-ne}
\end{align}

We first claim that (i) for any arm $k$, when $\mathbb{E}[N^j_k(t)] = 2 \theta$, it must hold that $\sum_{k'\ne k} \min \{\mathbb{E}[N^j_{k'}(t)], 2\theta\} \ge 2(K-m)\theta$.

To prove this, let $k^*$ be the first arm such that $\mathbb{E}[N^j_{k^*}(t)] = 2 \theta$. Since there are at most $m-1$ arms in $\mathcal{A}^j$, there are at least $K-m+1$ arms in $[K] \setminus \mathcal{A}^j$. 
That is, when $\mathbb{E}[N^j_{k^*}(t)]$ increases $\delta$, there are another $K-m$ arms $k'$ have their $\mathbb{E}[N^j_{k'}(t)]$ increases $\delta$ since players are doing uniform exploration. Consequently, when $\mathbb{E}[N^j_{k^*}(t)] = 2 \theta$, we have $\sum_{k'\ne k^*} \min \{\mathbb{E}[N^j_{k'}(t)], 2\theta\} = \sum_{k'\ne k^*} \mathbb{E}[N^j_{k'}(t)] \ge 2(K-m)\theta$.

Now consider another arm $k \ne k^*$. 
% we know that at this time $t$, 
At this time step $t$, it follows that 
$\sum_{k'\ne k} \min \{\mathbb{E}[N^j_{k'}(t)], 2\theta\} \ge \sum_{k'\ne k^*} \min \{\mathbb{E}[N^j_{k'}(t)], 2\theta\} \ge 2(K-m)\theta$, which finishes the proof of our claim (i).

Next, we claim that (ii) condition on event $\overline{\mathcal{E}_0}$ and $\overline{\mathcal{E}_3}$, for any arm $k$ with $\mathbb{E}[N^j_{k}(t)] \ge 2 \theta$, if $\mathbb{E}[N^j_{k}(t)]$ increases by $\delta$, then there must be another arm $k'$ with $\mathbb{E}[N^j_{k'}(t)] < 2 \theta$, and $\mathbb{E}[N^j_{k'}(t)]$ increases by $\delta$.

This is proved by contradiction. Assume that when arm $k$ with $\mathbb{E}[N^j_{k}(t)] \ge 2 \theta$ increases its $\mathbb{E}[N^j_{k}(t)]$ by $\delta$, there are no other arm $k'$ with $\mathbb{E}[N^j_{k'}(t)] \le 2 \theta$. Then, by event $\overline{\mathcal{E}_3}$, 
all arms $k'' \notin \mathcal{A}^j(t)$ must satisfy $\mathbb{E}[N^j{k''}(t)] \ge 2\theta$, which implies $N^j_{k''}(t) \ge \theta$. 
% we know that for all arm $k'' \notin \mathcal{A}^j(t)$, $\mathbb{E}[N^j_{k''}(t)] \ge 2 \theta$, which implies $N^j_{k''}(t) \ge \theta$. 
By Lemma \ref{lm:exp-times}, we must figure out which arm is optimal in $[K] \setminus \mathcal{A}^j(t)$ and do not need to explore, 
% this means $\mathbb{E}[N^j_{k}(t)]$ cannot increase, and finish the proof of claim (ii).
Thus, $\mathbb{E}[N^j_k(t)]$ cannot increase, leading to a contradiction and completing the proof of claim (ii).

From claim (ii), we know that once $\sum_{k'\ne k} \min \{\mathbb{E}[N^j_{k'}(t)], 2\theta\} = 2(K-1)\theta$, $\mathbb{E}[N^j_{k}(t)]$ cannot increase. Also, when $\mathbb{E}[N^j_{k}(t)]$ increases by $\delta$, $\sum_{k'\ne k} \min \{\mathbb{E}[N^j_{k'}(t)], 2\theta\}$ must increase by $\delta$. By claim (i), we know that when $\mathbb{E}[N^j_k(t)] = 2 \theta$, $\sum_{k'\ne k} \min \{\mathbb{E}[N^j_{k'}(t)], 2\theta\} \ge 2(K-m)\theta$. Combining these results yields that $\mathbb{E}[N^j_k(t)]$ is at most $2\theta + 2(K-1)\theta -  2(K-m)\theta = 2m\theta$.

Finally, under event $\overline{\mathcal{E}_3}$, it follows that 
% Then by event $\overline{\mathcal{E}_3}$, we prove that 
$N^j_k(t) \le 3m\theta = {288 m \log(T)}/{\Delta^2}$.
%Now we prove that $\mathbb{E}[N^j_k(t)]$ will increase only when there are some arm $k'$,  
%Suppose that there exists an arm \( k \) such that \( N^j_k(t) \ge 3\theta \). Then, it follows that \( \mathbb{E}[N^j_k(t)] \ge \frac{2}{3} \cdot 3\theta = 2\theta \). Since player \( j \) pulls arms uniformly from \( [K] \setminus \mathcal{A}^j(t) \), an increase in \( \mathbb{E}[N^j_k(t)] \) implies a uniform increase across all arms in \( [K] \setminus \mathcal{A}^j(t) \). When \( \mathbb{E}[N^j_k(t)] \ge 2\theta \),
%\[
%\sum_{k \in [K] \setminus \mathcal{A}^j(t)} \mathbb{E}[N^j_k(t)] \ge 2(K - m + 1)\theta~.
%\]
%Now, let \(\mathbb{E}[N^j_k(t)] \ge 2\theta + 2(m - 1)\theta = 2m\theta\) and the following holds:
%\[
%\sum_{k \in [K] \setminus \mathcal{A}^j(t)} \mathbb{E}[N^j_k(t)] \ge 2(K - m + 1)\theta + 2(m-1)\theta = 2K\theta~,
%\]
%which implies \( \mathbb{E}[N^j_{k^\prime}(t)] \ge 2\theta \) for all \( k^\prime \in [K] \setminus \mathcal{A}^j(t) \). By \eqref{eq:n-ne}, this yields \( N^j_{k^\prime}(t) \ge \theta \). Therefore, by Lemma~\ref{lm:exp-times}, each such arm \( k^\prime \) has been distinguished and will no longer be explored by player \( j \) in future steps.
\end{proof}

\begin{lemma}
Under events $\overline{\mathcal{E}_1}$ and $\overline{\mathcal{E}_2}$,
\begin{align*}
    \sum_{j\le M}\sum_{t \in \mathcal{T}^j} \mathds{1}\left[\mathcal{A}^j(t) \ne \mathcal{A}^j(t+1) \right] \leq 3m^2M ~. 
\end{align*}
\label{lm:switch-times}
\end{lemma}
\begin{proof}
This bound follows from the fact that arm removals from \(\mathcal{A}^j\) are triggered by the permanent departure of players. Each such departure may cause a switch from exploitation to exploration for up to \(m\) remaining players, and each of these players may remove up to \(m\) arms from $\mathcal{A}^j$. Since at most \(M\) players can permanently leave the system, the total number of such removals is at most \(m^2M\). 

Since adding arms to $\mathcal{A}^j(t)$ and removing arms from $\mathcal{A}^j(t)$ can be a one-one mapping, except for those arms in $\mathcal{A}^j(t)$ at the end of the game. Hence, the number of times adding arms to $\mathcal{A}^j(t)$ is at most $m^2M + mM$.

Since both adding or removing leads to the change of $\mathcal{A}^j(t)$, taking the summation, we prove that $\sum_{j\le M}\sum_{t \in \mathcal{T}^j} \mathds{1}\left[\mathcal{A}^j(t) \ne \mathcal{A}^j(t+1) \right] \leq 2m^2M + mM \le 3m^2M$.
\end{proof}

The following analysis focuses on bounding the regret arising from $\vb, \va+\vc$ and $\vd$.

\Appvb*

\begin{proof}
Each exploration phase ends when both Condition~\ref{eq:cd1} and Condition~\ref{eq:cd2} are satisfied. Let \(T_c\) denote the time steps required for a player to satisfy Condition~\ref{eq:cd1} after Condition~\ref{eq:cd2} has been met during any exploration phase and $\mathcal{A}^j(t)$ does not change. $\mathbb{E}[T_c]$ is bounded as 
\begin{align}
    \mathbb{E}\left[T_c\right] &\leq 2\max_{j\leq M,t\in \mathcal{T}^j}\left[ \frac{1-\varepsilon}{K- |\mathcal{A}^j(t)|} \prod _{{j^{\prime}\in \mathcal{M}_\mathrm{exp}(t) }} \left(1 - \frac{1-\varepsilon}{K-|\mathcal{A}^{j^\prime}(t)|} \right) \right]^{-2} \label{eq:tao-a} \\
    &\leq 2\left[\min_{j\leq M,t\in \mathcal{T}^j} \frac{1-\varepsilon}{K- |\mathcal{A}^j(t)|} \prod _{{j^{\prime}\in \mathcal{M}_\mathrm{exp}(t) }} \left(1 - \frac{1-\varepsilon}{K-|\mathcal{A}^{j^\prime}(t)|} \right) \right]^{-2} \nonumber \\
    &\leq 2\left[\frac{9}{10K} \left(1- \frac{1}{K-m} \right)^{m} \right]^{-2} \label{eq:tao-b} \\
    &\leq 2\left[ \frac{9}{10K} \left((1- \frac{1}{m}\right)^{m} \right]^{-2} \label{eq:tao-c} \\
    &\leq {4e^2K^2} ~, \label{eq:tao-d}
\end{align}
where \eqref{eq:tao-a} holds because, at any given time $t$, the event that player $j$ selects arm $k$ and observes $\eta_k(t) = 0$ can be modeled as a Bernoulli trial. Condition \ref{eq:cd1} is satisfied only after observing two consecutive collision-free rounds. This corresponds to the waiting time until the first occurrence of two consecutive successes in a Bernoulli process, whose expected length is $({1 + p})/{p^2} \leq 2/p^2$, where $p$ is the single-trial success probability. \eqref{eq:tao-b} follows from the assumption that $\varepsilon \leq 1/10$, and \eqref{eq:tao-c} follows by Assumption \ref{Assumption:active_players}.

Define:
\begin{align*}
    & \mathcal{K}_1(t):= \{ k\leq K: k\text{ does not satisfy Condition \ref{eq:cd1} at step } t  \}~, \\
    & \mathcal{K}_2(t):= \{ k\leq K: k\text{ does not satisfy Condition \ref{eq:cd2} at step } t  \} ~.
\end{align*}
We deompose $\vb$ as
\begin{align*}
    \vb &= \sum_{j \le M}\sum_{t \in \mathcal{T}^j_{\mathrm{exp}}}  \mathbb{E}\left[\sum_{k\leq K } \mathds{1}[\pi^j(t)=k] \mathds{1}[t \notin \mathcal{G}^{j}_1 \cup \mathcal{G}^{j}_2]  \,\middle|\, \overline{\mathcal{E}_0} \right] \\
    &\leq \sum_{j \le M}\sum_{t \in \mathcal{T}^j_{\mathrm{exp}}}  \mathbb{E}\left[\sum_{k\leq K } \mathds{1}[\pi^j(t)=k] \mathds{1}[t \notin \mathcal{G}^{j}_1 \cup \mathcal{G}^{j}_2 ] \,\middle|\, \overline{\mathcal{E}_0}\cap \overline{\mathcal{E}_3} \right] + \sum_{j \le M} T^j_{\mathrm{exp}} \Pr[\overline{\mathcal{E}_3} ] \\
    &\leq \sum_{j \le M}\sum_{t \in \mathcal{T}^j_{\mathrm{exp}}}  \mathbb{E}\left[\sum_{k\leq K } \mathds{1}[\pi^j(t)=k] \mathds{1}[t \notin \mathcal{G}^{j}_1 \cup \mathcal{G}^{j}_2 ]\mathds{1}\{\pi^j(t) \in \mathcal{K}_2(t) \} \,\middle|\, \overline{\mathcal{E}_0}\cap \overline{\mathcal{E}_3} \right] \\
    &\quad + \sum_{j \le M}\sum_{t \in \mathcal{T}^j_{\mathrm{exp}}}  \mathbb{E}\left[\sum_{k\leq K } \mathds{1}[\pi^j(t)=k] \mathds{1}[t \notin \mathcal{G}^{j}_1 \cup \mathcal{G}^{j}_2 ]\mathds{1}\{\pi^j(t) \notin \mathcal{K}_2(t) \} \,\middle|\, \overline{\mathcal{E}_0}\cap \overline{\mathcal{E}_3} \right] + \sum_{j \le M} T^j_{\mathrm{exp}} \Pr[\overline{\mathcal{E}_3} ] \\
    &\leq \underbrace{\sum_{j \le M}\sum_{t \in \mathcal{T}^j_{\mathrm{exp}}}  \mathbb{E}\left[\sum_{k\notin \mathcal{A}^j(t) } \mathds{1}[\pi^j(t)=k] \mathds{1}[t \notin \mathcal{G}^{j}_1 \cup \mathcal{G}^{j}_2 ]\mathds{1}\{\pi^j(t) \in \mathcal{K}_2(t) \} \,\middle|\, \overline{\mathcal{E}_0}\cap \overline{\mathcal{E}_3} \right]}_{\vbo} \\
    &\quad + \underbrace{\sum_{j \le M}\sum_{t \in \mathcal{T}^j_{\mathrm{exp}}}  \mathbb{E}\left[\sum_{k\in \mathcal{A}^j(t) } \mathds{1}[\pi^j(t)=k] \mathds{1}[t \notin \mathcal{G}^{j}_1 \cup \mathcal{G}^{j}_2 ]\mathds{1}\{\pi^j(t) \in \mathcal{K}_2(t) \} \,\middle|\, \overline{\mathcal{E}_0}\cap \overline{\mathcal{E}_3} \right]}_{\vbt} \\
    &\quad + \underbrace{\sum_{j \le M}\sum_{t \in \mathcal{T}^j_{\mathrm{exp}}}  \mathbb{E}\left[\sum_{k\leq K } \mathds{1}[\pi^j(t)=k] \mathds{1}[t \notin \mathcal{G}^{j}_1 \cup \mathcal{G}^{j}_2 ]\mathds{1}\{\pi^j(t) \notin \mathcal{K}_2(t), \pi^j(t) \in \mathcal{K}_1(t) \} \,\middle|\, \overline{\mathcal{E}_0}\cap \overline{\mathcal{E}_3} \right]}_{\vbth} \\
    &\quad + \underbrace{\sum_{j \le M}\sum_{t \in \mathcal{T}^j_{\mathrm{exp}}}  \mathbb{E}\left[\sum_{k\leq K } \mathds{1}[\pi^j(t)=k] \mathds{1}[t \notin \mathcal{G}^{j}_1 \cup \mathcal{G}^{j}_2 ]\mathds{1}\{\pi^j(t) \notin \mathcal{K}_2(t), \pi^j(t) \notin \mathcal{K}_1(t) \} \,\middle|\, \overline{\mathcal{E}_0}\cap \overline{\mathcal{E}_3} \right]}_{\vbf} \\
    &\quad + \underbrace{\sum_{j \le M} T^j_{\mathrm{exp}} \Pr[\overline{\mathcal{E}_3} ]}_{\vbfi} ~.
\end{align*}
Here \(\vbo\) corresponds to the regret incurred from regular exploration from $[K] \setminus \mathcal{A}^j(t)$ when Condition \ref{eq:cd2} is not satisfied. \(\vbt\) captures the regret associated with exploring arms that are already occupied, to determine whether they have been released, when Condition \ref{eq:cd2} is not satisfied. $\vbth$ is the regret that arms satisfy Condition \ref{eq:cd2} but has not satisfy Condition \ref{eq:cd1} yet. Note that $\vbf$ $=0$ since players will leave the exploration phase if both Condition \ref{eq:cd1} and Condition \ref{eq:cd2} are satisfied. $\vbfi$ accounts for the regret due to the bad event $\mathcal{E}_3$.

$\vbo$ is upper bounded as 
\begin{align}
    \vbo &= \sum_{j \le M}\sum_{t \in \mathcal{T}^j_{\mathrm{exp}}}  \mathbb{E}\left[\sum_{k\notin \mathcal{A}^j(t) } \mathds{1}[\pi^j(t)=k, \eta_k(t) = 0] \mathds{1}[t \notin \mathcal{G}^{j}_1 \cup \mathcal{G}^{j}_2 ]\mathds{1}\{\pi^j(t) \in \mathcal{K}_2(t) \} \,\middle|\, \overline{\mathcal{E}_0}\cap \overline{\mathcal{E}_3} \right] \nonumber\\
    &\quad + \sum_{j \le M}\sum_{t \in \mathcal{T}^j_{\mathrm{exp}}}  \mathbb{E}\left[\sum_{k\notin \mathcal{A}^j(t) } \mathds{1}[\pi^j(t)=k, \eta_k(t) = 1] \mathds{1}[t \notin \mathcal{G}^{j}_1 \cup \mathcal{G}^{j}_2 ]\mathds{1}\{\pi^j(t) \in \mathcal{K}_2(t) \} \,\middle|\, \overline{\mathcal{E}_0}\cap \overline{\mathcal{E}_3} \right] \nonumber\\
    &= \sum_{j \le M}  \mathbb{E}\left[ \sum_{t \in \mathcal{T}^j_{\mathrm{exp}}} \sum_{k\notin \mathcal{A}^j(t) } \mathds{1}[\pi^j(t)=k, \eta_k(t) = 0] \mathds{1}[t \notin \mathcal{G}^{j}_1 \cup \mathcal{G}^{j}_2 ]\mathds{1}\{\pi^j(t) \in \mathcal{K}_2(t) \} \,\middle|\, \overline{\mathcal{E}_0}\cap \overline{\mathcal{E}_3} \right] \nonumber\\
    &\quad + \sum_{j \le M} \mathbb{E}\left[\sum_{t \in \mathcal{T}^j_{\mathrm{exp}}}  \sum_{k\notin \mathcal{A}^j(t) } \mathds{1}[\pi^j(t)=k, \eta_k(t) = 1] \mathds{1}[t \notin \mathcal{G}^{j}_1 \cup \mathcal{G}^{j}_2 ]\mathds{1}\{\pi^j(t) \in \mathcal{K}_2(t) \} \,\middle|\, \overline{\mathcal{E}_0}\cap \overline{\mathcal{E}_3} \right] \nonumber\\
    &\leq \sum_{j \le M}  \sum_{k\leq K }  \left(\frac{288m\log(T)}{\Delta^2} \right)  + \sum_{j \le M} \sum_{k\leq K } \left( \frac{1-1/2e}{1/2e} \frac{288m\log(T)}{\Delta^2} \right) \label{eq:boundB-a} \\
    &\leq \frac{576emKM\log(T)}{\Delta^2} ~, \nonumber
\end{align}
where \eqref{eq:boundB-a} follows from Lemma \ref{lm:mlogT-3}, and the probability that player $j$ pulls an arm \(k \notin \mathcal{A}^j(t)\) during the exploration phase and encounters a collision is at most \(1 - \frac{1}{2e}\). %, as established in \eqref{eq:pdbar-c2}. 

Since player $j$ pulls an arm in $\mathcal{A}^j(t)$ with probability $\varepsilon$ during the exploration phase (Line \ref{l:nsame_1}, Algorithm \ref{alg:2}), $\vbt$ is upper bounded by $\sum_{j\leq M} \varepsilon T^j_{\mathrm{exp}}$.

$\vbth$ is bounded as
\begin{align}
    \vbth &= \sum_{j \le M}\sum_{t \in \mathcal{T}^j_{\mathrm{exp}}}  \mathbb{E}\left[\sum_{k\leq K } \mathds{1}[\pi^j(t)=k] \mathds{1}[t \notin \mathcal{G}^{j}_1 \cup \mathcal{G}^{j}_2 ]\mathds{1}\{\pi^j(t) \notin \mathcal{K}_2(t), \pi^j(t) \in \mathcal{K}_1(t) \} \,\middle|\, \overline{\mathcal{E}_0}\cap \overline{\mathcal{E}_3} \right] \nonumber \\
    %&\leq \sum_{j \le M}  \mathbb{E}\left[\sum_{t \in \mathcal{T}^j_{\mathrm{exp}}} \sum_{k\notin \mathcal{A}^j(t)} \mathds{1}[t \notin \mathcal{G}^{j}_1 \cup \mathcal{G}^{j}_2, \mathcal{A}^j(t) \ne \mathcal{A}^j(t+1)] \,\middle|\, \overline{\mathcal{E}_0}\cap \overline{\mathcal{E}_3} \right] \nonumber \\
    &\leq \sum_{j \le M}\sum_{t \in \mathcal{T}^j}   \mathds{1}\left[\mathcal{A}^j(t) \ne \mathcal{A}^j(t+1) \right] \mathbb{E}[T_c] \label{eq:bound-e3-c} \\
    &\leq 3m^2M \mathbb{E}[T_c] \label{eq:bound-e3-a} \\
    &\leq 12e^2m^2K^2M ~, \label{eq:bound-e3-b}
\end{align}
Note that each update of \(\mathcal{A}^j(t)\) may trigger a new phase transition between exploration and exploitation, and each such transition requires player~\(j\) to satisfy both Condition~\ref{eq:cd1} and Condition~\ref{eq:cd2}. Thus, \eqref{eq:bound-e3-c} is the product of (i) the number of times \(\mathcal{A}^j(t)\) changes and (ii) the number of steps required for player~\(j\) to satisfy both Condition~\ref{eq:cd1} after Condition~\ref{eq:cd2} has been met. \eqref{eq:bound-e3-a} is from Lemma \ref{lm:switch-times}. \eqref{eq:bound-e3-b} is derived directly from \eqref{eq:tao-d}.

By Lemma~\ref{lm:p-e-4}, \(\vbfi\) is bounded by \(2KM^2\). Combining the bounds for \(\vbo\) through \(\vbfi\), we obtain the final bound for \(\vb\).
\end{proof}

\Appvavc*

\begin{proof}
Recall that the definitions are
\begin{align*}
     \va = \sum_{j \le M} \mathbb{E}\left[|\mathcal{G}^{j}_2| \,\middle|\, \overline{\mathcal{E}_0} \right] ,\quad
    \vc  = \sum_{j \le M}\sum_{t \in \mathcal{T}^j_{\mathrm{exp}}}  \mathbb{E}\left[\mathds{1}[t \in \mathcal{G}^{j}_1] \,\middle|\, \overline{\mathcal{E}_0} \right] ~,
\end{align*}
where
\begin{align*}
    & \mathcal{G}^{j}_1 = \left\{ T^j_\mathrm{start} \le  t \le T^j_\mathrm{end} : \exists j^{\prime}\neq j, j^{\prime}\in[M],\exists k\leq K, k\notin \mathcal{A}^j(t), \hat{k}^{j^{\prime}}(t) = k \right\} ~,\\
    & \mathcal{G}^{j}_2 = \left\{T^j_\mathrm{start} \le  t \le T^j_\mathrm{end}: \exists k\in \mathcal{A}^j(t), \forall j^{\prime}\neq j, j^{\prime}\in [M], \hat{k}^{j^{\prime}}\ne k \right\} ~.
\end{align*}
\(\va\) denotes the regret incurred when players have not yet removed released arms from \(\mathcal{A}^j(t)\). \(\vc\) corresponds to the regret caused by not yet adding occupied arms into \(\mathcal{A}^j(t)\).

We begin by analyzing the regret term $\va$.
\begin{align}
    \va &= \sum_{j\leq M} \mathbb{E}\left[\sum_{t\in \mathcal{T}^j} \mathds{1}[\exists k\in \mathcal{A}^j(t), \forall j^{\prime}\neq j, j^{\prime} \in [M], \hat{k}^{j^{\prime}}\ne k] \,\middle|\, \overline{\mathcal{E}_0}  \right] \nonumber \\
    &\leq  \sum_{j\leq M} \mathbb{E}\left[\sum_{t\in \mathcal{T}^j} \mathds{1}[\exists k\in \mathcal{A}^j(t), \forall j^{\prime}\neq j, j^{\prime} \in [M], \hat{k}^{j^{\prime}}\ne k] \,\middle|\, \overline{\mathcal{E}_2}\cap \overline{\mathcal{E}_0}  \right] + T^j\Pr[\mathcal{E}_2] \nonumber \\
    &\leq \sum_{j\leq M} \sum_{t \in \mathcal{T}^j} \frac{1}{3} \mathds{1}\left[\mathcal{A}^j(t) \ne \mathcal{A}^j(t+1) \right] \mathbb{E}[T_r^j] + \sum_{j\leq M}T^j \frac{2KM}{T} \label{eq:boundA-a} \\
    &\leq m^2M \mathbb{E}[T_r^j] + 2KM^2 \nonumber \\
    &\leq \frac{1141m^3M\ln(T)}{\varepsilon} + 2KM^2  ~, \nonumber
\end{align} 
where \eqref{eq:boundA-a} holds because when an arm $k$ is released, player $j$ will removes it from $\mathcal{A}^j(t)$ after $\mathbb{E}[T^j_r]$ time steps in expectation under the condition of $\overline{\mathcal{E}_2}$, as formally established in Lemma \ref{lm:temp-step-new}. 
Moreover, the total number of times arms are removed from $\mathcal{A}^j(t)$ is upper bounded by $m^2M$, as stated in Lemma \ref{lm:switch-times}. 
% \eqref{eq:boundA-b} comes from Lemma \ref{lm:switch-times}.

% 

% The bound in \eqref{eq:boundA-b} uses the fact that $S \leq m^2M$, since arm releases are triggered by the permanent departure of some player. Each such departure may cause the remaining players to switch from the exploitation phase to the exploration phase, during which at most $m$ players may remove arms $\mathcal{A}^j$, and each player can remove up to $m$ arms. Since at most $M$ players can permanently leave the system, the total number of removals is bounded by $m^2M$.

Next, we turn to bounding $\vc$.
\begin{align}
    \vc &= \sum_{j \le M} \mathbb{E}\left[ \sum_{t \in \mathcal{T}^j_{\mathrm{exp}}} \mathds{1} [\exists j^{\prime}\neq j, j^{\prime}\in[M],\exists k\leq K, k\notin \mathcal{A}^j(t), \hat{k}^{j^{\prime}}(t) = k ] \,\middle|\, \overline{\mathcal{E}_0} \right] \nonumber \\
    &\leq \sum_{j \le M} \mathbb{E}\left[ \sum_{t \in \mathcal{T}^j_{\mathrm{exp}}} \mathds{1} [\exists j^{\prime}\neq j, j^{\prime}\in[M],\exists k\leq K, k\notin \mathcal{A}^j(t), \hat{k}^{j^{\prime}}(t) = k ] \,\middle|\, \overline{\mathcal{E}_1} \cap \overline{\mathcal{E}_0} \right] + T^j_{\mathrm{exp}}\Pr[\mathcal{E}_1]  \nonumber \\
    &\leq \sum_{j\leq M} \sum_{t \in \mathcal{T}^j} \frac{2}{3} \mathds{1}\left[\mathcal{A}^j(t) \ne \mathcal{A}^j(t+1) \right] \mathbb{E}[T_o^j]  + \sum_{j\leq M}T^j \frac{2KM}{T} \label{eq:boundC-a} \\
    &\leq 2m^2M\mathbb{E}[T_o^j]  + 2KM^2 \nonumber \\
    &\leq 3852m^2KM\ln(T) + 2KM^2 ~, \nonumber 
\end{align}
where \eqref{eq:boundC-a} follows from the fact that when an arm becomes occupied, player $j$ adds it to $\mathcal{A}^j(t)$ after $E[T^j_o]= 964K\ln(T)$ time steps under the condition of $\overline{\mathcal{E}_1}$, as established in Lemma \ref{lm:temp-step-new}. The total number of times arms are added into $\mathcal{A}^j(t)$ is upper bounded by $2m^2M$, as stated in Lemma~\ref{lm:switch-times}.
% \eqref{eq:boundC-b} comes from Lemma \ref{lm:switch-times}.

% 

Finally, combining the bounds of $\va$ and $\vc$ leads to the desired result.
\end{proof}

\Appvd*

\begin{proof}
Recall the definition of $\vd$ is
\begin{align*}
    \vd = \sum_{j \le M}\sum_{t \in \mathcal{T}^j_{\mathrm{explt}}}  \mathbb{E}\left[\left(1 - \mathds{1}[\pi^j(t) \le m_t, \eta_{\pi^j(t)}(t) = 0]\right)\mathds{1}[t \notin \mathcal{G}^{j}_2] \,\middle|\,  \overline{\mathcal{E}_0} \right] ~,
\end{align*}
which represents the regret incurred during the exploitation phase. When player $j$ is in this phase, she either pulls her estimated best arm $\hat{k}^j$ or pulls an arm uniformly sampled from $\mathcal{A}^j(t)$. Accordingly, we decompose $\vd$ as follows:
\begin{align}
    \vd &\le \sum_{j \le M}\sum_{t \in \mathcal{T}^j_{\mathrm{explt}}}  \mathbb{E}\left[ \left(1 - \mathds{1}[\pi^j(t) \le m_t, \eta_{\pi^j(t)}(t) = 0]\right) \mathds{1}[\pi^j(t)= \hat{k}^j, t \notin \mathcal{G}^{j}_2] \,\middle|\,  \overline{\mathcal{E}_0} \right] \nonumber \\
    &\quad + \sum_{j \le M}\sum_{t \in \mathcal{T}^j_{\mathrm{explt}}}  \mathbb{E}\left[ \left(1 - \mathds{1}[\pi^j(t) \le m_t, \eta_{\pi^j(t)}(t) = 0]\right)  \mathds{1}[\pi^j(t) \in \mathcal{A}^j(t), t \notin \mathcal{G}^{j}_2] \,\middle|\,  \overline{\mathcal{E}_0} \right] \nonumber \\
    &\leq \sum_{j \le M}\sum_{t \in \mathcal{T}^j_{\mathrm{explt}}}  \mathbb{E}\left[ \left(1- \mathds{1}[ \eta_{\hat{k}^j}(t) = 0] \right) \mathds{1}[\pi^j(t)= \hat{k}^j, t \notin \mathcal{G}^{j}_2] \,\middle|\,  \overline{\mathcal{E}_0} \right] \label{eq:dmp-D-a} \\ 
    &\quad + \sum_{j \le M}\sum_{t \in \mathcal{T}^j_{\mathrm{explt}}}  \mathbb{E}\left[ \mathds{1}[\pi^j(t) \in \mathcal{A}^j(t), t \notin \mathcal{G}^{j}_2] \,\middle|\,  \overline{\mathcal{E}_0} \right] \nonumber \\
    &\leq \sum_{j \le M}\sum_{t \in \mathcal{T}^j_{\mathrm{explt}}}  \mathbb{E}\left[ \eta_{\hat{k}^j}(t)\mathds{1}[\pi^j(t)= \hat{k}^j, t \notin \mathcal{G}^{j}_2] \,\middle|\,  \overline{\mathcal{E}_0} \right] + \sum_{j \le M}\varepsilon T^j_{\mathrm{explt}}  \label{eq:dmp-D-b} \\ 
    &\le \underbrace{\sum_{j \le M}\sum_{t \in \mathcal{T}^j_{\mathrm{explt}}}  \mathbb{E}\left[ \eta_{\hat{k}^j}(t) \mathds{1}[\pi^j(t)= \hat{k}^j, \exists j^\prime\ne j,j^\prime\in[M], \hat{k}^j\notin \mathcal{A}^{j^\prime}(t) , t \notin \mathcal{G}^{j}_2] \,\middle|\,  \overline{\mathcal{E}_0} \right]}_{\vdo} \nonumber \\
    &\quad + \underbrace{\sum_{j \le M}\sum_{t \in \mathcal{T}^j_{\mathrm{explt}}}  \mathbb{E}\left[ \eta_{\hat{k}^j}(t)  \mathds{1}[\pi^j(t)= \hat{k}^j, \forall j^\prime\ne j, j^\prime\in[M], \hat{k}^j\in \mathcal{A}^{j^\prime}(t) , t \notin \mathcal{G}^{j}_2 ] \,\middle|\,  \overline{\mathcal{E}_0} \right]}_{\vdt} \nonumber \\
    &\quad + \sum_{j \le M}\varepsilon T^j_{\mathrm{explt}}~,  \nonumber 
\end{align}
where \eqref{eq:dmp-D-a} follows from the fact that \(t \notin \mathcal{G}_2\) implies no arms are mistakenly included in \(\mathcal{A}^j(t)\). Therefore, when player \(j\) is in the exploitation phase, she is exploiting an arm \(\hat{k}^j \leq m_t\). 
\eqref{eq:dmp-D-b} holds because $(1 - \mathds{1}[\eta_{\pi^j(t)}(t) = 0]) = \eta_{\pi^j(t)}(t)$, and player $j$ pulls arms from $\mathcal{A}^j(t)$ with probability $\varepsilon$ during the exploitation phase.

At the last inequality, $\vdo$ represents the regret caused when a newly joining player $j^\prime$ has not yet added $\hat{k}^j$ to her set $\mathcal{A}^{j^\prime}(t)$, explores $\hat{k}^j$, and collides with player $j$. The term $\vdt$ accounts for the regret incurred when any player $j^\prime$ has already added $\hat{k}^j$ into $\mathcal{A}^{j^\prime}(t)$, but still selects $\hat{k}^j$ and collides with player $j$.

$\vdo$ is bounded as
\begin{align}
    \vdo &= \sum_{j \le M} \mathbb{E}\left[ \sum_{t \in \mathcal{T}^j_{\mathrm{explt}}}  \mathds{1}[\pi^j(t)= \hat{k}^j, t \notin \mathcal{G}^{j}_2] \mathds{1}[ \exists j^\prime\ne j, j^\prime\in[M], \pi^{j^\prime}(t)= \hat{k}^j, \hat{k}^j\notin \mathcal{A}^{j^\prime}(t) ] \,\middle|\,  \overline{\mathcal{E}_0} \right] \nonumber \\
    &\leq \sum_{j \le M} \mathbb{E}\left[ \sum_{t \in \mathcal{T}^j_{\mathrm{explt}}}  \mathds{1}[\pi^j(t)= \hat{k}^j, \exists j^\prime\ne j, j^\prime\in[M], \pi^{j^\prime}(t)= \hat{k}^j, \hat{k}^j\notin \mathcal{A}^{j^\prime}(t) ] \,\middle|\,  \overline{\mathcal{E}_0} \right] \nonumber \\
    &= \sum_{j^{\prime} \le M} \mathbb{E}\left[ \sum_{t \in \mathcal{T}^{j^{\prime}}_{\mathrm{exp}}} \mathds{1} [\exists j\neq j^{\prime}, j\leq M,\exists k\leq K, k\notin \mathcal{A}^{j^{\prime}}(t), \hat{k}^{j}(t) = k ] \,\middle|\, \overline{\mathcal{E}_0} \right] \label{eq:boundD-o-a} \\
    &= \vc  \nonumber \\
    &\leq 3852m^2KM\ln(T) + 2KM^2  ~, \nonumber 
\end{align}
where \eqref{eq:boundD-o-a} holds because for any collision event on arm \( k \) between an exploiting player \( j \) and an exploring player \( j^\prime \), it is equivalent to count the event from the perspective of \( j^\prime \), who pulls arm \( k \notin \mathcal{A}^{j^\prime}(t) \) while \( k = \hat{k}^j \) is being exploited by player \( j \). The rest of the analysis is identical to that of \(\vc\).

Next, we proceed to bound $\vdt$.
\begin{align}
    \vdt &= \sum_{j \le M}\sum_{t \in \mathcal{T}^j_{\mathrm{explt}}}  \mathbb{E}\left[ \eta_{\hat{k}^j}(t)  \mathds{1}[\pi^j(t)= \hat{k}^j, \forall j^\prime\ne j, \hat{k}^j\in \mathcal{A}^{j^\prime}(t) , t \notin \mathcal{G}^{j}_2 ] \,\middle|\,  \overline{\mathcal{E}_0} \right] \nonumber \\
    &= \sum_{j \le M}\sum_{t \in \mathcal{T}^j_{\mathrm{explt}}}  \mathbb{E}\left[ \mathds{1}[\pi^j(t)= \hat{k}^j, t \notin \mathcal{G}^{j}_2] \mathds{1}[ \forall j^\prime\ne j,\pi^{j^\prime}=\hat{k}^j, \hat{k}^j\in \mathcal{A}^{j^\prime}(t) ] \,\middle|\,  \overline{\mathcal{E}_0} \right] \nonumber \\
    &\leq \sum_{j \le M} \max_{j^\prime\leq M} \varepsilon T^{j^\prime}_{\mathrm{explt}} ~, \label{eq:boundD-t-a}
\end{align}
where \eqref{eq:boundD-t-a} is since player $j^\prime$ pulls arms in $\mathcal{A}^{j^\prime}(t)$ with probability $\varepsilon$.
\end{proof}

We now aggregate the bounds for \(\va\), \(\vb\), \(\vc\), \(\vd\), and \(\ve\) to obtain the total regret \(R(T)\).
\begin{align}
    R(T) &\leq \frac{576emKM\log(T)}{\Delta^2} + 12e^2m^2K^2M + 2KM^2 + \sum_{j\leq M} \varepsilon T^j_{\mathrm{exp}}  \nonumber \\
    &\quad +  \frac{1141m^3M\ln(T)}{\varepsilon} +  3852m^2KM\ln(T) + 4KM^2 + 2KM^2 \nonumber \\
    &\quad + 3852m^2KM\ln(T) + 2KM^2 + \sum_{j \le M}  \varepsilon  ( T^j_{\mathrm{explt}} + \max_{j^\prime\leq M} T^{j^\prime}_{\mathrm{explt}} )\nonumber \\
    &\leq \frac{576emKM\log(T)}{\Delta^2} + 7704m^2KM\ln(T) + (4emKM)^2 \nonumber \\
    &\quad + \frac{1141m^3M\ln(T)}{\varepsilon} + 2\varepsilon MT \label{eq:all-b} \\
    &\leq \frac{576emKM\log(T)}{\Delta^2} + 7704m^2KM\ln(T) + (4emKM)^2 \nonumber \\
    &\quad + 96m^{3/2}M\sqrt{T\ln(T)} \label{eq:all-a} ~,
\end{align}
where \eqref{eq:all-a} follows from the choice $\varepsilon= \sqrt{1141m^3\ln(T) / 2T}$.

\textbf{Discussion on Unknown $m$}

Although algorithm \ref{alg:1} requires \(m\) as input, it can still function when \(m\) is unknown. By Assumption \ref{Assumption:active_players}, \(m\) is upper-bounded by \(K/2\). Therefore, in Line \ref{l:cor-grt}, we can simply replace \(m\) with \(K/2\), and the algorithm remains valid.

In the theoretical analysis, when \(m\) is unknown, the set \(\mathcal{A}^j\) may contain up to \(K/2 - 1\) arms, since some released arms may not have been removed yet. As a result, it holds that
\begin{align}
    \sum_{j\le M}\sum_{t \in \mathcal{T}^j} \mathds{1}\left[\mathcal{A}^j(t) \ne \mathcal{A}^j(t+1) \right] \leq 3(\frac{K}{2})^2M \leq \frac{3K^2M}{4}  ~,
    \label{eq:switch-times-new}
\end{align}
Finally, setting $\varepsilon= \sqrt{1141K^3\ln(T) / 16T}$ leads to the following corollary:
\begin{corollary}
Given $K$ arms and $M$ players, $\varepsilon = \min\{ \sqrt{\frac{1141K^3\ln(T)}{16T}}, \frac{1}{K}, \frac{1}{10} \}$, the regret of Algorithm \ref{alg:1} is bounded by
    \begin{align*}
    R(T) \leq & \frac{288eK^2M\log(T)}{\Delta^2} + 34K^{3/2}M\sqrt{T\ln(T)} + 1926K^3M\ln(T) + (3eK^2M)^2  ~,
\end{align*}
where $\Delta = \min_{k\leq m}(\mu_k- \mu_{k+1})$.
\label{cor:upper}
\end{corollary}

\section{Technical Lemmas}

\begin{lemma}[Hoeffding's Inequality]
Let $X_1,...,X_N$ be i.i.d variables with $X_i\in[0,1]$ for any $i\leq N$. Define $\hat{\mu}:= \frac{1}{N} \sum_{i\leq N} X_i$. Denote the expectation of $X_i$ by $\mu$. For any $\delta>0$,
\begin{align*}
    \Pr[|\hat{\mu} - \mu| \geq \delta] \leq 2\exp(-{2N\delta^2}) ~.
\end{align*}
\label{lm:hfd}
\end{lemma}

\begin{lemma}[Hoeffding's Inequality for Sum]
Let $X_1,...,X_N$ be independent variables with $X_i\in[0,1]$ for any $i\leq N$. Define $S_N:= \sum_{i\leq N} X_i$. For any $t>0$,
\begin{align*}
    \Pr\left[|S_N- \mathbb{E}[S_N] | \geq t\right] \leq 2\exp\left(-\frac{2t^2}{N}\right) ~.
\end{align*}
\label{lm:hfd-sum}
\end{lemma}

\begin{lemma}[Chernoff Bound] 
Let $X_1,...,X_N$ be independent variables with $X_i\sim \mathrm{Bernoulli}(p_i)$. Define $S_N:= \sum_{i\leq N} X_i$. $\mathbb{E}[S] = \sum_{i\leq N} p_i$. For any $\delta>0$,
\[
\Pr\left[ |S_N- \mathbb{E}[S_N]| \geq \delta \mathbb{E}[S_N] \right] \leq 2\exp\left( -\frac{\delta^2\mathbb{E}[S_N]}{3} \right) ~.
\]
\label{lm:chr}
\end{lemma}

\section{Experimental Details} \label{app:exp}

This section provides additional details of the experiments.

\begin{table}[t]
  \centering
  % ---- Row 1 ----
\begin{subtable}[b]{0.24\textwidth}
  \centering
  \begin{tabular}{ccc}
    \toprule
    \textbf{j} & \textbf{Start} & \textbf{End} \\
    \midrule
    1  & 25419  & 1107891 \\
    2  & 522732 & 1427541 \\
    3  & 770967 & 1493795 \\
    4  & 760785 & 1561277 \\
    5  & 119594 & 1713244 \\
    6  & 887212 & 1472214 \\
    7  & 729606 & 1637557 \\
    8  & 310982 & 1325183 \\
    9  & 330898 & 1063558 \\
    10 & 863103 & 1623298 \\
    11 & 358465 & 1115869 \\
    12 & 771270 & 1074044 \\
    \bottomrule
  \end{tabular}
  \\[1em]
  \caption{Players 1--12}
  \label{tb:players-1-12}
\end{subtable}
\hfill
% ---- Row 2 ----
\begin{subtable}[b]{0.24\textwidth}
  \centering
  \begin{tabular}{ccc}
    \toprule
    \textbf{j} & \textbf{Start} & \textbf{End} \\
    \midrule
    13 & 706857 & 1729007 \\
    14 & 5522   & 1815461 \\
    15 & 772244 & 1198715 \\
    16 & 74550  & 1986886 \\
    17 & 140924 & 1802196 \\
    18 & 280934 & 1542696 \\
    19 & 828737 & 1356753 \\
    20 & 388677 & 1271349 \\
    21 & 45227  & 1325330 \\
    22 & 88492  & 1195982 \\
    23 & 597899 & 1921874 \\
    24 & 939498 & 1894827 \\
    \bottomrule
  \end{tabular}
  \\[1em]
  \caption{Players 13--24}
  \label{tb:players-13-24}
\end{subtable}
\hfill
% ---- Row 3 ----
\begin{subtable}[b]{0.24\textwidth}
  \centering
  \begin{tabular}{ccc}
    \toprule
    \textbf{j} & \textbf{Start} & \textbf{End} \\
    \midrule
    25 & 969584 & 1775132 \\
    26 & 546710 & 1184854 \\
    27 & 311711 & 1520068 \\
    28 & 258779 & 1662522 \\
    29 & 34388  & 1909320 \\
    30 & 122038 & 1495176 \\
    31 & 684233 & 1440152 \\
    32 & 304613 & 1097672 \\
    33 & 965632 & 1808397 \\
    34 & 65051  & 1948885 \\
    35 & 607544 & 1170524 \\
    36 & 592414 & 1046450 \\
    37 & 199673 & 1514234 \\
    \bottomrule
  \end{tabular}
  \\[1em]
  \caption{Players 25--37}
  \label{tb:players-25-37}
\end{subtable}
\hfill
% ---- Row 4 ----
\begin{subtable}[b]{0.24\textwidth}
  \centering
  \begin{tabular}{ccc}
    \toprule
    \textbf{j} & \textbf{Start} & \textbf{End} \\
    \midrule
    38 & 456069 & 1785175 \\
    39 & 292144 & 1366361 \\
    40 & 611852 & 1139493 \\
    41 & 431945 & 1291229 \\
    42 & 304242 & 1524756 \\
    43 & 181824 & 1183404 \\
    44 & 832442 & 1212339 \\
    45 & 20584  & 1969909 \\
    46 & 601115 & 1708072 \\
    47 & 58083  & 1866176 \\
    48 & 156018 & 1155994 \\
    49 & 731993 & 1598658 \\
    50 & 374540 & 1950714 \\
    \bottomrule
  \end{tabular}
  \\[1em]
  \caption{Players 38--50}
  \label{tb:players-38-50}
\end{subtable}
\caption{Players' Active periods for comparison on varying $M$ under the random asynchronization setting.}
\label{tb:m-set-r}
\end{table}

\begin{table}[ht]
  \centering
  % ---- Row 1 ---- (synthetic data)
\begin{subtable}[t]{0.24\textwidth}
  \centering
  \begin{tabular}{ccc}
    \toprule
    \textbf{j} & \textbf{Start} & \textbf{End} \\
    \midrule
    1--3   & $0$             & $1 \times 10^5$ \\
    4--6   & $8 \times 10^4$ & $2 \times 10^6$ \\
    7--10  & $0$             & $2 \times 10^6$ \\
    \bottomrule
  \end{tabular}
  \\[1em]
  \caption{M=10.}
  \label{tb:m-set-s-10}
\end{subtable}
\hfill
\begin{subtable}[t]{0.24\textwidth}
  \centering
  \begin{tabular}{ccc}
    \toprule
    \textbf{j} & \textbf{Start} & \textbf{End} \\
    \midrule
    1--7   & $0$             & $1 \times 10^5$ \\
    8--13  & $8 \times 10^4$ & $2 \times 10^6$ \\
    14--20 & $0$             & $2 \times 10^6$ \\
    \bottomrule
  \end{tabular}
  \\[1em]
  \caption{M=20.}
  \label{tb:m-set-s-20}
\end{subtable}
% \hfill
% \begin{subtable}[t]{0.25\textwidth}
%   \centering
%   \begin{tabular}{ccc}
%     \toprule
%     \textbf{j} & \textbf{Start} & \textbf{End} \\
%     \midrule
%     1--10   & $0$             & $1 \times 10^5$ \\
%     11--20  & $8 \times 10^4$ & $2 \times 10^6$ \\
%     21--30  & $0$             & $2 \times 10^6$ \\
%     \bottomrule
%   \end{tabular}
%   \\[1em]
%   \caption{M=30.}
%   \label{tb:m-set-s-30}
% \end{subtable}
\hfill
\begin{subtable}[t]{0.25\textwidth}
  \centering
  \begin{tabular}{ccc}
    \toprule
    \textbf{j} & \textbf{Start} & \textbf{End} \\
    \midrule
    1--17   & $0$             & $1 \times 10^5$ \\
    18--33  & $8 \times 10^4$ & $2 \times 10^6$ \\
    34--50  & $0$             & $2 \times 10^6$ \\
    \bottomrule
  \end{tabular}
  \\[1em]
  \caption{M=50.}
  \label{tb:m-set-s-50}
\end{subtable}
\caption{Players' active periods for comparison on varying $M$ under the synthetic asynchronization setting.}
\label{tb:m-set-s}
\end{table}

\subsection{Implementation of Baseline Algorithms}

D-MC requires a shared clock across all players to synchronize the start of new epochs, and the size of the exploration phase in an epoch needs to depend on the lower bound of $\Delta$. To ensure a fair comparison in our decentralized and asynchronous setting, we implement D-MC such that players do not have access to a global clock. Instead, each player resets its epoch every \(T / 5\) steps, and half of an epoch is used to explore.
MCTopM is not designed for asynchronous environments and assumes knowledge of the total number of players \(M\). In our setting, 
this corresponds to knowing \(m_t\) for any time step \(t\). Since MCTopM cannot estimate \(m_t\) like D-MC does, and in accordance with our algorithm which takes $m$ as input, we set each player's estimate of \(m_t\) to \(m\) throughout the experiment. 

\subsection{Comparison of Number of Players}

\textbf{Setup: } 
We evaluate the performance of our algorithm under different numbers of players, with \(M = 10\), \(20\), and \(50\). The environment consists of Gaussian bandits, where the reward of each arm \(k\) is drawn from \(\mathcal{N}(\mu_k, 0.5^2)\). The smallest mean is fixed at \(\mu_{K} = 0.1\), and the gap between adjacent arms is set to \(0.05\). The number of arms is fixed at \(K = 100\) across all experiments. 
Both random and synthetic asynchronous scenarios are considered. For the random asynchronous setting, we use the same generation process described earlier, with the same random seed but a larger number of players, \(M = 50\). Specifically, each player \(j\) is active from time step \(T^j_{\mathrm{start}} \in [1, T/2]\) to \(T^j_{\mathrm{end}} \in [T/2, T]\), with \(T^j_{\mathrm{end}} - T^j_{\mathrm{start}} \geq T/50\). The resulting data is provided in Table~\ref{tb:m-set-r}, where the experiment with \(M = 10\) uses the first 10 players, and those with \(M = 20\) and \(50\) use the first 20 and 50 players, respectively. The synthetic case is manually constructed and summarized in Table~\ref{tb:m-set-s}.

\update{
\textbf{Result Analysis on Figure \ref{fig:r-m}:} $\ \ $ 
Figure~\ref{fig:r-m-10}-\ref{fig:r-m-50} presents the performance under different values of \(M\) in the random asynchronous setting. As \(M\) increases, all algorithms exhibit higher regret. This is primarily due to the decentralized nature of the environment, where players cannot communicate directly and therefore tend to explore independently. All algorithms exhibit slow regret growth near the end of the time horizon, which is expected since players gradually leave the system in this setting. Once no players remain active, regret accumulation naturally stops. Both ACE and UCB exhibit early convergence compared to SMAA, GoT, DYN-MMAB, D-MC, and MCTopM.

From Figure~\ref{fig:s-m-10} to Figure \ref{fig:s-m-50}, we compare the performance under different values of \(M\) in the synthetic asynchronous setting. D-MC shows a phase-wise growth pattern, with distinct stages of regret increase. UCB maintains lower regret in the early stages but shows linear regret growth later. 
In contrast, ACE demonstrates both stability and strong convergence across all settings, including varying levels of asynchrony and different numbers of players. This robustness makes it a highly reliable choice in decentralized and asynchronous environments.

\textbf{Result Analysis on Figure \ref{fig:ucb-m-rdm}:} $\ \ $ 
We compare ACE with different types of UCB algorithms with various parameters in Figure \ref{fig:ucb-m-rdm}. 
%ACE displays a similar multi-phase behavior when \(M\) is small (Figure~\ref{fig:ucb-m-rdm-10}). %This behavior fades for larger \(M\), as the lack of communication in decentralized settings forces each player to explore more independently, slowing down convergence. 
While UCB algorithms demonstrate superior performance in the random asynchronous setting (Figure \ref{fig:ucb-m-rdm-10}–\ref{fig:ucb-m-rdm-50}), they still incur linearly increasing regret in the synthetic asynchronous setting (Figure \ref{fig:ucb-m-syn-10}-\ref{fig:ucb-m-syn-50}). In contrast, ACE converges and eventually outperforms UCB algorithms in the synthetic asynchronous setting.
}

\begin{figure}[th]
    \centering
    \begin{subfigure}[b]{0.33\textwidth}
        \includegraphics[width=\textwidth]{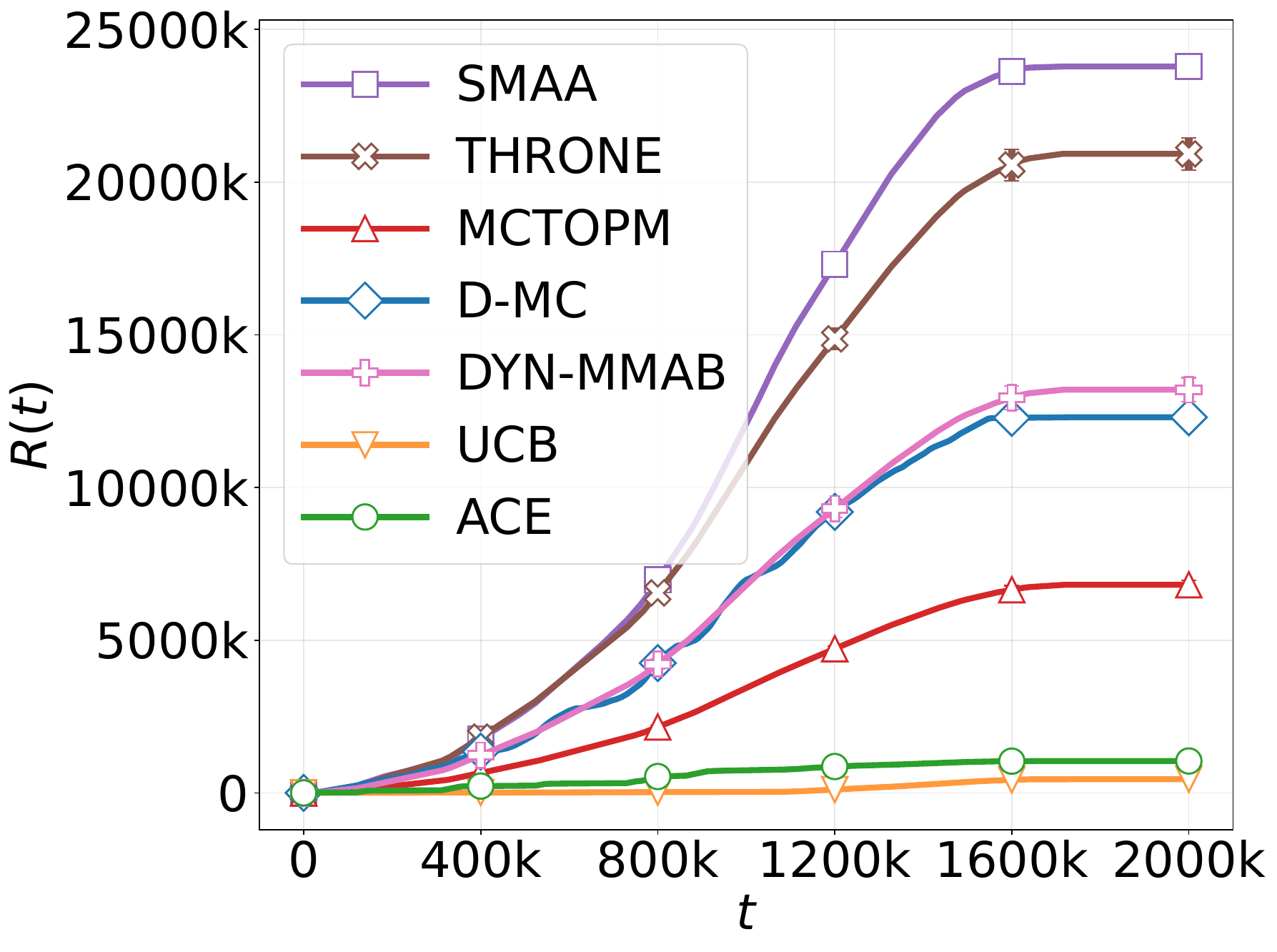}
        \caption{M=10, random.}
        \label{fig:r-m-10}
    \end{subfigure}
    \hfill
    \begin{subfigure}[b]{0.33\textwidth}
        \includegraphics[width=\textwidth]{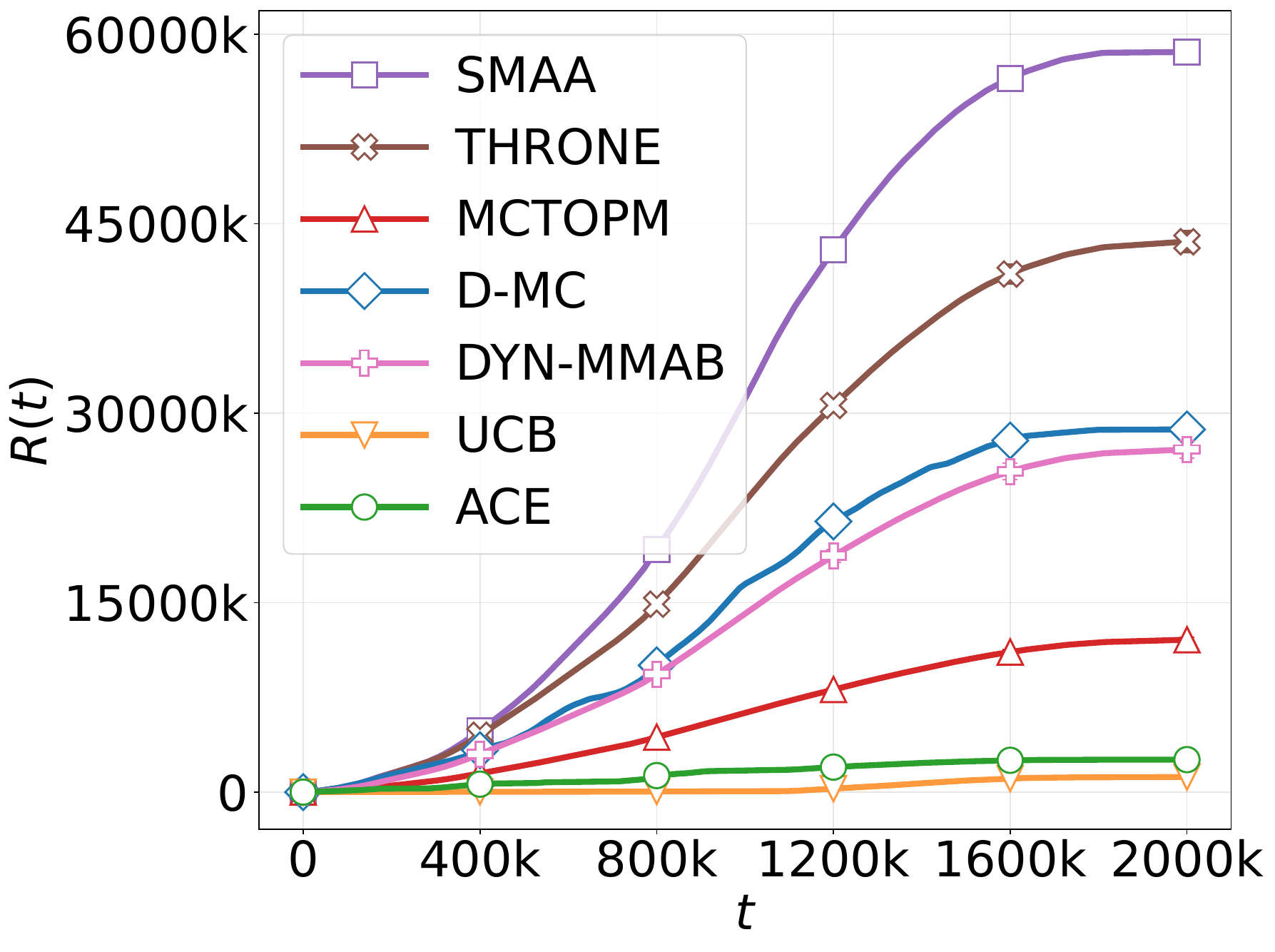}
        \caption{M=20, random.}
        \label{fig:r-m-20}
    \end{subfigure}
    \hfill
    \begin{subfigure}[b]{0.33\textwidth}
        \includegraphics[width=\textwidth]{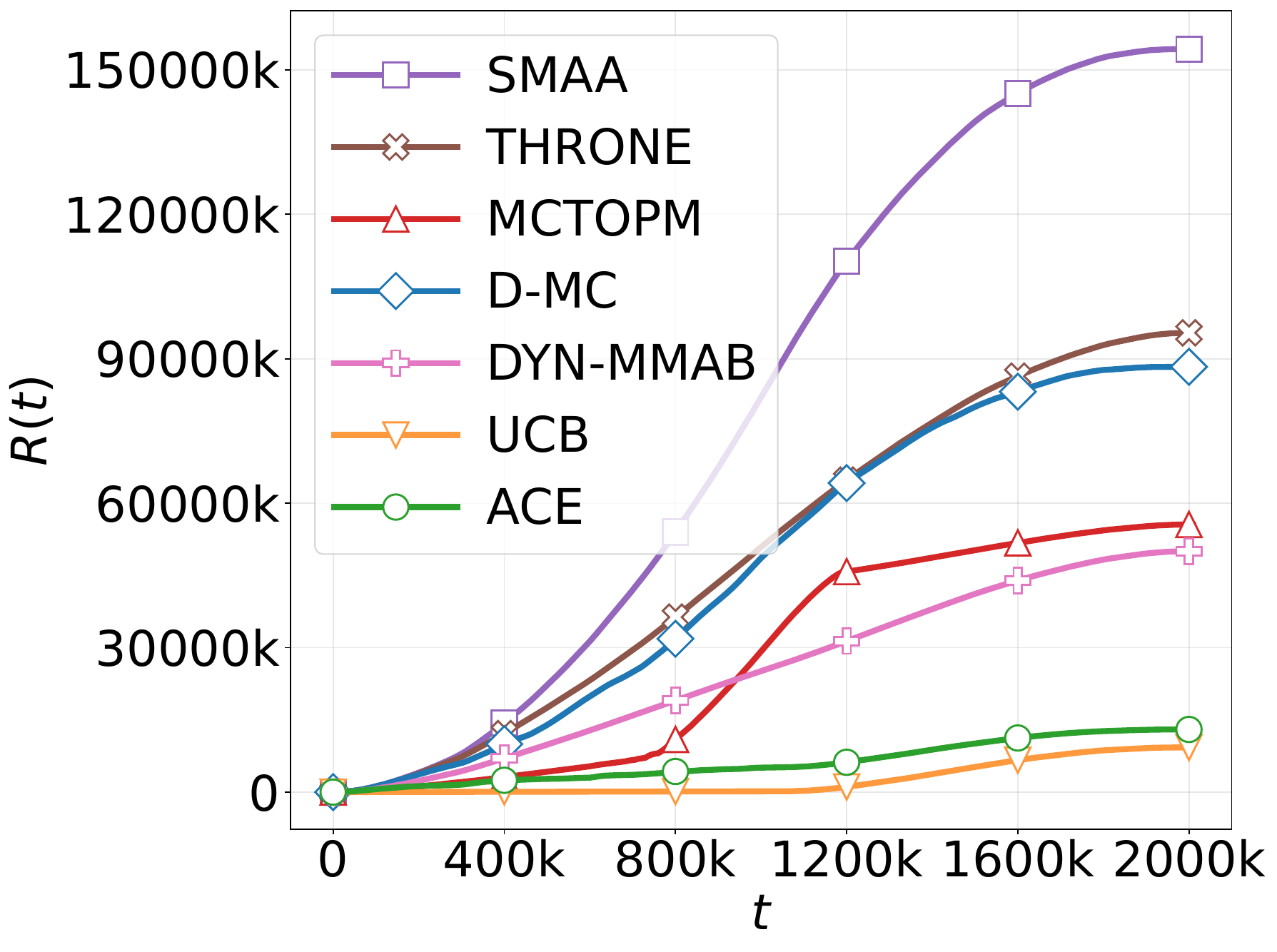}
        \caption{M=50, random.}
        \label{fig:r-m-50}
    \end{subfigure}
    \\[1em]
    \begin{subfigure}[b]{0.33\textwidth}
        \includegraphics[width=\textwidth]{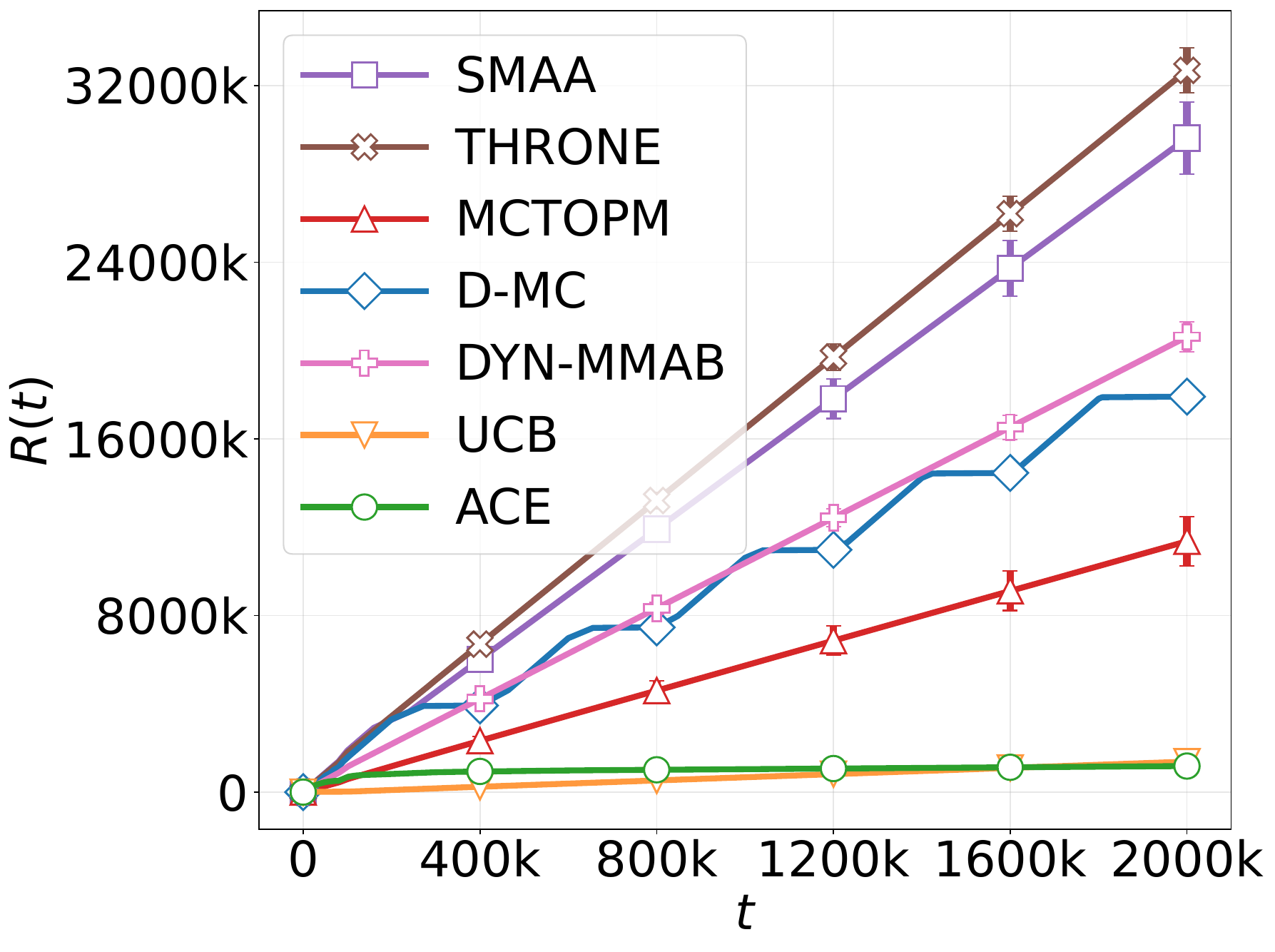}
        \caption{M=10, synthetic.}
        \label{fig:s-m-10}
    \end{subfigure}
    \hfill
    \begin{subfigure}[b]{0.33\textwidth}
        \includegraphics[width=\textwidth]{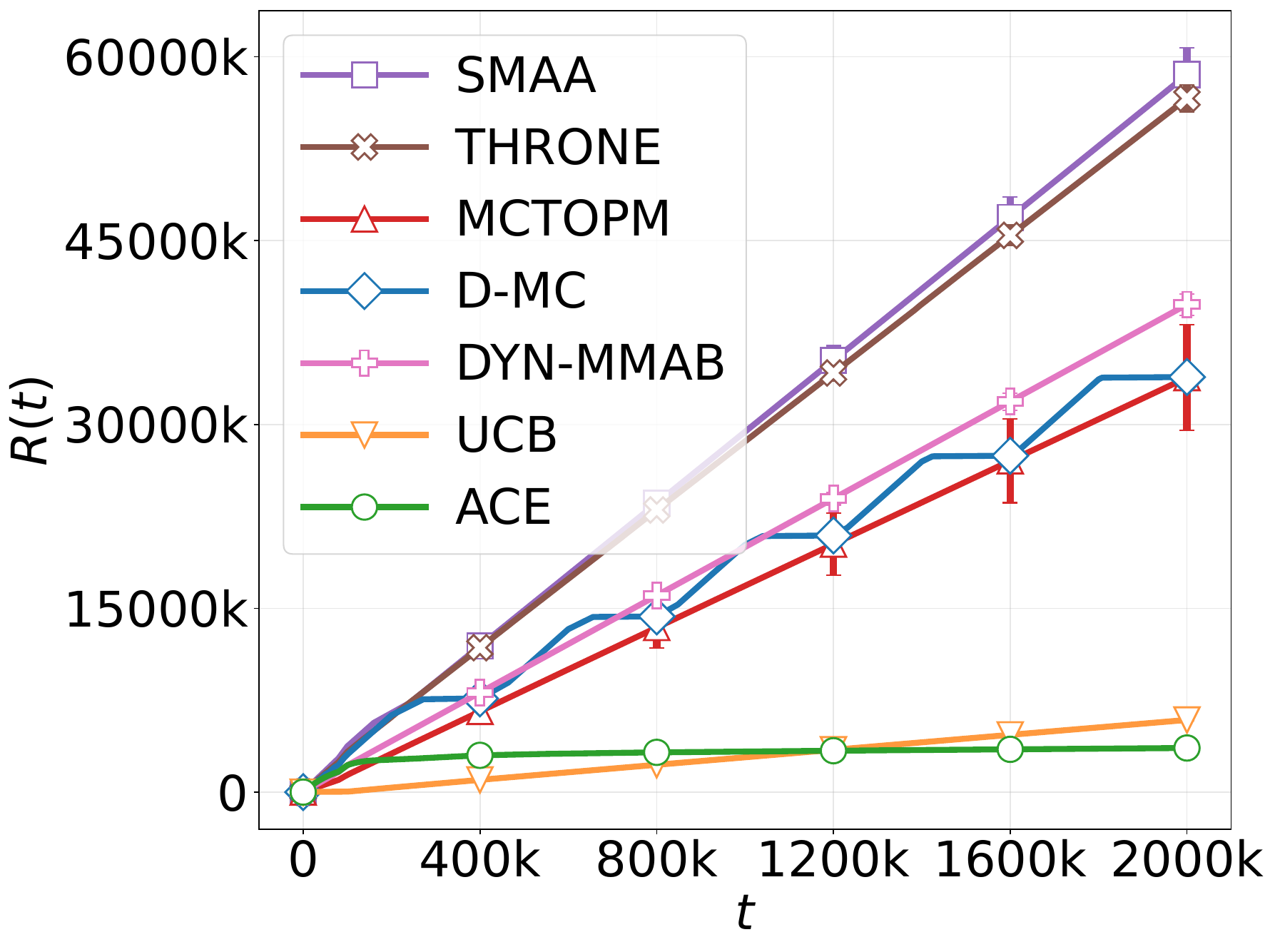}
        \caption{M=20, synthetic.}
        \label{fig:s-m-20}
    \end{subfigure}
    \hfill
    \begin{subfigure}[b]{0.33\textwidth}
        \includegraphics[width=\textwidth]{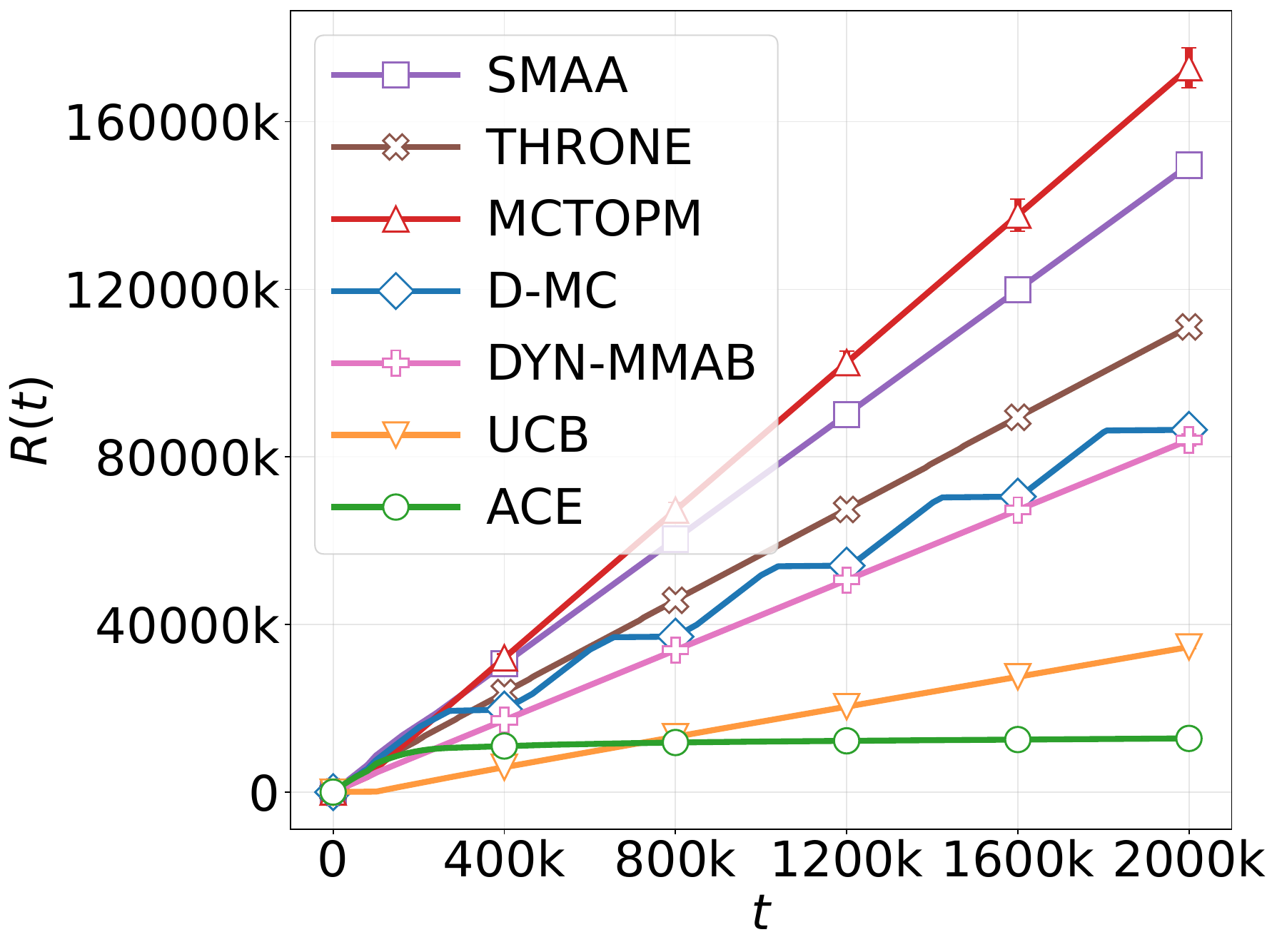}
        \caption{M=50, synthetic.}
        \label{fig:s-m-50}
    \end{subfigure}
    \caption{Comparison of cumulative regret for different numbers of players \(\mathbf{M}\) under different asynchronization settings.}
    \Description{Comparison of cumulative regret for different numbers of players \(\mathbf{M}\) under different asynchronization settings.}
    \label{fig:r-m}
\end{figure}

\begin{figure}[ht]
    \centering
    \begin{subfigure}[b]{0.33\textwidth}
        \includegraphics[width=\textwidth]{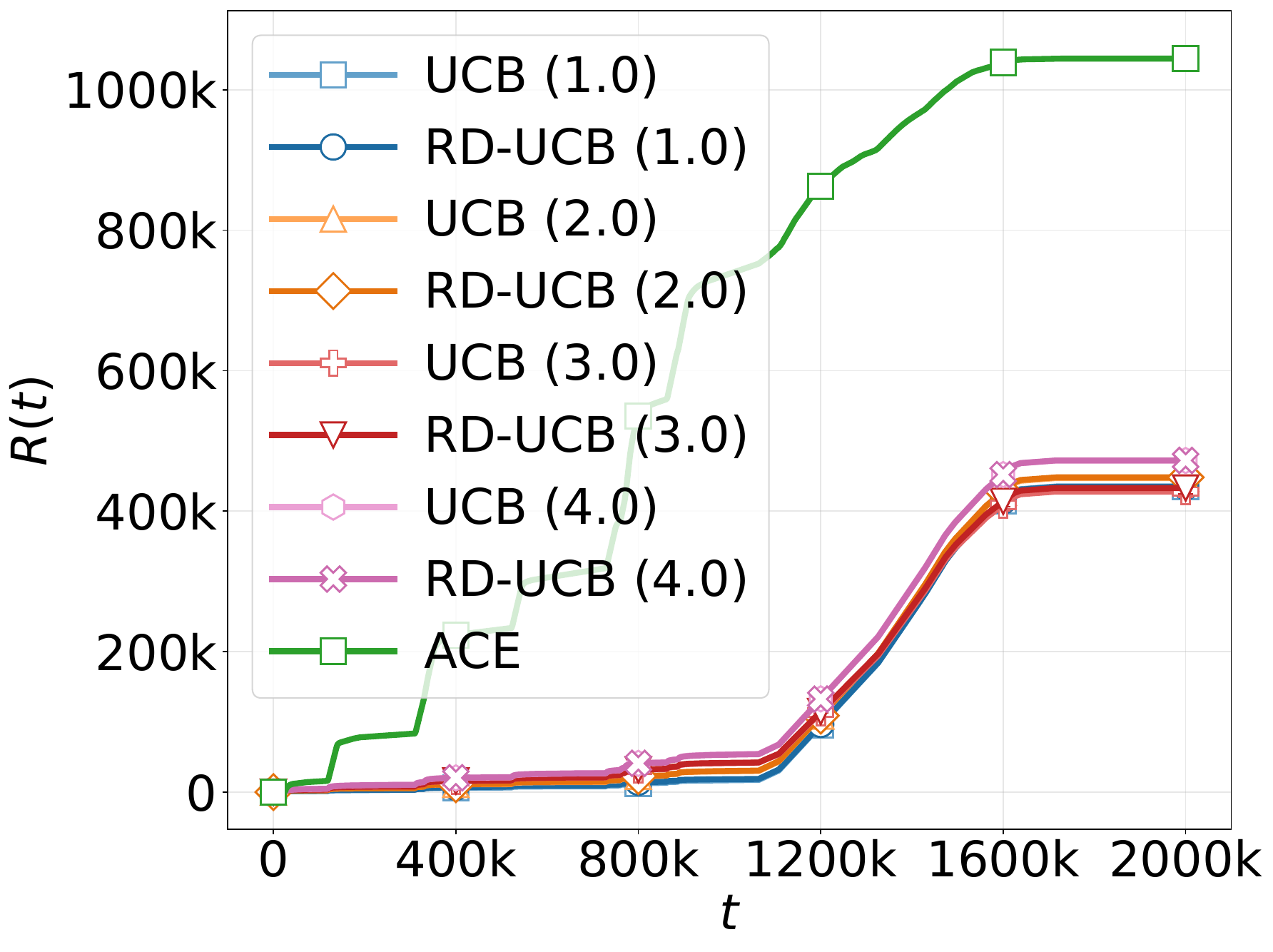}
        \caption{M=10, random. with UCBs.}
        \label{fig:ucb-m-rdm-10}
    \end{subfigure}
    \hfill
    \begin{subfigure}[b]{0.33\textwidth}
        \includegraphics[width=\textwidth]{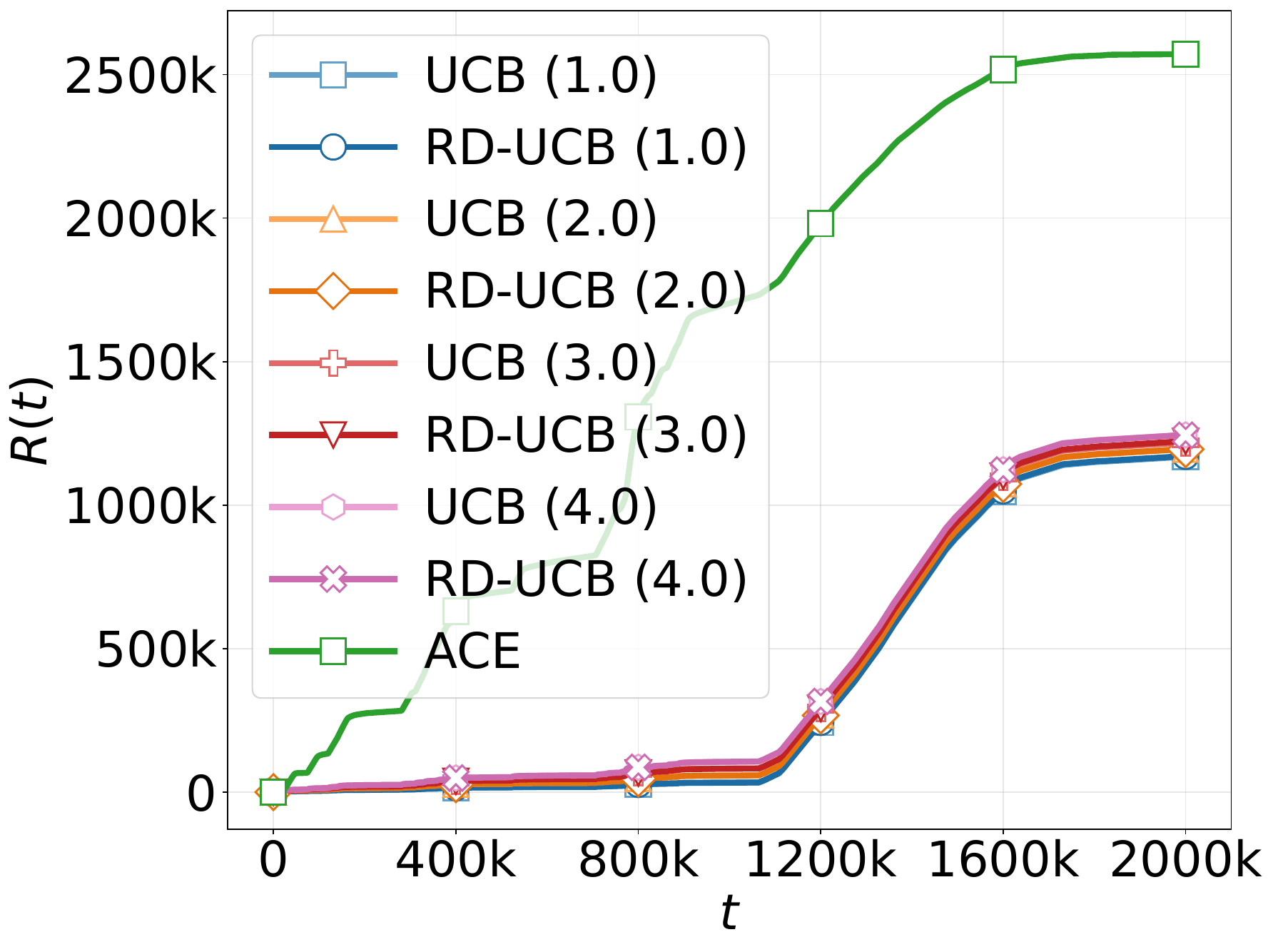}
        \caption{M=20, random. with UCBs.}
        \label{fig:ucb-m-rdm-20}
    \end{subfigure}
    \hfill
    \begin{subfigure}[b]{0.33\textwidth}
        \includegraphics[width=\textwidth]{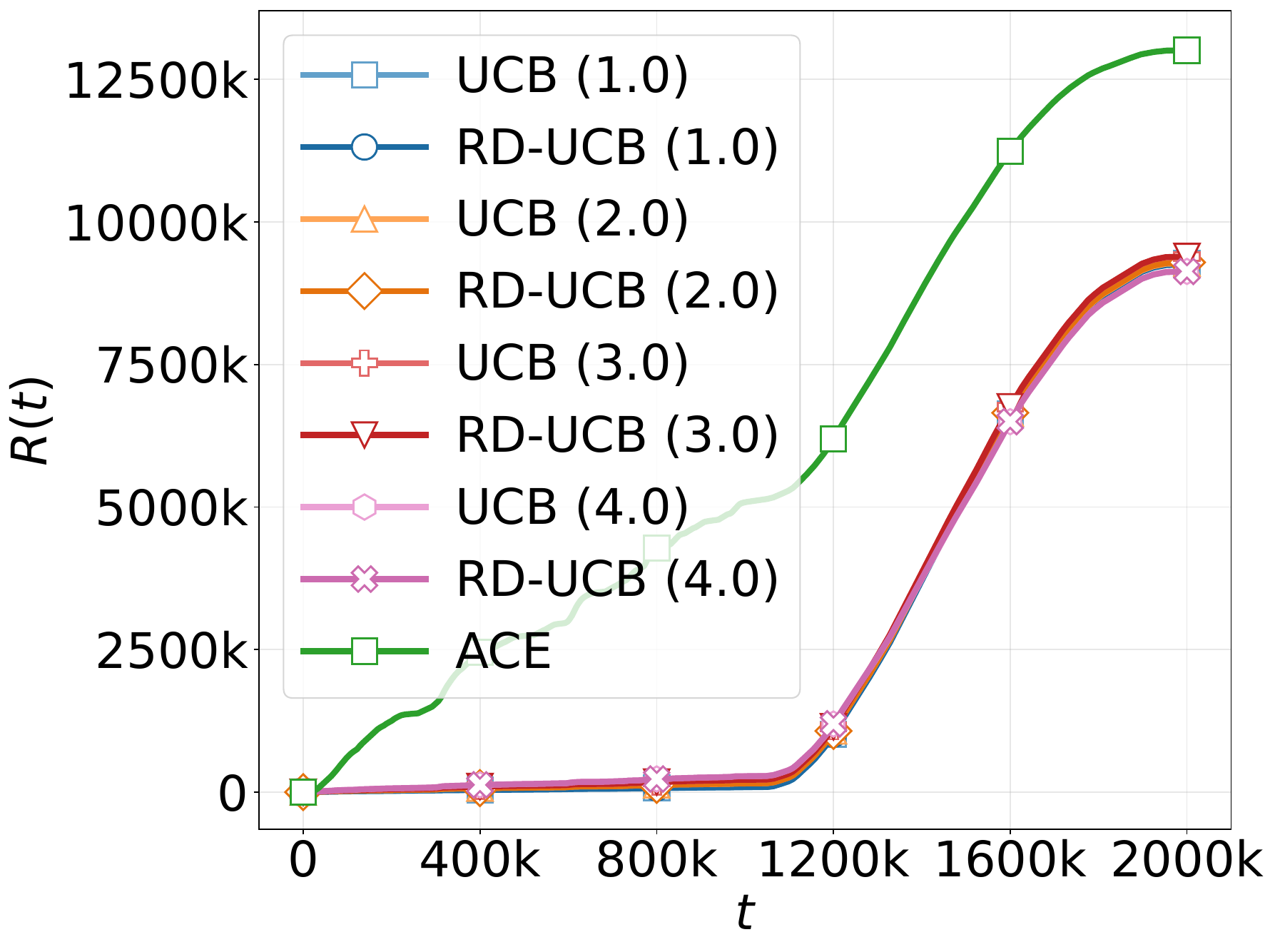}
        \caption{M=50, random. with UCBs.}
        \label{fig:ucb-m-rdm-50}
    \end{subfigure}
    \\[1em]
    \begin{subfigure}[b]{0.33\textwidth}
        \includegraphics[width=\textwidth]{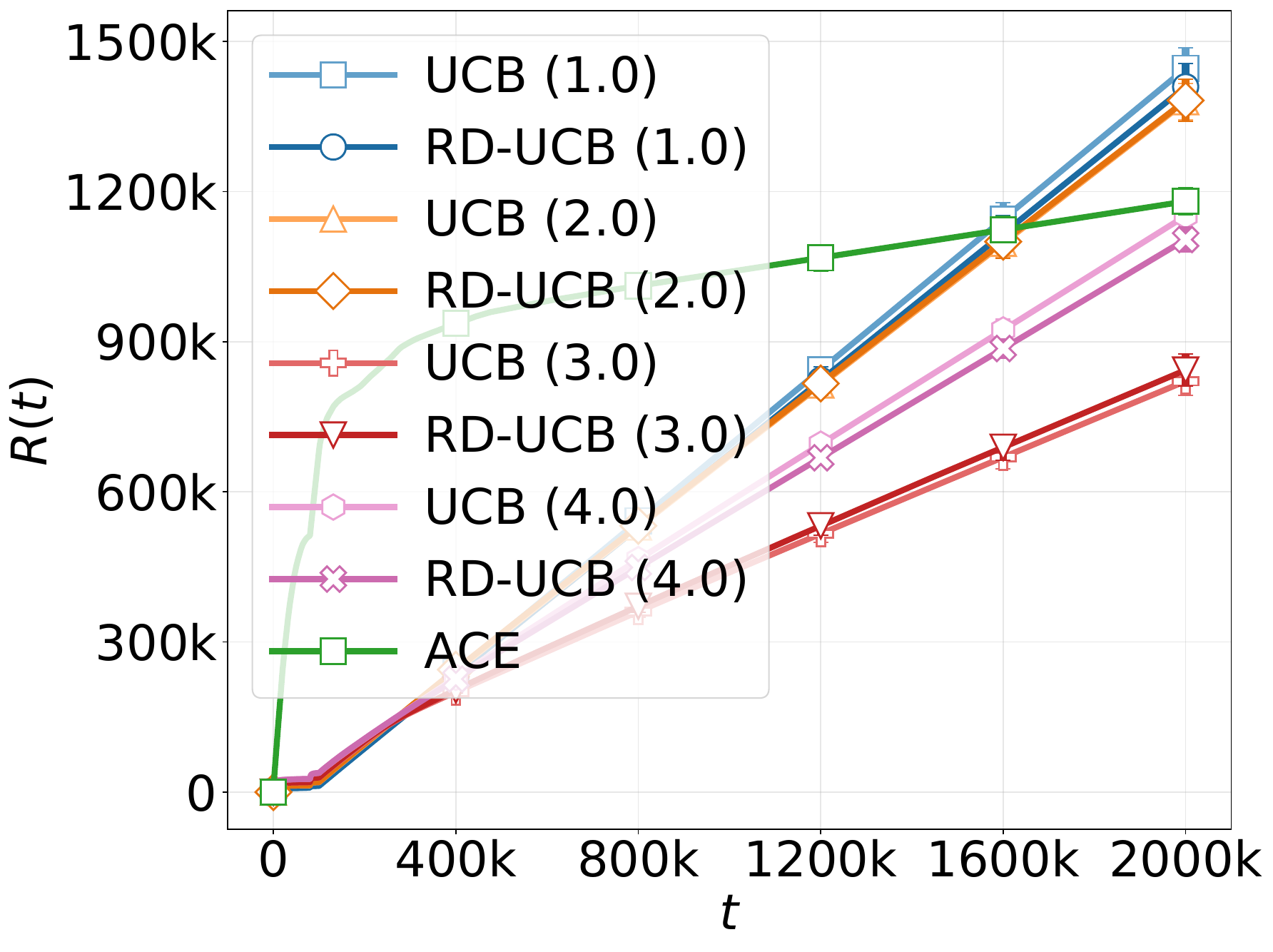}
        \caption{M=10, synthetic. with UCBs.}
        \label{fig:ucb-m-syn-10}
    \end{subfigure}
    \hfill
    \begin{subfigure}[b]{0.33\textwidth}
        \includegraphics[width=\textwidth]{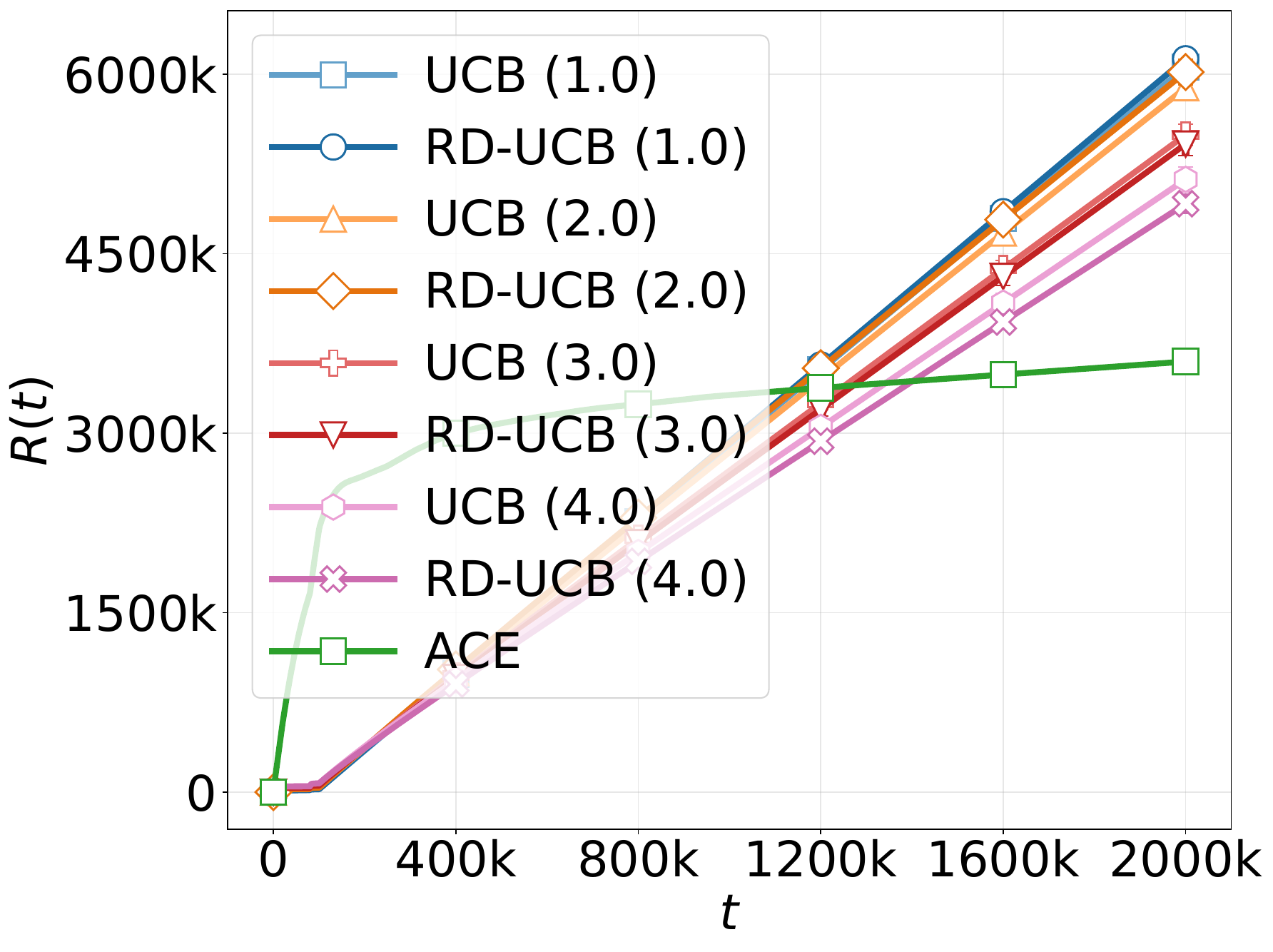}
        \caption{M=20, synthetic. with UCBs.}
        \label{fig:ucb-m-syn-20}
    \end{subfigure}
    \hfill
    \begin{subfigure}[b]{0.33\textwidth}
        \includegraphics[width=\textwidth]{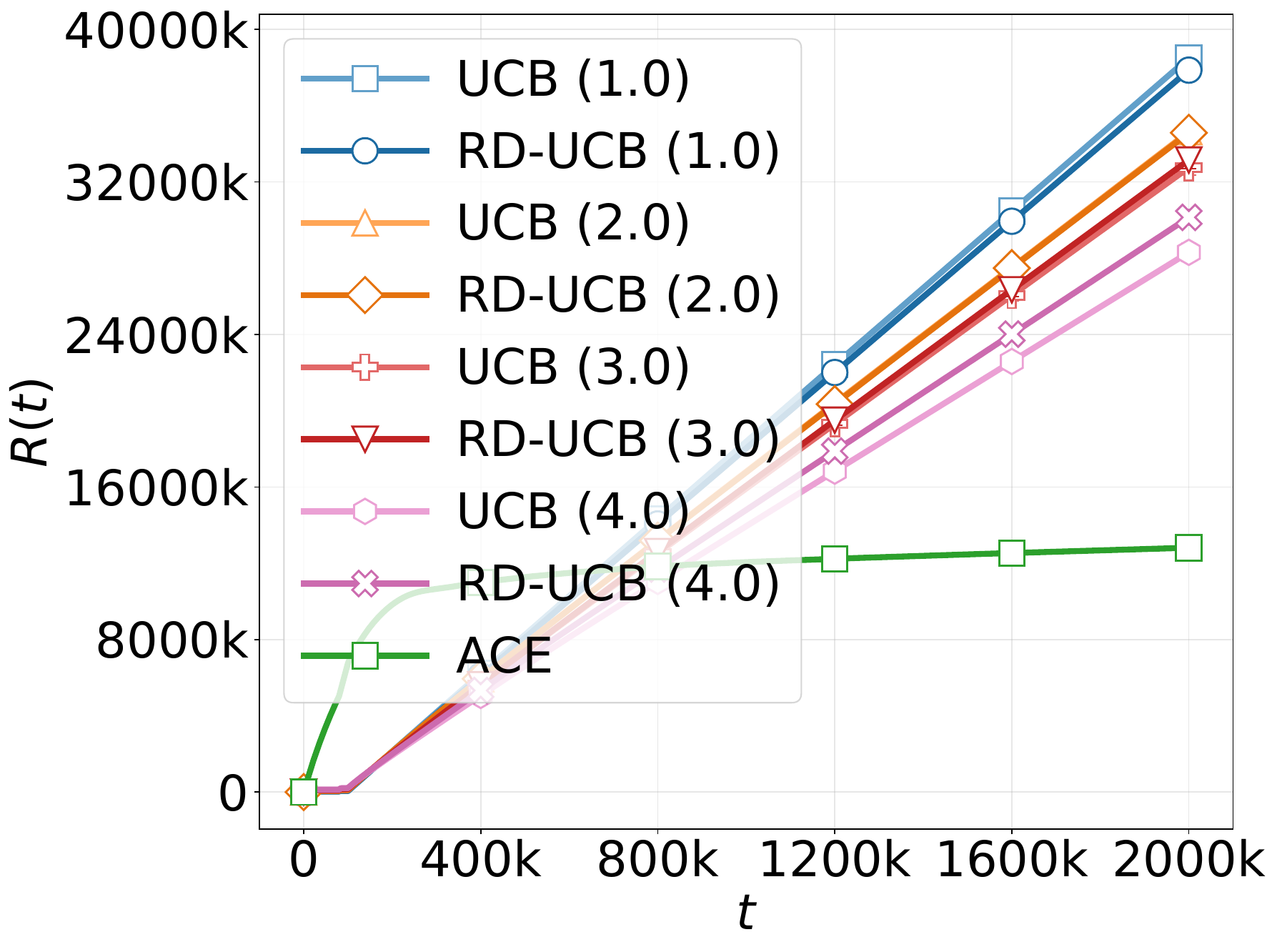}
        \caption{M=50, synthetic. with UCBs.}
        \label{fig:ucb-m-syn-50}
    \end{subfigure}
    \caption{Comparison of cumulative regret between UCB with multiple parameters and ACE for different $\mathbf{M}$ under different asynchronous settings.}
    \Description{Comparison of cumulative regret between UCB with multiple parameters and ACE for different $\mathbf{M}$ under different asynchronous settings.}
    \label{fig:ucb-m-rdm}
\end{figure}

\end{document}